\documentclass{article} 
\usepackage{iclr2027_conference,times}
\usepackage{amsmath,amssymb,amsthm,mathtools}
\usepackage{graphicx}
\usepackage{float}
\usepackage{subcaption}
\usepackage{booktabs}
\usepackage{tabularx}
\usepackage{hyperref}
\usepackage{url}
\usepackage{xcolor}
\definecolor{cKernel}{RGB}{0,80,160}
\definecolor{cCurv}{RGB}{180,40,40}
\definecolor{cSP}{RGB}{194,59,75}
\definecolor{cSigma}{RGB}{30,85,150}
\usepackage{enumitem}
\usepackage{microtype}
\usepackage{multirow}
\hypersetup{colorlinks=true,linkcolor=blue!60!black,citecolor=blue!60!black,urlcolor=blue!60!black}
\usepackage{babel,adjustbox,booktabs,multirow}

\newtheorem{proposition}{Proposition}
\newtheorem{assumption}{Assumption}
\newtheorem{definition}{Definition}
\newcommand{\U}{\mathcal U}
\newcommand{\NLL}{\mathrm{NLL}}
\newcommand{\argmax}{\operatorname*{arg\,max}}
\newcommand{\argmin}{\operatorname*{arg\,min}}
\newcommand{\mup}{\mu\mathrm{P}}
\newcommand{\mutransfer}{\mu\mathrm{Transfer}}
\newcommand{\sigmat}{\sigma\mathrm{Transfer}}
\newcommand{\blockdiag}{\operatorname{blockdiag}}

\newtheorem{theorem}{Theorem}
\newtheorem{lemma}{Lemma}
\newtheorem{corollary}{Corollary}
\newtheorem*{formaltheorem}{Theorem}

\DeclarePairedDelimiter{\norm}{\lVert}{\rVert}
\DeclarePairedDelimiter{\abs}{\lvert}{\rvert}

\newcommand{\E}{\mathbb E}
\newcommand{\tr}{\operatorname{tr}}
\newcommand{\op}{\mathrm{op}}

\title{$\sigma$Transfer: Uncertainty Transfer from Small to Large Networks under $\mup$}

\iclrfinalcopy

\author{
\textbf{Richard Bergna}$^{1,2}$\thanks{
Correspondence: \texttt{rsb63@cam.ac.uk}.
}
\quad
\textbf{Fernando Ruiz Mazo}$^{1}$
\quad
\textbf{Nicol\`o Felicioni}$^{2}$
\\[0.4em]
\textbf{Jos\'e Miguel Hern\'andez-Lobato}$^{1}$
\quad
\textbf{Kamil Ciosek}$^{2}$
\\[0.6em]
$^{1}$University of Cambridge
\quad
$^{2}$Spotify
}

\iclrfinalcopy
\begin{document}

\maketitle
\lhead{Preprint}

\begin{abstract}
Reliable predictive uncertainty in Laplace approximations depends critically on
the prior precision, yet selecting it requires a posterior sweep that
is prohibitively expensive for neural networks with billions of
parameters. Under the Maximal Update Parametrization ($\mup$), we
derive a rescaling of the prior covariance that makes the selected
precision stable as model width grows.
This leads to $\sigmat$: we select the precision on a smaller model and
zero-shot transfer it to the much larger model, i.e., without searching for
the precision on the larger model at all. We show
convergence of the prior kernel, posterior covariance, selected
precision, and posterior-derived decisions under explicit conditions, and verify $\sigmat$ across
regression, image classification, and Transformer readouts. For
example, measured precision-sweep speedups reach
${\sim}5000\times$ when transferring from width 128 to 4096 on MNIST, at a
target-NLL degradation of $0.002$; transferring from a public 1B to 7B
model gives a median search speedup of ${\sim}2.3\times$ (up to
${\sim}330\times$), with a mean measured target-NLL increase below $10^{-4}$
across ten tasks. The same
posterior stability also enables transfer of acquisition,
OOD-detection, and abstention decisions without constructing a target
posterior.
\end{abstract}

\section{Introduction}
\label{sec:intro}

\begin{figure}[ht]
\centering
\includegraphics[width=\textwidth]{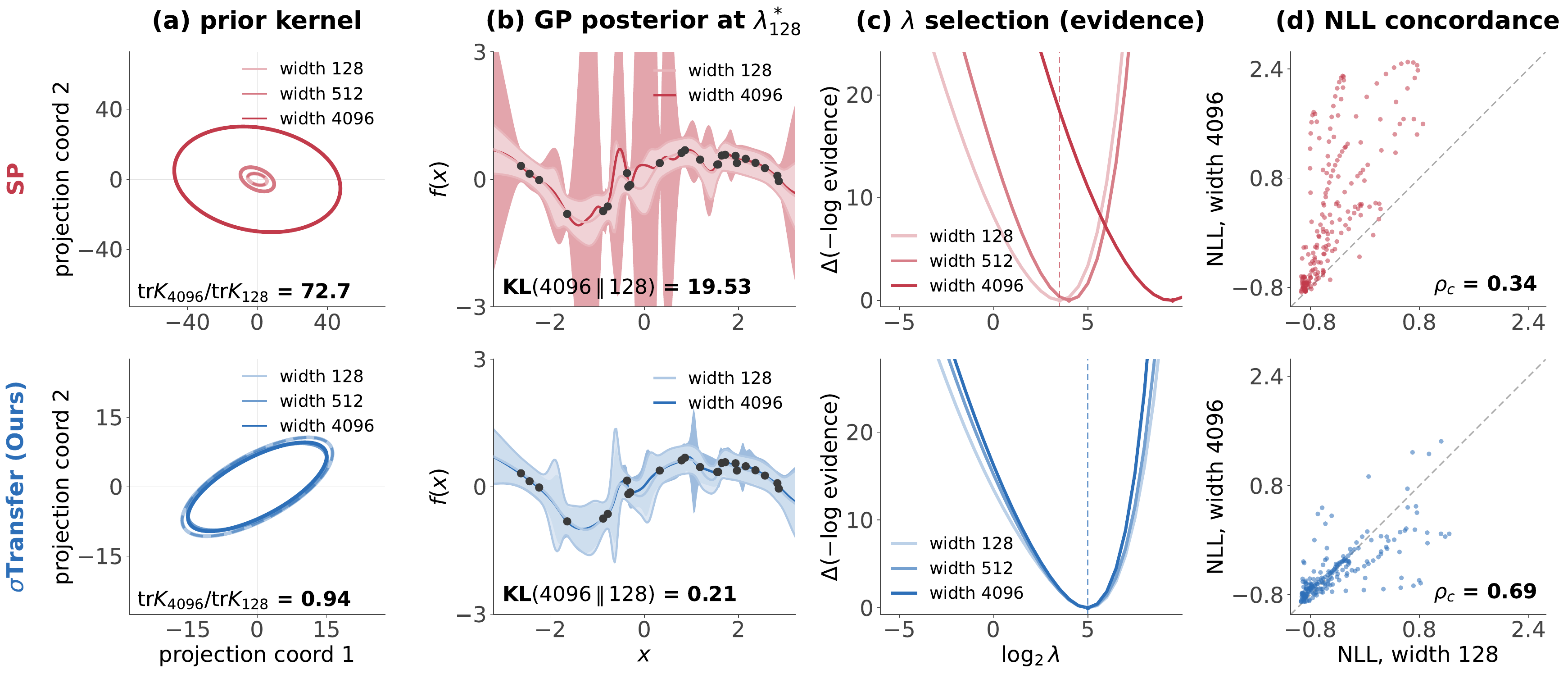}
\caption{\textbf{Why prior precision transfers under $\sigmat$.}
The same regression architecture at widths $128$, $512$, and $4096$,
under standard parametrization (\textcolor{cSP}{SP, red}) and
\textcolor{cSigma}{$\sigmat$ (blue)}.
\textbf{(a)}~The prior kernel, in a common two-dimensional projection.
\textbf{(b)}~Predictive posteriors using the precision selected at
width $128$.
\textbf{(c)}~Plug-in evidence curves used to select it; the dashed line
marks the width-$128$ optimum.
\textbf{(d)}~Predictive NLL per example, width $128$ versus wider
networks; points on the diagonal indicate agreement.
Overall, $\sigmat$ remains stable across width, while SP can drift.}
\label{fig:concept}
\end{figure}

The Laplace approximation for neural
networks~\citep{denker1990transforming,mackay1992practical} provides
predictive uncertainty estimates and supports evidence-based model
selection. Its
linearised form places a Gaussian
posterior on a trained network, yielding predictive distributions
useful for
out-of-distribution detection~\citep{daxberger2021laplace, kristiadi2020being}, active
learning~\citep{gal2017deep,kirsch2019batchbald}, and selective
prediction. The resulting uncertainty estimates depend on the prior
precision, a hyperparameter that controls the scale of posterior uncertainty
and is commonly selected using validation or marginal
likelihood~\citep{kristiadi2021learnable,immer2021scalable,
antoran2022adapting}. Selecting it typically requires repeated posterior evaluations across
candidate precisions, making the sweep increasingly costly as network
width grows.

Can we instead search for the prior precision on a small network and
reuse it, unchanged, on a much larger one? Under standard
parametrization (SP), PyTorch's default and the convention of a single
global learning rate, this transfer can be unreliable: the same
precision can produce different uncertainty at different widths. The
Maximal Update Parametrization ($\mup$;
\citealp{yang2021tensor,yang2021tuning}) is an alternative convention
that rescales each layer's initialization, forward multiplier, and
learning rate with width. This stabilizes training dynamics across
width, allowing hyperparameters such as the learning rate to be tuned at
a small width and reused at a larger one. Applying the Laplace
approximation to a $\mup$ network is not enough on its own: the prior
precision that is optimal at one width is still not optimal at another.

We introduce $\sigmat$, which builds on $\mup$ to make the optimal prior
precision stable across width (Figure~\ref{fig:concept}(c)), so a
precision chosen on a small (cheap) network is near-optimal for a large
(expensive) one. $\sigmat$ normalizes the prior geometry: each layer's
prior is scaled by a factor fixed by the width and $\mup$'s coordinate
scalings (Section~\ref{sec:keyidea-revised}). This adds no new
hyperparameter.

Roughly speaking, under $\sigmat$, networks of different widths
approximate the same limiting prior kernel: the function-space
covariance that the parameter prior induces on predictions, once the
network is linearized around its trained weights. If this kernel is
stable across width, we show that the Laplace posterior is stable too.
We use this stability in two ways: a small network can select the prior
precision for a larger one (\emph{prior-precision transfer}), or make
uncertainty-based decisions itself, such as which data to label or
when the larger network should abstain, without constructing a Laplace
posterior for it (\emph{decision transfer}).
Under the conditions of Section~\ref{sec:proof-goals}, our main
guarantees can be summarized by the progression
\[
\underbrace{K_n\to K_\infty}
_{\text{prior-kernel stability}}
\;\Longrightarrow\;
\underbrace{
m_n^\lambda\to m_\infty^\lambda,\;
C_n^\lambda\to C_\infty^\lambda}
_{\text{posterior mean and covariance stability}}
\;\Longrightarrow\;
\underbrace{\lambda_n^\star\to\lambda^\star}
_{\text{prior-precision transfer}}
,\quad
\underbrace{\text{same decision}}
_{\text{decision transfer}} .
\]
Figure~\ref{fig:concept} illustrates this progression at finite width, from the prior kernel to the posterior, the selected precision, and per-point negative log-likelihood (NLL) concordance. Our contributions are:
\begin{enumerate}[leftmargin=1.2em,itemsep=2pt,topsep=2pt,parsep=0pt]
\item We introduce $\sigmat$, a method that stabilizes the prior kernel
and therefore the Laplace posterior across width, so that uncertainty
computed on a small (cheap) network transfers to a much larger
(expensive) one.
\item \emph{Prior-precision transfer}: the optimal prior precision is
stable across width (Figure~\ref{fig:concept}(c)), so a precision chosen
on the small network is near-optimal for the large one, and we bound the
excess NLL from using it. The large network then fits a single posterior
instead of a full sweep.
\item \emph{Decision transfer}: decisions made from the small network's
posterior are approximately those the large network would make, because
the posterior mean and covariance both approach a common limit and so
approach each other. We demonstrate this for data acquisition,
out-of-distribution detection, and abstention, where the large network
needs no Laplace posterior at all.
\end{enumerate}

\section{Related work}
\label{sec:related}

The Tensor Programs series developed a mathematical framework for
neural networks at large width, settling questions such as the
convergence of standard architectures to Gaussian processes in the
infinite-width
limit~\citep{yang2020tensor, yang2020tensorii, yang2019wide}.
A key practical
outcome was the Maximal Update Parametrization
($\mup$; \citealp{yang2021tensor}), which scales initialization
and per-layer multipliers so that feature updates neither vanish nor
explode as width grows, keeping training stable and expressive
regardless of model size. Because $\mup$ networks converge to a common
limit as width grows, optimization hyperparameters such as the learning
rate tuned on a small proxy carry over to a larger target.
\citet{yang2021tuning} demonstrated this with $\mutransfer$,
transferring these settings from a 40M-parameter proxy to a
6.7B-parameter language
model at a search cost of about $7\%$ of one pretraining run, then
training the large model once with the optimal learning rate found on
the smaller proxy.

\citet{denker1990transforming} first used the Hessian to
obtain predictive error bars for neural networks;
\citet{mackay1992practical} developed this into a full Bayesian
framework, placing a Gaussian posterior over trained weights and
selecting the prior precision by maximising the marginal likelihood. Laplace posteriors have since achieved strong
results in out-of-distribution detection, active learning, and
selective prediction~\citep{daxberger2021laplace}. The main
computational barrier is the Hessian: storing it costs $O(P^2)$ and
inverting it $O(P^3)$ in the number of parameters $P$. Last-layer
restrictions~\citep{daxberger2021laplace}, Kronecker-factored
curvature~\citep{ritter2018scalable}, and online evidence
maximisation~\citep{immer2021scalable} have each reduced this cost.
Even so, selecting the prior precision requires sweeping over many
candidates and evaluating the posterior at each one. When the target
model has millions or billions of parameters, this search alone can
dominate the cost of Laplace inference.

\section{Preliminaries}
\label{sec:background}

\paragraph{Linearized Laplace.}
The Laplace approximation represents the parameter posterior by a
Gaussian centred at the trained parameters $\hat\theta_n$ of a network
with width $n$. Under the
standard Gaussian prior
$\theta\mid\lambda\sim\mathcal N(0,\lambda^{-1}I)$, the Laplace
approximation gives
\[
\theta\mid D,\lambda
\;\sim\;
\mathcal N\!\left(\hat\theta_n,(H_n+\lambda I)^{-1}\right),
\qquad
H_n=J_n(D)^\top\Omega_nJ_n(D).
\]
Here $H_n$ is the generalized Gauss--Newton curvature, $D$ the training
inputs, and $X,X'$ arbitrary finite sets of evaluation inputs. The
function $f_n(x;\theta)\in\mathbb R^{d_{\mathrm{out}}}$ denotes the
neural-network output: a scalar in our regression experiments and the
pre-softmax logit vector in classification;
$J_n(X)=\partial f_n(X;\hat\theta_n)/\partial\theta$ is its Jacobian
with respect to the stored parameters, stacked over the inputs in $X$,
and $\Omega_n$ is the negative-log-likelihood Hessian with respect to
these outputs, evaluated at the trained predictions.
Linearized Laplace~\citep{immer2021improving} uses a first-order expansion of $f_n$ around
$\hat\theta_n$ to propagate the Gaussian parameter posterior to the
output space of $f_n$. The resulting Gaussian output posterior has mean
and covariance
\[
m_n^\lambda(X)=f_n(X;\hat\theta_n),
\qquad
C_n^\lambda(X,X')=
J_n(X)(H_n+\lambda I)^{-1}J_n(X')^\top.
\]
With the trained parameters $\hat\theta_n$ held fixed, $\lambda$ leaves
the mean unchanged and controls the scale of posterior uncertainty:
larger values contract the covariance, whereas smaller values permit
greater uncertainty.
For classification, predictive class probabilities are defined by
averaging softmax over the Gaussian logit posterior; we approximate
this expectation using either a probit approximation or Monte Carlo
sampling. For Gaussian regression, observation noise is added when
forming the predictive distribution.
The validation-NLL curve is the map $\lambda\mapsto\NLL_n(\lambda)$ from
prior precision to the NLL of the width-$n$ posterior predictive on a
fixed validation set.

\paragraph{$\mup$ and hyperparameter transfer.}
The Maximal Update Parametrization ($\mup$), introduced by
\citet{yang2021tensor}, enables training hyperparameters
such as the learning rate to transfer across model
widths~\citep{yang2021tuning}. $\mup$
achieves this through width-dependent scaling of the initialization,
forward pass, and
parameter updates. Each layer is updated as much as possible without
causing logits or activations to diverge as width tends to infinity, so
hidden representations keep changing during training instead of staying
close to their random initialization as the network widens.
The same $\mup$ network can be expressed using different parameter
units. Writing $\theta_n$ for the stored weights of a network with width
$n$, we introduce rescaled parameter coordinates $\phi_n$ through
\[
\phi_n=T_n^{-1}\theta_n,
\qquad
T_n=\blockdiag\bigl(t_1(n)I,\ t_2(n)I,\ldots\bigr).
\]
Here $T_n$ applies one positive scale factor to each parameter tensor.
These factors depend on width but remain fixed during training. The
initialization and update rules are adjusted consistently with these
parameter units~\citep[Proposition~2.2.3]{yang2023tensor}.


\section{Method: \texorpdfstring{$\sigmat$}{sigmaTransfer}}
\label{sec:keyidea-revised}

We introduce $\sigmat$, a Laplace method that rescales the prior so
that a selected prior precision or posterior-based decisions can
transfer from a smaller network to a larger one. This section defines
$\sigmat$ and develops the intuition for its stability across width;
Section~\ref{sec:proof-goals} states the conditions and formal
guarantees for both, and Appendix~\ref{app:proofs} provides the proofs.

\paragraph{The transfer problem.}
To transfer a prior precision or posterior-based decision across width,
the same numerical value of the prior $\lambda$ must induce similar linearized
Laplace output posteriors as the network width grows. If the predictive mean
$m_n^\lambda$ and covariance $C_n^\lambda$ each approach fixed limits,
then sufficiently wide networks of widths $n$ and $N$ approximate the
same limiting posterior and hence one another, even when $N\gg n$. We
call a quantity evaluated on fixed inputs \emph{width-stable} when it
approaches a fixed limit as width grows;
Definition~\ref{def:width-stable} makes this precise.

\begin{definition}[Width-stability]
\label{def:width-stable}
Fix a finite input set $X$. A quantity $Q_n$ indexed by the network width
$n$ and evaluated on $X$ is \emph{width-stable} if it converges entrywise,
$Q_n\to Q_\infty$ as $n\to\infty$.
\end{definition}

Predictive-mean stability has already been
established.\footnote{For linearized Laplace, the predictive mean is the
trained $\mup$ network output, whose width-stability under matched
fixed-step training follows
from~\citet{yang2021tensor,yang2023tensor}.}
We therefore focus on stabilizing the covariance.
Section~\ref{sec:background} gives the posterior covariance under an
isotropic prior, which we write here as
\begin{equation}
C_{n,I}^\lambda(X,X')
=
J_n(X)(H_n+\lambda I)^{-1}J_n(X')^\top,
\qquad
H_n=J_n(D)^\top\Omega_nJ_n(D).
\label{eq:isotropic-posterior-revised}
\end{equation}
Here $n$ denotes the network width, $J_n$ is the network-output
Jacobian at the trained parameters, and $\Omega_n$ is the output-space
NLL curvature on the training data. This covariance expression does not reveal when width-stability holds,
because $J_n$ and $H_n$ change dimension with width.

To identify how to make this covariance width-stable, we apply the
Woodbury identity~\citep{woodbury1950inverting} to rewrite it in
function space. The key step is to do so for a general Gaussian prior
with covariance $\lambda^{-1}{\color{cKernel}S_n}$, rather than the
isotropic prior alone. For the full GGN and $\lambda>0$,
\begingroup\small
\begin{equation}
C_n^\lambda(X,X')
=\frac{{\color{cKernel}K_n(X,X')}}{\lambda}-\frac{{\color{cKernel}K_n(X,D)}}{\lambda^{2}}\,{\color{cCurv}\Omega_n^{1/2}}\bigl(I+\lambda^{-1}{\color{cCurv}\Omega_n^{1/2}}\,{\color{cKernel}K_n(D,D)}\,{\color{cCurv}\Omega_n^{1/2}}\bigr)^{-1}{\color{cCurv}\Omega_n^{1/2}}\,{\color{cKernel}K_n(D,X')},
\label{eq:function-space-posterior-revised}
\end{equation}
\endgroup
where
${\color{cKernel}K_n(X,X')}=J_n(X)\,{\color{cKernel}S_n}\,J_n(X')^\top$
is the prior kernel. In this form, the prior geometry $S_n$ enters the
posterior covariance only through ${\color{cKernel}K_n}$. The
posterior also depends on the prior precision $\lambda$, which is
fixed, and the likelihood curvature ${\color{cCurv}\Omega_n}$, which
is width-stable for the likelihoods considered here.\footnote{For
fixed-noise Gaussian regression, $\Omega_n$ is independent of width.
For softmax classification, it converges whenever the trained logits
converge.} However, ${\color{cKernel}K_n}$ is not guaranteed to be
width-stable: for instance, the isotropic choice $S_n=I$ gives
$K_{n,I}=J_nJ_n^\top$, which can diverge as width grows. We therefore
need to choose $S_n$ so that ${\color{cKernel}K_n}$ is width-stable;
Eq.~\eqref{eq:function-space-posterior-revised} then implies that
$C_n^\lambda$ is width-stable.

\paragraph{The $\sigmat$ prior kernel.}
We now construct ${\color{cKernel}S_n}$ so that the prior kernel
${\color{cKernel}K_n}$ is width-stable. Building on the width scalings
of $\mup$, we introduce width-normalized parameter coordinates
$\phi_n=T_n^{-1}\theta_n$. The block-diagonal map $T_n$ rescales each
parameter tensor so that its contribution to the prior kernel remains
at a comparable scale across widths. Appendix~\ref{app:sn} derives
these factors.

The reason for working in these coordinates is that \textbf{we prove
that the isotropic prior kernel in these width-normalized coordinates
is width-stable}. Under the conditions of
Theorem~\ref{thm:prior-kernel-stability}
(see Section~\ref{sec:proof-goals} for the precise statement),
\[
{\color{cKernel}K_n(X,X')}
=
J_{\phi,n}(X)\,J_{\phi,n}(X')^\top
\xrightarrow{n\to\infty}
K_\infty(X,X').
\]
Here $J_{\phi,n}$ is the network-output Jacobian with respect to the
width-normalized coordinates $\phi_n$, rather than the stored parameters.
Since $\phi_n=T_n^{-1}\theta_n$, the chain rule gives
$J_{\phi,n}=J_nT_n$. We can therefore express this same stable kernel
in the stored parameter coordinates:
\begin{equation}
{\color{cKernel}K_n(X,X')}
=
J_{\phi,n}(X)\,J_{\phi,n}(X')^\top
=
J_n(X)\,T_nT_n^\top\,J_n(X')^\top
=
J_n(X)\,{\color{cKernel}S_n}\,J_n(X')^\top,
\label{eq:sigmat-kernel-revised}
\end{equation}
where ${\color{cKernel}S_n}=T_nT_n^\top$. Equivalently,
$\lambda{\color{cKernel}S_n^{-1}}$ is the isotropic precision $\lambda I$
of the width-normalized coordinates expressed in the stored
coordinates, which is why we call it a width-normalized prior.
The posterior covariance in
the stored parameter coordinates is therefore
\begin{equation}
C_n^\lambda(X,X')
=
J_n(X)
\bigl(H_n+\lambda\,{\color{cKernel}S_n^{-1}}\bigr)^{-1}
J_n(X')^\top.
\label{eq:sigmat-posterior-revised}
\end{equation}
Thus $\sigmat$ changes only the prior term from $\lambda I$ to
$\lambda{\color{cKernel}S_n^{-1}}$; the trained checkpoint,
predictions, and likelihood curvature remain unchanged, and $\lambda$
remains the only tuned prior hyperparameter.
Section~\ref{sec:proof-goals} states the conditions under which this
gives prior-precision and decision transfer across width.
In practice, $S_n$ is never constructed as a dense matrix: within a
parameter tensor with width factor $t$, the prior precision is the
scalar $\lambda/t^2$, which is what one passes to the Laplace
implementation.
Appendix~\ref{app:sn} gives the block scales and verifies the
coordinate identity numerically.

\section{Theoretical guarantees}
\label{sec:proof-goals}

Section~\ref{sec:keyidea-revised} constructed the $\sigmat$ prior so
that the prior kernel $K_n$ is width-stable. We now propagate this
stability through the posterior to practical transfer guarantees;
formal statements and proofs appear in Appendix~\ref{app:proofs}.
Together with convergence of the likelihood curvature and trained
predictive mean, the results build on each other under the conditions
below:
\[
\underbrace{K_n\to K_\infty}
_{\substack{\text{\scriptsize Kernel stability}\\
\text{\scriptsize Thm.~\ref{thm:prior-kernel-stability}}}}
\;\Longrightarrow\;
\underbrace{C_n^\lambda\to C_\infty^\lambda}
_{\substack{\text{\scriptsize Covariance stability}\\
\text{\scriptsize Thm.~\ref{thm:posterior-covariance-stability}}}}
\;\Longrightarrow\;
\underbrace{\lambda_n^\star\to\lambda^\star}
_{\substack{\text{\scriptsize Precision transfer}\\
\text{\scriptsize Thm.~\ref{thm:lambda-transfer}}}}
,\quad
\underbrace{a_{n,\lambda}^\star=a_{\infty,\lambda}^\star}
_{\substack{\text{\scriptsize Decision transfer}\\
\text{\scriptsize Thm.~\ref{thm:decision-transfer}}}}
\]

\begin{theorem}[Prior-kernel stability; informal]
\label{thm:prior-kernel-stability}
Consider the scalar-output, fixed-depth, equal-width ReLU $\mup$ MLP of
Appendix~\ref{app:sn}. Suppose its checkpoint is obtained after a fixed
finite number of width-matched $\mup$ training steps and satisfies the
regularity conditions stated precisely in
Appendix~\ref{app:proof-kernel}. Then, on every fixed finite input set,
the $\sigmat$ prior kernel
$K_n(x,x')=J_n(x)S_nJ_n(x')^\top$ converges jointly almost surely as
$n\to\infty$; the trained outputs converge jointly as well. The same
conclusion holds for any fixed output dimension. The full statement,
assumptions, and proof are in Appendix~\ref{app:proof-kernel}.
\end{theorem}

The proof builds on Tensor Programs~\citep{yang2023tensor}, which
establish large-width limits for empirical averages arising during a
fixed number of training steps. The $\sigmat$ normalization expresses
the prior kernel as a finite combination of these averages, so their
convergence implies convergence of the kernel.
Appendix~\ref{app:proof-kernel} provides the smoothing argument needed
to extend this reasoning to ReLU.

\begin{theorem}[Posterior-covariance stability; informal]
\label{thm:posterior-covariance-stability}
For full-GGN Laplace, given prior-kernel convergence and width-stable
positive-semidefinite likelihood curvature, the linearized-Laplace
output covariance $C_n^\lambda$ converges uniformly over any interval
$[\lambda_-,\lambda_+]$ with $0<\lambda_-\le\lambda_+<\infty$
(proof in Appendix~\ref{app:proof-posterior}).
\end{theorem}

Thus the same numerical precision produces increasingly similar
posterior uncertainty as the networks become wider.

\begin{theorem}[Prior-precision transfer; informal]
\label{thm:lambda-transfer}
If the posterior mean and covariance on a fixed validation set
converge uniformly in $\lambda$ over an interval
$\Lambda=[\lambda_-,\lambda_+]$ with $0<\lambda_-\le\lambda_+<\infty$,
then the validation-NLL curves converge uniformly.
If the limiting curve has a unique minimizer $\lambda^\star$,
any sequence of minimizers satisfies
$\lambda_n^\star\to\lambda^\star$.
Even without a unique minimizer, the excess NLL from using the
width-$n$ proxy's precision on a target of width $N$ satisfies
\begin{equation}
\NLL_N(\lambda_n^\star)
-\min_{\lambda\in\Lambda}\NLL_N(\lambda)
\le 2\sup_{\lambda\in\Lambda}
\lvert\NLL_n(\lambda)-\NLL_N(\lambda)\rvert.
\label{eq:nll-transfer-bound}
\end{equation}
Full conditions and proof in Appendix~\ref{app:proof-precision}.
\end{theorem}

\begin{theorem}[Posterior-derived decision transfer; informal]
\label{thm:decision-transfer}
Fix $\lambda>0$ and a finite action set. Suppose the posterior mean and
covariance converge on the fixed finite set containing all inputs used
by the action scores. If the scores depend continuously on these
moments through scoring rules shared across widths, they converge
uniformly. If the limiting best action has a positive gap over all
alternatives, two sufficiently wide networks choose the same action. At
finite widths, even without this gap, the loss in target score from
using the proxy's action is at most twice the largest proxy--target
score difference. Full conditions and proof appear in
Appendix~\ref{app:proof-decision}.
\end{theorem}

\paragraph{Scope.}
The posterior-transfer results apply beyond MLPs: for any width-indexed
architecture with $\theta_n=T_n\phi_n$,
Theorems~\ref{thm:posterior-covariance-stability}--\ref{thm:decision-transfer}
apply under their stated conditions whenever the width-normalized prior
kernel, predictive mean, and likelihood curvature converge.
Theorem~\ref{thm:prior-kernel-stability} establishes kernel convergence
for the stated finite-step ReLU MLP setting. For last-layer Laplace, the
kernel condition reduces to convergence of the normalized final-feature
Gram matrix. Extensions to diagonal/KFAC curvature, to finitely many
sequential acquisition rounds, and to decisions at the transferred
precision appear in Appendices~\ref{app:structured-approximations}
and~\ref{app:proof-decision}, under additional conditions. Precision
transfer using fixed-checkpoint evidence additionally requires
convergence of the prior-weighted squared parameter norm and a unique
limiting maximizer; posterior convergence alone is insufficient
(Appendix~\ref{app:proof-precision}).

\section{Results}
\label{sec:results}

We evaluate $\sigmat$ in two deployment regimes. Under Regime~A (prior transfer), a
small proxy selects the Laplace prior precision $\lambda$, and the larger
target constructs one posterior at the transferred value. This removes
the target-side precision sweep, not the target posterior itself. Under
Regime~B (decision transfer), the proxy supplies posterior-derived
decisions, such as acquisition rankings and active-learning selections,
and the target requires no posterior at deployment. Target posteriors in
our Regime-B experiments are constructed only as evaluation references.
To measure transfer quality we
track two estimands: the selection gap $|\Delta\lambda|$, the distance
between the proxy and target optimal precisions in bits, and the
predictive transfer gap $|\Delta\NLL|$, the absolute target test-NLL difference
between proxy-selected and target-selected precisions. All posteriors
are post-hoc GGN Laplace around fixed checkpoints; formal definitions
of all evaluation metrics are in Appendix~\ref{app:metrics}.

The main tables compare four methods. SP with an isotropic prior
is standard Laplace practice and serves as the main baseline. We add
three further comparisons to isolate the source of any improvement.
$\mup$ with an isotropic prior uses the $\mup$ parametrization but
keeps the standard Laplace prior, testing whether $\mup$ training alone
helps transfer. SP + metric applies our prior geometry under the standard
parametrization; this combination is not covered by
Theorem~\ref{thm:prior-kernel-stability}, but it serves as a
mis-specification control and, as we see below, sometimes helps
empirically. $\sigmat$ combines $\mup$ training with the width-normalized
prior. The two $\mup$ methods share MAP checkpoints, so any difference
between them is due to the prior geometry alone.
Appendix~\ref{app:normalization-baselines} evaluates additional empirical
scale-matching controls on the same checkpoints.

\subsection{Regression}

\begin{figure}[ht]
\centering
\includegraphics[width=\linewidth]{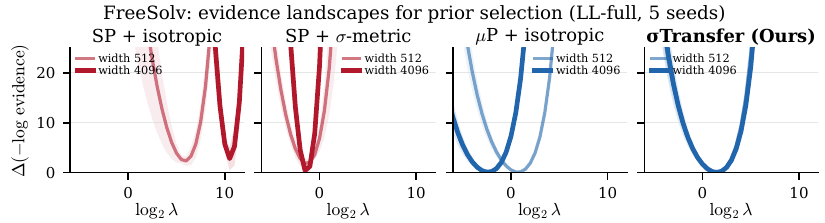}
\caption{Evidence as a function of prior precision $\lambda$ on
FreeSolv. Light and dark curves show proxy (width $512$) and target
(width $4096$); bands are mean $\pm$ SEM over five seeds. Under SP the
landscape shifts with width, so the proxy's optimum is far from the
target's. Under $\sigmat$ the two curves nearly overlap.
Table~\ref{tab:molecular} reports numerical results for all three
datasets.}
\label{fig:molecular-evidence-fourarm}
\end{figure}

\begin{table}[ht]
\centering
\caption{Prior-precision transfer on three regression datasets (proxy
width $512$, target width $4096$, full-covariance last-layer Laplace,
five paired seeds, mean $\pm$ s.d.). $|\Delta\NLL|$ scaled by $10^3$;
lower is better for all columns. See Appendix~\ref{app:metrics} for
metric definitions.}
\label{tab:molecular}
\vspace{0.15em}

\scriptsize
\setlength{\tabcolsep}{2.0pt}
\renewcommand{\arraystretch}{0.98}

\begin{adjustbox}{max width=\linewidth}
\begin{tabular}{@{}l|ccc|ccc|ccc@{}}
  \toprule
  \multirow{2}{*}{\textbf{Method}}
    & \multicolumn{3}{c|}{\textbf{ESOL}}
    & \multicolumn{3}{c|}{\textbf{FreeSolv}}
    & \multicolumn{3}{c}{\textbf{Lipophilicity}} \\
  \cmidrule(lr){2-4} \cmidrule(lr){5-7} \cmidrule(l){8-10}
    & $\lvert\Delta\lambda\rvert$ & $\lvert\Delta\NLL\rvert$ & $\mathrm{KL}_{4096\|512}$
    & $\lvert\Delta\lambda\rvert$ & $\lvert\Delta\NLL\rvert$ & $\mathrm{KL}_{4096\|512}$
    & $\lvert\Delta\lambda\rvert$ & $\lvert\Delta\NLL\rvert$ & $\mathrm{KL}_{4096\|512}$ \\
  \midrule
  SP + isotropic
    & $1.5_{\pm0.4}$ & $41.8_{\pm30.7}$ & $0.208_{\pm0.073}$
    & $5.0_{\pm2.2}$ & $452.8_{\pm429.5}$ & $0.079_{\pm0.055}$
    & $4.0_{\pm2.7}$ & $140.2_{\pm119.7}$ & $0.173_{\pm0.027}$ \\
  SP + metric
    & $1.5_{\pm0.5}$ & $8.40_{\pm6.59}$ & $0.208_{\pm0.071}$
    & $1.2_{\pm1.4}$ & $9.65_{\pm10.95}$ & $0.079_{\pm0.055}$
    & $2.1_{\pm1.2}$ & $3.39_{\pm1.93}$ & $0.173_{\pm0.028}$ \\
  $\mup$ + isotropic
    & $3.0_{\pm1.4}$ & $0.947_{\pm0.566}$ & $\mathbf{0.167}_{\pm0.087}$
    & $3.1_{\pm0.5}$ & $1.564_{\pm0.825}$ & $\mathbf{0.008}_{\pm0.004}$
    & $3.2_{\pm1.0}$ & $0.094_{\pm0.024}$ & $\mathbf{0.057}_{\pm0.058}$ \\
  \textbf{$\sigmat$ (ours)}
    & $\mathbf{0.9}_{\pm0.8}$ & $\mathbf{0.548}_{\pm0.412}$ & $\mathbf{0.167}_{\pm0.087}$
    & $\mathbf{0.5}_{\pm0.4}$ & $\mathbf{0.461}_{\pm0.430}$ & $\mathbf{0.008}_{\pm0.004}$
    & $\mathbf{0.4}_{\pm0.2}$ & $\mathbf{0.033}_{\pm0.019}$ & $\mathbf{0.057}_{\pm0.058}$ \\
  \bottomrule
\end{tabular}
\end{adjustbox}
\vspace{-0.5em}
\end{table}

We test Regime~A on
ESOL~\citep{delaney2004esol},
FreeSolv~\citep{mobley2014freesolv},
and Lipophilicity~\citep{wu2018moleculenet} under scaffold shift,
transferring an evidence-selected precision from width $512$ to $4096$
using full-covariance last-layer Laplace and five paired seeds
(full protocol in Appendix~\ref{app:regression-details}).
On FreeSolv the SP + isotropic evidence landscape shifts markedly
across widths, whereas the $\sigmat$ curves nearly overlap
(Figure~\ref{fig:molecular-evidence-fourarm}). Across all three
datasets $\sigmat$ gives the smallest selection gap $|\Delta\lambda|$
($0.4$--$0.9$ bits, against $1.5$--$5.0$ for SP + isotropic) and a
predictive transfer gap $|\Delta\NLL|$ orders of magnitude smaller
(Table~\ref{tab:molecular}). We also report KL divergence to measure
agreement between the proxy's and target's predictive distributions
(Appendix~\ref{app:metrics}).

\subsection{Image classification}

\begin{table*}[th]
\centering
\caption{Prior-precision transfer on MNIST and Fashion-MNIST across
three last-layer Laplace approximations (proxy width $128$, target
width $4096$, validation-NLL-selected $\lambda$, 10 paired seeds,
mean $\pm$ s.d.). $|\Delta\NLL|$ scaled by $10^2$; lower is better
for all columns. See Appendix~\ref{app:metrics} for metric definitions
and Appendix~\ref{app:evidence-transfer} for evidence-selected
results.}\label{tab:mnist_fmnist_approx_transfer}
\footnotesize \setlength{\tabcolsep}{2.5pt}
\begin{adjustbox}{width=0.95\textwidth}
\begin{tabular}{l|cc|cc|cc|cc|cc|cc}
\toprule
\multirow{3}{*}{\textbf{Method}}
  & \multicolumn{6}{c|}{\textbf{MNIST}} & \multicolumn{6}{c}{\textbf{FMNIST}} \\
\cmidrule(lr){2-7} \cmidrule(lr){8-13}
  & \multicolumn{2}{c|}{\textbf{LL full}} & \multicolumn{2}{c|}{\textbf{LL diag}}
  & \multicolumn{2}{c|}{\textbf{LL KFAC}}
  & \multicolumn{2}{c|}{\textbf{LL full}} & \multicolumn{2}{c|}{\textbf{LL diag}}
  & \multicolumn{2}{c}{\textbf{LL KFAC}} \\
  & $|\Delta\lambda|$ & $|\Delta\NLL|$ & $|\Delta\lambda|$ & $|\Delta\NLL|$
  & $|\Delta\lambda|$ & $|\Delta\NLL|$
  & $|\Delta\lambda|$ & $|\Delta\NLL|$ & $|\Delta\lambda|$ & $|\Delta\NLL|$
  & $|\Delta\lambda|$ & $|\Delta\NLL|$ \\
\midrule 
SP + isotropic & $1.75_{\pm.26}$ & $9.8_{\pm4.6}$ & $0.45_{\pm.28}$ & $0.8_{\pm0.9}$ & $1.50_{\pm.24}$ & $6.9_{\pm3.2}$
  & $1.55_{\pm.28}$ & $9.0_{\pm4.1}$ & {\boldmath$0.25_{\pm.35}$} & $0.9_{\pm1.2}$ & $1.85_{\pm.24}$ & $13.2_{\pm3.2}$ \\
SP + metric & $3.25_{\pm.26}$ & $9.6_{\pm2.1}$ & $4.55_{\pm.28}$ & $14.8_{\pm1.9}$ & $3.50_{\pm.24}$ & $10.7_{\pm1.9}$
  & $3.45_{\pm.28}$ & $39.0_{\pm6.3}$ & $5.25_{\pm.35}$ & $67.2_{\pm7.3}$ & $3.15_{\pm.24}$ & $33.3_{\pm5.5}$ \\
$\mup$ + isotropic & $4.75_{\pm.26}$ & $15.7_{\pm3.0}$ & $5.30_{\pm.26}$ & $20.6_{\pm3.2}$ & $4.75_{\pm.26}$ & $16.5_{\pm2.7}$
  & $4.60_{\pm.21}$ & $34.9_{\pm2.8}$ & $4.55_{\pm.28}$ & $47.2_{\pm5.0}$ & $4.45_{\pm.28}$ & $30.5_{\pm2.6}$ \\
\textbf{$\sigmat$ (ours)} & {\boldmath$0.25_{\pm.26}$} & {\boldmath$0.2_{\pm0.3}$} & {\boldmath$0.30_{\pm.26}$} & {\boldmath$0.1_{\pm0.2}$} & {\boldmath$0.25_{\pm.26}$} & {\boldmath$0.2_{\pm0.3}$}
  & {\boldmath$0.40_{\pm.21}$} & {\boldmath$0.4_{\pm0.4}$} & $0.45_{\pm.28}$ & {\boldmath$0.6_{\pm0.6}$} & {\boldmath$0.55_{\pm.28}$} & {\boldmath$0.3_{\pm0.3}$} \\
\bottomrule
\end{tabular}
\end{adjustbox}
\end{table*}

\paragraph{Prior transfer (Regime~A).}
Table~\ref{tab:mnist_fmnist_approx_transfer} reports both estimands on
MNIST~\citep{lecun1998gradient} and
Fashion-MNIST~\citep{xiao2017fashion}
across three last-layer Laplace approximations (full, diagonal, KFAC;
see Appendix~\ref{app:curvature-approximations} for definitions).
Among the four methods in Table~\ref{tab:mnist_fmnist_approx_transfer},
$\sigmat$ has the smallest predictive transfer gap $|\Delta\NLL|$ in
every column, and its worst case ($0.6 \times 10^{-2}$, FMNIST
LL-diagonal) still beats every other method on every approximation and
dataset. Its selection gap $|\Delta\lambda|$ is smallest in five of the
six columns and never exceeds $0.55$ bits, against up to $1.85$ for
SP + isotropic; the exception is FMNIST LL-diagonal, where SP +
isotropic selects marginally closer.
The same selection-gap ordering holds under plug-in evidence selection,
although its optimum is not always test-NLL-optimal
(Appendix~\ref{app:evidence-transfer}).
Table~\ref{tab:efficiency} reports the compute: proxy-side selection is
${\sim}5000\times$ faster than the target-side sweep for the full
last-layer posterior and $94\times$ for all-parameter diagonal Laplace,
and after transfer the target builds one posterior at the transferred
precision for a test-NLL degradation of only $0.002$.
When the selection criterion is OOD detection instead of validation
NLL, $\sigmat$ also transfers: at width $4096$ the proxy-selected
precision achieves $0.950$ AUROC on held-out OOD data, matching
target-side tuning, whereas SP loses $0.072$ and $\mup$ with an
isotropic prior loses $0.327$. The match holds at every width above
the proxy (Appendix~\ref{app:ood-transfer}).


\paragraph{Decision transfer (Regime~B).}
\label{sec:mnist-al}
Figure~\ref{fig:al-pair} shows a sequential active-learning experiment
on MNIST ($1{,}000$ labelled points, $50$ rounds). Selecting data with
the proxy produces nearly the same learning curve on the target as
using the target's own posterior, with final accuracy only $0.003$
below the target reference and a paired area-under-the-learning-curve
difference covering zero; under SP the curves diverge, meaning decisions are not
transferable. Per-round Spearman
correlation between proxy and target EPIG scores~\citep{smith2023prediction} is $\bar\rho=0.75$
under $\sigmat$, against $0.07$ for SP + isotropic
(Figure~\ref{fig:al-pair}, right panel), consistent with our
theoretical guarantees (Section~\ref{sec:proof-goals}). On a fixed
candidate pool the agreement is higher still:
$\rho=0.854\pm0.025$ under $\sigmat$, against $0.669$ for SP. The same
agreement ordering holds on Fashion-MNIST; we claim decision transfer
on both datasets and downstream utility on MNIST.
Appendix~\ref{app:al-details} gives full learning curves, AULC
contrasts, and per-dataset acquisition-agreement tables.

\begin{figure}[t]
\centering
\includegraphics[width=\linewidth]{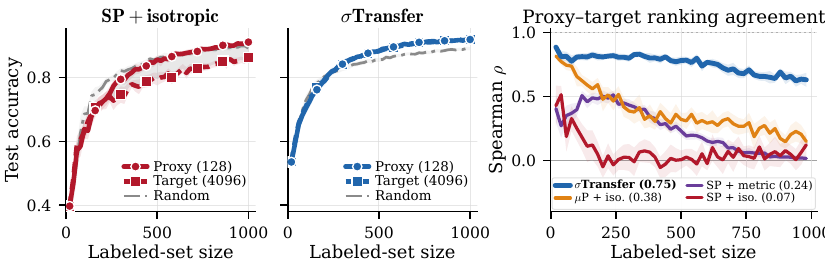}
\caption{Sequential active learning on MNIST with proxy acquisition
(EPIG, proxy width $128$, target width $4096$, 20 seeds, bands: 95\%
CI). Left and centre: target test accuracy under proxy, target-reference,
and random acquisition. Under $\sigmat$ the proxy and target curves
nearly overlap; under SP + isotropic the target's own acquisition falls
below random. Right: per-round Spearman $\rho$ between proxy and target
rankings. See Appendix~\ref{app:al-details} for Fashion-MNIST and
AULC contrasts.}
\label{fig:al-pair}
\end{figure}

\subsection{Language models}

We train a ladder of decoder-only GPT models at widths
$128$, $256$, $512$, $1024$, and $2048$ ($8$ layers, $8$ heads) on
$10^{9}$ FineWeb-Edu tokens~\citep{penedo2024fineweb} under both SP and $\mup$. Training and
posterior details are in
Appendix~\ref{app:ood-gate}.

\paragraph{Prior transfer (Regime~A).}
\label{sec:llm-lambda}
Figure~\ref{fig:llm-lambda} (left) shows the evidence-selected prior
precision across this ladder (ten downstream tasks). Under $\sigmat$ the optimum converges with width; under SP it
drifts about $0.65$ bits per width doubling without saturating, so no
proxy width is safe. The proxy's choice under $\sigmat$ lands
$0.65\pm0.34$ bits from the target's own and degrades target test NLL
by only $0.0014$ on average, compared to $2.65\pm0.63$ bits and a
$0.013$ degradation under SP (nearly an order of magnitude larger).
Three of ten tasks have absolute target-NLL changes below $10^{-4}$;
the largest change is $0.0082$ (Figure~\ref{fig:llm-lambda}, middle).
On the public u-$\mup$ pairs~\citep{blake2025u} ($1$B$\to$$7$B),
seven of ten tasks select the same precision as the target, and eight have
absolute target-NLL changes below $10^{-4}$ (Figure~\ref{fig:llm-lambda},
right; Appendices~\ref{app:llm-lambda} and~\ref{app:umup-details}).

\begin{figure}[ht]
\centering
\includegraphics[width=\linewidth]{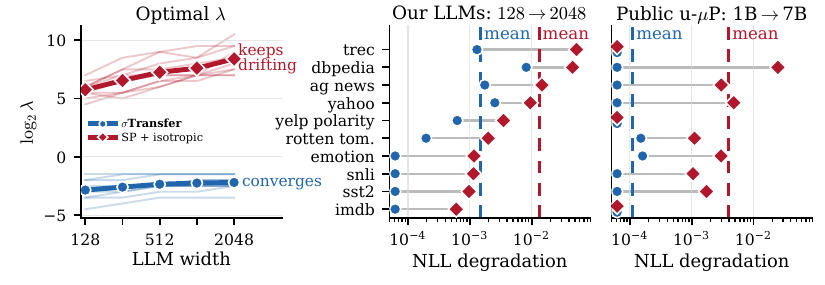}
\caption{A prior precision chosen on a proxy transfers under
$\sigmat$, not under SP. Left: evidence-based selection on
in-distribution training data, per task (thin) and task mean (thick).
Middle and right: target NLL at the proxy's $\lambda$ minus the
target's own (lower is better), same ten tasks.}
\label{fig:llm-lambda}
\end{figure}

\paragraph{Decision transfer (Regime~B).}
\label{sec:llm-ood-gate}
Figure~\ref{fig:llm-decisions} runs sequential EPIG on
AG News~\citep{zhang2015character}: proxy
widths $128$ to $1024$ acquire labels for a width-$2048$ target, all at
a common, untuned $\lambda=1$. $\sigmat$ has the highest mean Spearman correlation
between proxy and target rankings of which data to label, across all
proxy widths and initial label budgets. For $128\!\to\!2048$ transfer
with $5{,}000$ initial labels, target accuracy is lower by $0.0014$ with
proxy-selected labels than with target-selected labels
(Table~\ref{tab:efficiency}). The bottom row trains
the larger model on the datapoints the smaller one selected, and
measures the absolute test-NLL gap to training it on the target's own
selections, an oracle that requires the very target posterior we are
avoiding. $\sigmat$ has the smallest mean absolute test-NLL gap in
nine of twelve settings.
Appendix~\ref{app:ood-gate} reports a second Regime-B test, where a
width-$128$ proxy gate lets the target answer by MAP alone: it abstains
on $86$--$88\%$ of unseen OOD prompts while serving $96\%$ of
in-distribution traffic.

\begin{figure}[t]
\centering
\includegraphics[width=\linewidth]{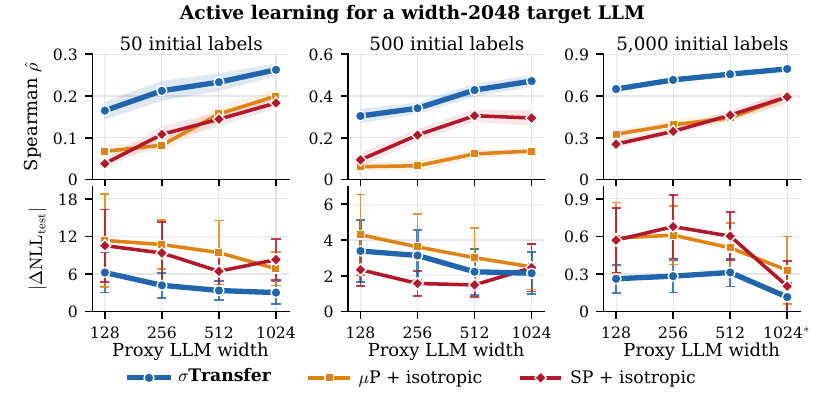}
\caption{Sequential active learning on AG News, starting with $50$,
$500$ or $5{,}000$ labelled examples and acquiring $250$ additional
labels using EPIG. Top: Spearman correlation between the proxy's and target's
rankings of unlabelled data by informativeness (EPIG), averaged over
acquisition steps. Bottom: mean absolute test-NLL gap when the target is trained on
proxy-selected rather than target-selected labels (ticks $\times 10^{-3}$;
lower is better). All methods use $\lambda=1$.
Means and $95\%$ CIs over ten seeds; the starred setting uses five.}
\label{fig:llm-decisions}
\end{figure}

\subsection{Computational savings and limitations}

\begin{table}[ht]
\centering\small
\setlength{\tabcolsep}{4pt}
\resizebox{\linewidth}{!}{%
\begin{tabular}{@{}llllrrrl@{}}
\toprule
\textbf{Task} & \textbf{Dataset} & \textbf{Transfer} & \textbf{Approx.} & \textbf{Proxy (s)} &
\textbf{Target (s)} & \textbf{\color{cKernel}Speedup} & \textbf{Degradation} \\
\midrule
$\lambda$ sweep     & MNIST      & $128\!\to\!4096$  & LL full      & $1.5$  & $7{,}560$  & ${\color{cKernel}{\sim}5000\times}$ & $0.002$ NLL \\
OOD-aware $\lambda$ & MNIST      & $128\!\to\!4096$  & all $+$ diag & $4.1$  & $384$      & ${\color{cKernel}94\times}$         & $0.001$ AUROC \\
evidence search     & LLM (ours) & $128\!\to\!2048$  & LL full      & $3.5$  & $25$       & ${\color{cKernel}{\sim}7\times}$    & $0.0014$ NLL \\
evidence search     & LLM (public) & 1B$\to$7B       & LL full      & $4$    & $9$        & ${\color{cKernel}{\sim}2.3\times}$  & $0$ NLL \\
EPIG acquisition    & AG News    & $128\!\to\!2048$  & LL full      & $64.9$ & $187.6$    & ${\color{cKernel}{\sim}3\times}$     & $0.0014$ ACC \\
evidence search & PenDigits & $32\!\to\!128$ & all $+$ full & $0.68$ & $14.95$ & ${\color{cKernel}{\sim}22\times}$ & $-0.007$ NLL \\
evidence search & letter    & $32\!\to\!128$ & all $+$ full & $0.77$ & $24.74$ & ${\color{cKernel}{\sim}32\times}$ & $0.078$ NLL \\

\bottomrule
\end{tabular}}
\caption{$\sigmat$ does the expensive selection on the proxy rather
than the target. Speedup is target $\div$ proxy selection time;
deployment still fits one target posterior, so the end-to-end saving
is capped by the $61$-point grid. Degradation is the resulting change
in the stated metric, positive meaning worse. EPIG
times average 10 seeds; all times exclude common MAP training and
one-off frozen-trunk feature extraction. For the public 1B$\to$7B pair,
the row reports AG News, whose speedup is the median across nine timed
tasks, calculated from the rounded times in
Table~\ref{tab:llm-lambda-public}. LL\,=\,last-layer.}
\label{tab:efficiency}
\end{table}

Table~\ref{tab:efficiency} shows the saving scaling with the cost of
the target-side procedure: up to ${\sim}5000\times$ on MNIST, where the
target would otherwise sweep a full last-layer Laplace, and
a median ${\sim}2.3\times$ on the public 1B$\to$7B
pair~\citep{blake2025u}, ranging from $1\times$ to ${\sim}330\times$ across
nine timed tasks. It
is far smaller where that procedure is already cheap, and $\sigmat$
does not always win: on FMNIST under last-layer diagonal curvature,
SP + isotropic selects a marginally closer precision.
Appendix~\ref{app:tabular-classification} gives the PenDigits and Letter setup.

\section{Conclusion}
\label{sec:conclusion}

We introduced $\sigmat$, a method that uses a small (cheap) network to make uncertainty estimation and uncertainty-based decision-making cheaper for a much larger (expensive) network. It supports two forms of transfer. In \emph{prior-precision transfer}, the precision search is performed on the small network and the selected precision is transferred zero-shot to the much larger network, without any target-side search. In \emph{decision transfer}, the small network's uncertainty estimates are used directly to select data, flag out-of-distribution inputs, or decide when the larger network should abstain. We justify both forms of transfer theoretically through a series of results showing accurate transfer across network widths. Empirically, $\sigmat$ preserves predictive performance while reducing the cost of uncertainty-sensitive tasks by up to several orders of magnitude. We discuss the limitations of our approach and directions for future work in Appendix~\ref{app:limitations}.

\subsection*{AI use statement}
We derived the main theoretical results. For
Theorem~\ref{thm:prior-kernel-stability}, generative AI tools helped us
navigate the Tensor Programs
literature~\citep{yang2019wide,yang2020tensor,yang2021tensor,yang2023tensor}
and identify its key tools, including the Master
Theorem~\citep[Theorem~2.6.10]{yang2023tensor}, which we then instantiated
for our setting. AI tools also assisted with writing up proofs and some
derivation steps, implementing the method and experiments, analysing
results, literature search, feedback on the experiments and manuscript,
and drafting and editing. We reviewed all AI-assisted work, checked
numerical claims against saved outputs and theoretical statements against
their proofs, and take full responsibility for the content of this paper.

\subsection*{Reproducibility statement}
The appendices provide the parameterization conventions, implementation
details, assumptions, and complete proofs used in this work. Experimental
protocols report model widths, random seeds, posterior approximations,
hyperparameter-selection rules, evaluation metrics, and compute measurements.
Code, configuration files, and result manifests will be included in the
supplementary material.

\bibliographystyle{iclr2027_conference}
\bibliography{references}

\appendix
\clearpage

\begin{table}[ht]
\centering
\caption{Roadmap to the appendix.}
\label{tab:appendix-roadmap}
\begingroup
\small
\renewcommand{\arraystretch}{1.15}
\begin{tabularx}{\textwidth}{@{}p{0.34\textwidth}X@{}}
\toprule
\textbf{Location} & \textbf{Contents} \\
\midrule
Appendix~\ref{app:limitations} & Limitations and directions for future work. \\
Appendix~\ref{app:metrics} & Metrics for prior-precision and decision transfer. \\
Appendix~\ref{app:experimental-details} & Curvature approximations, the geometric mechanism, regression setup, and additional normalization baselines. \\
Appendix~\ref{app:evidence-transfer} & Evidence-selected precision transfer on MNIST and Fashion-MNIST. \\
Appendix~\ref{app:ood-transfer} & Additional results under distribution shift. \\
Appendix~\ref{app:al-details} & Sequential active-learning setup and results. \\
Appendix~\ref{app:ood-gate} & OOD detection and abstention setup, calibration, and per-dataset results. \\
Appendix~\ref{app:llm-lambda} & Per-dataset precision transfer and the neural $g$-prior comparison. \\
Appendix~\ref{app:umup-details} & External validation on public u-$\mup$/SP model pairs. \\
Appendix~\ref{app:sn} & Coordinate definition, MLP derivation, block scales, and implementation details. \\
Appendix~\ref{app:proofs} & Formal proofs, structured curvature, and covered decision rules. \\
\bottomrule
\end{tabularx}
\endgroup
\end{table}
\clearpage

\section{Limitations and future work}
\label{app:limitations}

Theorem~\ref{thm:prior-kernel-stability} establishes kernel convergence
for fixed-depth ReLU MLPs under finite-step training and regularity
assumptions. Establishing the corresponding convergence conditions for
our Transformer settings and determining the proxy width needed for a
given transfer accuracy remain open.

Future work could extend $\sigmat$ to applications in which reliable
uncertainty is especially important, including medical prediction and
generative models such as diffusion models. Another important direction
is to move beyond width-only transfer and derive prior-scaling rules for
models whose width and depth increase together. This would extend the
broader $\mup$ scaling programme to Laplace priors and could enable prior
precision and uncertainty-based decisions to transfer across entire
model families, rather than only across networks of different widths.

\section{Evaluation metrics}
\label{app:metrics}

All experiments evaluate transfer fidelity: how well does a quantity
chosen or computed at proxy width reproduce the corresponding quantity
at target width? The metrics below are grouped by regime.
Where a $95\%$ confidence interval is reported, it is an interval for
the mean using the Student-$t$ critical value with $n-1$ degrees of
freedom, where $n$ is the number of seeds. Tables and figures that
report $\pm$ give sample standard deviations unless stated otherwise.
For language-model prior transfer, $\pm$ denotes sample standard deviation
across tasks. Our models use one pretraining seed per width and
parametrization; the public models use one checkpoint per size and
parametrization.

\paragraph{Regime~A (prior transfer).}
Each method selects $\lambda^\star_n$ independently at every width $n$
by minimizing validation NLL or maximizing the fixed-checkpoint plug-in Laplace evidence score
over a fixed grid $\log_2\lambda\in[-15,15]$.
Target-side selection is performed only to define the reference optimum
for these metrics. In zero-shot deployment, the proxy-selected value is
reused and the target constructs its posterior once at that value.

\emph{Selection gap.}
The absolute distance between the proxy and target optima, measured in
bits of log-precision:
\begin{equation}
  \lvert\Delta\lambda\rvert
  \;=\;
  \lvert\log_2\lambda^\star_{\mathrm{proxy}}
       -\log_2\lambda^\star_{\mathrm{target}}\rvert.
  \label{eq:selection-gap}
\end{equation}
A gap of zero means both widths select the same grid point. We compute
the absolute value per seed or task and report the mean $\pm$ sample
standard deviation.

\emph{Predictive transfer gap.}
The absolute difference in target test NLL between using the
proxy-selected precision and the target's own optimum:
\begin{equation}
  \lvert\Delta\NLL\rvert
  \;=\;
  \bigl\lvert
    \NLL_{\mathrm{target}}(\lambda^\star_{\mathrm{proxy}})
    -\NLL_{\mathrm{target}}(\lambda^\star_{\mathrm{target}})
  \bigr\rvert.
  \label{eq:predictive-transfer-gap}
\end{equation}
This measures the predictive discrepancy induced by trusting the proxy's choice.
As with the selection gap, absolute values are taken per seed before
averaging. In the regression tables the values are scaled by
$10^3$; in the classification tables by $10^2$.

\emph{Cross-width predictive KL.}
A post-hoc diagnostic used in the regression experiments:
\begin{equation}
  \mathrm{KL}_{t\|p}
  \;=\;
  \mathbb{E}_{x}\,
  \mathrm{KL}\!\bigl[
    p_{\mathrm{target}}^{\lambda^\star_{\mathrm{target}}}(\cdot\mid x)
    \;\big\|\;
    p_{\mathrm{proxy}}^{\lambda^\star_{\mathrm{proxy}}}(\cdot\mid x)
  \bigr],
  \label{eq:pred-kl}
\end{equation}
in nats (including observation noise). This measures overall
predictive-distribution agreement, not action agreement or a
transferred-minus-local performance difference. Because methods using isotropic and width-normalized priors share
trained checkpoints within each parametrization family, their KL values
are expected to be nearly identical.

\emph{OOD AUROC.}
When the selection criterion is OOD detection rather than validation
NLL, we report the area under the receiver operating characteristic
curve for distinguishing in-distribution from out-of-distribution test
inputs at the selected precision. The OOD-AUROC
\emph{regret} is the gap between the transferred and target-selected
AUROC values.

\paragraph{Regime~B (decision transfer).}

\emph{Spearman rank correlation~($\rho$).}
The rank correlation between the proxy's and target's acquisition scores
(e.g.\ EPIG values) over a common candidate pool. Higher $\rho$ means
the proxy ranks candidates in the same order as the target. Reported per
round in sequential active learning and on a single fixed pool for
one-shot acquisition.

\emph{Top-$K$ overlap.}
The fraction of the proxy's top-$K$ selected candidates that also
appear in the target's top-$K$. It complements rank correlation by
focusing on the items that would actually be chosen.

\emph{Area under the learning curve (AULC).}
In sequential active learning, we integrate the test accuracy over
acquisition rounds. The proxy-AULC is compared to the target-reference
AULC and to random acquisition to assess whether proxy-guided selection
translates into downstream utility.

\emph{OOD abstention gate.}
In the language-model gating experiment (Section~\ref{sec:llm-ood-gate}),
the proxy's BALD score~\citep{houlsby2011bayesian} is thresholded to decide whether the target
should abstain. We report the fraction of OOD prompts correctly
abstained on (OOD-accuracy) and the fraction of in-distribution traffic
served, calibrated on one OOD source and evaluated on a different,
unseen one.

\section{Experimental details}
\label{app:experimental-details}

Our main benchmark is a two-hidden-layer ReLU MLP on flattened MNIST
($d=784$), with hidden width
$n\in\{128,256,512,1024,2048,4096\}$. We compare \textbf{SP}, standard
parametrization with an isotropic stored-coordinate prior;
\textbf{$\mup$ + isotropic}, $\mup$ training with the same
stored-coordinate prior; and \textbf{$\sigmat$}, $\mup$ training with the
width-normalized prior (ours). The two $\mup$ methods share the same
trained checkpoints, isolating the effect of the prior geometry. All models reach $100\%$ training accuracy; at
width $4096$, mean validation accuracy is $0.950$ for SP and $0.934$ for the
$\mup$ models, so any SP uncertainty failure below is not an optimization
failure.

For MNIST and Fashion-MNIST, we sample $5{,}000$ training and $1{,}000$
validation examples from the official training split, and $2{,}000$
test examples from the official test split. Flattened pixels lie in
$[0,1]$, and the classification head has no bias. We use Adam for SP and
MuAdam for $\mup$, with the same base learning rate $10^{-3}$ across
widths and parametrizations, for $2{,}000$ full-batch updates without
early stopping or weight decay.

Unless stated otherwise, we use a GGN Laplace approximation and search a fixed
$61$-point grid, $\log_2\lambda\in[-15,15]$. The width-$128$ model is the
proxy; target widths use the proxy-selected numerical $\lambda$ unchanged.
Each target constructs one posterior at this transferred value; target-side
sweeps are run only to obtain reference optima and predictive-transfer-gap metrics.
We report $10$ seeds; numerical summaries use sample standard deviations. We select by validation
NLL in the primary experiment and retain the fixed-checkpoint plug-in Laplace evidence score as a diagnostic
of criterion dependence. The implementation is verified against hand-built references: full-curvature
posterior precisions agree to relative Frobenius error $\sim 10^{-8}$.

\subsection{Curvature approximations}
\label{app:curvature-approximations}

The GGN Hessian $H$ can be stored and inverted in several ways, offering
different accuracy--compute trade-offs. The \emph{full} approximation stores the entire
$P\times P$ matrix ($O(P^2)$ memory, $O(P^3)$ inversion), giving the
exact linearized posterior. The \emph{diagonal} approximation retains
only the diagonal of $H$ ($O(P)$ memory), discarding all off-diagonal
correlations between parameters. The \emph{Kronecker-factored} (KFAC)
approximation~\citep{martens2015optimizing,ritter2018scalable} stores a
Kronecker product of two factors for each layer block, capturing within-layer
structure at reduced computational cost. In all three
cases the prefix \emph{last-layer} (LL) means the approximation is
applied only to the parameters of the final linear layer, with earlier
layers frozen at their MAP values; this reduces $P$ from the full
parameter count to the output-layer dimension alone.

\subsection{The geometric mechanism}

We isolate the mechanism on two moons, selecting $\lambda$ on validation
data and then measuring kernel and uncertainty agreement on a disjoint
held-out pool. Relative kernel drift from width $128$ to $4096$ falls from
$20.4\pm3.9$ under standard practice to $0.64\pm0.11$ under $\sigmat$; the
selected precision moves $2.9\pm0.1$ bits versus $-0.5\pm0.8$ bits, and
held-out uncertainty-rank agreement rises from $0.22\pm0.05$ to
$0.69\pm0.04$. The geometric correction therefore changes not only the
scale of the prior kernel but the posterior quantities used downstream.

On MNIST, linear CKA~\citep{kornblith2019similarity} is about $0.999$ for every method, indicating that in this
last-layer diagnostic the $\mup$ width-dependent head scale stabilizes
kernel magnitude rather than changing its shape. Kernel stability alone does
not imply posterior stability: the exact last-layer posteriors of
Section~\ref{sec:results} test whether it survives conditioning.

\subsection{Regression under scaffold shift}
\label{app:regression-details}

The regression benchmark is a three-hidden-layer tanh MLP on binary
Morgan fingerprints~\citep{rogers2010extended} ($d=2048$, radius $2$,
computed with RDKit), with a bias-free readout and base width $128$.
Each dataset uses the OGB scaffold split~\citep{hu2020open}
of its MoleculeNet original (\texttt{ogbg-molesol},
\texttt{ogbg-molfreesolv}, \texttt{ogbg-mollipo}), so training,
validation, and test molecules occupy disjoint scaffolds. Inputs are
left unstandardized, so a bit absent from the training scaffolds stays
$0/1$ when it first appears at validation or test time; targets use
training-split standardization.

We train for $20{,}000$ iterations at batch size $100$ and learning
rate $10^{-3}$, validating every $250$ iterations. The width-$512$
model is the proxy and width $4096$ the target. We report five seeds,
shared across all four arms; as in the classification benchmark the two
$\mup$ arms share trained checkpoints, so their difference isolates the
prior geometry.

\subsection{Additional normalization baselines}
\label{app:normalization-baselines}

We construct empirical controls that match the average prior-kernel magnitude
across widths. Their scales use trained features; $\sigmat$ derives its
scales from the architecture and parametrization.

For proxy width $n_0$ and training inputs $D$, set
\[
r_n=\frac{1}{|D|}\sum_{x\in D}\|J_{\mathrm{head},n}(x)\|_F^2,
\qquad
S_n^{\mathrm{meas}}=s_0\frac{r_{n_0}}{r_n}I,
\]
where $J_{\mathrm{head},n}$ is the Jacobian with respect to stored readout
weights, including any forward multiplier. The prior covariance is
$\lambda^{-1}S_n^{\mathrm{meas}}$.
We take $s_0=1$ for SP and the proxy's architectural covariance scale for
$\mup$. These controls use training features from both the proxy and target,
without a target precision search. They match the average prior-kernel trace;
this alone does not ensure kernel or posterior convergence. With one bias-free
readout block, global and blockwise normalization coincide, giving one
``+ empirical'' control per parametrization. ``SP + metric'' uses the
width-normalized prior under SP.

Tables~\ref{tab:regression-normalization}
and~\ref{tab:classification-normalization} extend
Tables~\ref{tab:molecular} and~\ref{tab:mnist_fmnist_approx_transfer},
respectively, with the empirical controls.
For matching seeds within each parametrization, priors share checkpoints,
data splits, and likelihood settings. All use the same $61$-point grid
$\log_2\lambda\in[-15,15]$ with half-bit spacing; selection criteria and seed
counts are given in the captions.

\begin{table}[ht]
\centering
\caption{Additional normalization baselines for regression: width
$512\to4096$, full-covariance last-layer Laplace, and evidence-selected
prior precision. Results are mean $\pm$ sample standard deviation over
five seeds for the methods reproduced from Table~\ref{tab:molecular}
and ten seeds for the empirical controls. $|\Delta\NLL|$ is multiplied
by $10^3$; other metrics follow Table~\ref{tab:molecular}. Bold marks the smallest reported mean
precision and NLL transfer gaps; lower is better.}
\label{tab:regression-normalization}
\begingroup
\scriptsize
\setlength{\tabcolsep}{2.0pt}
\renewcommand{\arraystretch}{1.05}
\begin{adjustbox}{max width=\linewidth}
\begin{tabular}{@{}l|ccc|ccc|ccc@{}}
\toprule
\multirow{2}{*}{\textbf{Method}}
 & \multicolumn{3}{c|}{\textbf{ESOL}}
 & \multicolumn{3}{c|}{\textbf{FreeSolv}}
 & \multicolumn{3}{c}{\textbf{Lipophilicity}} \\
\cmidrule(lr){2-4}\cmidrule(lr){5-7}\cmidrule(l){8-10}
 & $|\Delta\lambda|$ & $|\Delta\NLL|$ & $\mathrm{KL}_{4096\|512}$
 & $|\Delta\lambda|$ & $|\Delta\NLL|$ & $\mathrm{KL}_{4096\|512}$
 & $|\Delta\lambda|$ & $|\Delta\NLL|$ & $\mathrm{KL}_{4096\|512}$ \\
\midrule
SP + isotropic & $1.5_{\pm0.4}$ & $41.8_{\pm30.7}$ & $0.208_{\pm0.073}$ & $5.0_{\pm2.2}$ & $452.8_{\pm429.5}$ & $0.079_{\pm0.055}$ & $4.0_{\pm2.7}$ & $140.2_{\pm119.7}$ & $0.173_{\pm0.027}$ \\
SP + metric & $1.5_{\pm0.5}$ & $8.40_{\pm6.59}$ & $0.208_{\pm0.071}$ & $1.2_{\pm1.4}$ & $9.65_{\pm10.95}$ & $0.079_{\pm0.055}$ & $2.1_{\pm1.2}$ & $3.39_{\pm1.93}$ & $0.173_{\pm0.028}$ \\
$\mup$ + isotropic & $3.0_{\pm1.4}$ & $0.947_{\pm0.566}$ & $0.167_{\pm0.087}$ & $3.1_{\pm0.5}$ & $1.564_{\pm0.825}$ & $0.008_{\pm0.004}$ & $3.2_{\pm1.0}$ & $0.094_{\pm0.024}$ & $0.057_{\pm0.058}$ \\
\textbf{$\sigmat$ (ours)} & $\mathbf{0.9}_{\pm0.8}$ & $\mathbf{0.548}_{\pm0.412}$ & $0.167_{\pm0.087}$ & $\mathbf{0.5}_{\pm0.4}$ & $\mathbf{0.461}_{\pm0.430}$ & $0.008_{\pm0.004}$ & $\mathbf{0.4}_{\pm0.2}$ & $\mathbf{0.033}_{\pm0.019}$ & $0.057_{\pm0.058}$ \\
\addlinespace[2pt]
SP + empirical & $1.3_{\pm0.7}$ & $17.6_{\pm1.65}$ & $0.250_{\pm0.27}$ & $3.0_{\pm2.2}$ & $35.7_{\pm8.2}$ & $0.10_{\pm0.22}$ & $2.8_{\pm2.8}$ & $2.5_{\pm2.0}$ & $0.200_{\pm0.26}$ \\
$\mup$ + empirical & $1.1_{\pm0.9}$ & $0.650_{\pm0.412}$ & $0.200_{\pm0.26}$ & $0.7_{\pm0.537}$ & $0.650_{\pm0.475}$ & $0.05_{\pm0.2}$ & $0.5_{\pm0.782}$ & $0.050_{\pm0.16}$ & $0.10_{\pm0.22}$ \\
\bottomrule
\end{tabular}
\end{adjustbox}
\endgroup
\end{table}

\begin{table}[ht]
\centering
\caption{Additional normalization baselines for image classification:
width $128\to4096$, validation-NLL-selected prior precision, and full,
diagonal, and KFAC last-layer Laplace. Results are mean $\pm$ sample
standard deviation over ten paired seeds.
$|\Delta\NLL|$ is multiplied by $10^2$; other metrics follow
Table~\ref{tab:mnist_fmnist_approx_transfer}.
Bold marks the smallest reported mean in each column; lower is better.}
\label{tab:classification-normalization}
\begingroup
\footnotesize
\setlength{\tabcolsep}{2.5pt}
\begin{adjustbox}{max width=\linewidth}
\begin{tabular}{l|cc|cc|cc|cc|cc|cc}
\toprule
\multirow{3}{*}{\textbf{Method}}
 & \multicolumn{6}{c|}{\textbf{MNIST}}
 & \multicolumn{6}{c}{\textbf{FMNIST}} \\
\cmidrule(lr){2-7}\cmidrule(lr){8-13}
 & \multicolumn{2}{c|}{\textbf{LL full}}
 & \multicolumn{2}{c|}{\textbf{LL diag}}
 & \multicolumn{2}{c|}{\textbf{LL KFAC}}
 & \multicolumn{2}{c|}{\textbf{LL full}}
 & \multicolumn{2}{c|}{\textbf{LL diag}}
 & \multicolumn{2}{c}{\textbf{LL KFAC}} \\
 & $|\Delta\lambda|$ & $|\Delta\NLL|$
 & $|\Delta\lambda|$ & $|\Delta\NLL|$
 & $|\Delta\lambda|$ & $|\Delta\NLL|$
 & $|\Delta\lambda|$ & $|\Delta\NLL|$
 & $|\Delta\lambda|$ & $|\Delta\NLL|$
 & $|\Delta\lambda|$ & $|\Delta\NLL|$ \\
\midrule
SP + isotropic & $1.75_{\pm0.26}$ & $9.8_{\pm4.6}$ & $0.45_{\pm0.28}$ & $0.8_{\pm0.9}$ & $1.50_{\pm0.24}$ & $6.9_{\pm3.2}$ & $1.55_{\pm0.28}$ & $9.0_{\pm4.1}$ & $\mathbf{0.25}_{\pm0.35}$ & $0.9_{\pm1.2}$ & $1.85_{\pm0.24}$ & $13.2_{\pm3.2}$ \\
SP + metric & $3.25_{\pm0.26}$ & $9.6_{\pm2.1}$ & $4.55_{\pm0.28}$ & $14.8_{\pm1.9}$ & $3.50_{\pm0.24}$ & $10.7_{\pm1.9}$ & $3.45_{\pm0.28}$ & $39.0_{\pm6.3}$ & $5.25_{\pm0.35}$ & $67.2_{\pm7.3}$ & $3.15_{\pm0.24}$ & $33.3_{\pm5.5}$ \\
$\mup$ + isotropic & $4.75_{\pm0.26}$ & $15.7_{\pm3.0}$ & $5.30_{\pm0.26}$ & $20.6_{\pm3.2}$ & $4.75_{\pm0.26}$ & $16.5_{\pm2.7}$ & $4.60_{\pm0.21}$ & $34.9_{\pm2.8}$ & $4.55_{\pm0.28}$ & $47.2_{\pm5.0}$ & $4.45_{\pm0.28}$ & $30.5_{\pm2.6}$ \\
\textbf{$\sigmat$ (ours)} & $\mathbf{0.25}_{\pm0.26}$ & $\mathbf{0.2}_{\pm0.3}$ & $\mathbf{0.30}_{\pm0.26}$ & $\mathbf{0.1}_{\pm0.2}$ & $\mathbf{0.25}_{\pm0.26}$ & $\mathbf{0.2}_{\pm0.3}$ & $\mathbf{0.40}_{\pm0.21}$ & $\mathbf{0.4}_{\pm0.4}$ & $0.45_{\pm0.28}$ & $\mathbf{0.6}_{\pm0.6}$ & $\mathbf{0.55}_{\pm0.28}$ & $\mathbf{0.3}_{\pm0.3}$ \\
\addlinespace[2pt]
SP + empirical & $1.35_{\pm0.34}$ & $7.50_{\pm3.3}$ & $\mathbf{0.3}_{\pm0.35}$ & $0.5_{\pm0.5}$ & $0.50_{\pm0.41}$ & $1.7_{\pm0.4}$ & $1.60_{\pm0.32}$ & $10.1_{\pm4.1}$ & $0.45_{\pm0.28}$ & $0.75_{\pm0.36}$ & $2.05_{\pm0.37}$ & $14.8_{\pm4.0}$ \\
$\mup$ + empirical & $0.35_{\pm0.34}$ & $0.65_{\pm0.42}$ & $0.35_{\pm0.34}$ & $\mathbf{0.1}_{\pm0.3}$ & $0.4_{\pm0.211}$ & $0.35_{\pm0.3}$ & $0.6_{\pm0.3}$ & $\mathbf{0.4}_{\pm0.5}$ & $0.60_{\pm0.32}$ & $1.10_{\pm0.8}$ & $0.65_{\pm0.34}$ & $0.6_{\pm0.4}$ \\
\bottomrule
\end{tabular}
\end{adjustbox}
\endgroup
\end{table}

$\sigmat$ has the smallest reported mean NLL transfer gap across all
regression and classification settings, tying $\mup$ + empirical on
MNIST LL-diagonal and FMNIST LL-full.
These comparisons concern mean transfer gaps, not statistical significance
or absolute predictive quality.

\subsection{Additional classification datasets}
\label{app:tabular-classification}

We evaluate prior-precision transfer on
PenDigits~\citep{alpaydin1996pen} and
Letter Recognition~\citep{slate1991letter}, using two-hidden-layer
ReLU MLPs with proxy width $32$ and target width $128$.
For each dataset, we use $5{,}000$ training and $1{,}000$ validation
examples and standardize inputs using training-set statistics.
PenDigits retains its provided test split; Letter uses its final
$4{,}000$ examples for testing. The $\mup$ models are trained with
MuSGD for $500$ epochs, batch size $128$, momentum $0.9$, and an
initial learning rate of $0.05$ with cosine decay.

We use full-GGN Laplace over all parameters, with the common-logit
head direction held fixed. At each trained checkpoint, we select
the prior precision by maximizing the fixed-checkpoint Laplace
log-evidence. The proxy-selected precision is then applied unchanged
to the target. Table~\ref{tab:efficiency} reports the resulting
target test-NLL change relative to target-side evidence selection.
Timings include the complete Laplace fit and precision search,
excluding network training.

These all-parameter evidence searches extend beyond the convergence
guarantee of Corollary~\ref{cor:evidence-transfer}. Empirically,
proxy-side selection gives approximately $22\times$ and $32\times$
speedups on PenDigits and Letter, respectively, with target test-NLL
changes of $-0.007$ and $0.078$ relative to target-side evidence selection.

\section{Evidence-selected transfer on MNIST/FMNIST}
\label{app:evidence-transfer}

Table~\ref{tab:mnist_fmnist_evidence} repeats the main-text comparison with
$\lambda$ chosen by the fixed-checkpoint plug-in Laplace evidence score instead of validation NLL.
The selection gaps tell the same story, but the NLL column must be read with
care: at width $4096$ the target's own plug-in score optimum sits away from the
test optimum, so a transferred precision can attain a \emph{lower} test NLL
than the target's own selection. This is criterion miscalibration, not a
transfer gain, which is why validation NLL is the primary criterion here.

\begin{table*}[th]
\centering
\caption{Evidence-selected prior transfer (proxy $128\to$ target $4096$,
10 seeds, exact implementations; per-seed absolute values, $|\Delta\NLL|$
scaled by $10^2$; notation as in Table~\ref{tab:mnist_fmnist_approx_transfer}).}
\label{tab:mnist_fmnist_evidence}
\footnotesize \setlength{\tabcolsep}{2.5pt}
\begin{adjustbox}{width=0.95\textwidth}
\begin{tabular}{l|cc|cc|cc|cc|cc|cc}
\toprule
\multirow{3}{*}{\textbf{Method}}
  & \multicolumn{6}{c|}{\textbf{MNIST}} & \multicolumn{6}{c}{\textbf{FMNIST}} \\
\cmidrule(lr){2-7} \cmidrule(lr){8-13}
  & \multicolumn{2}{c|}{\textbf{LL full}} & \multicolumn{2}{c|}{\textbf{LL diag}}
  & \multicolumn{2}{c|}{\textbf{LL KFAC}}
  & \multicolumn{2}{c|}{\textbf{LL full}} & \multicolumn{2}{c|}{\textbf{LL diag}}
  & \multicolumn{2}{c}{\textbf{LL KFAC}} \\
  & $|\Delta\lambda|$ & $|\Delta\NLL|$ & $|\Delta\lambda|$ & $|\Delta\NLL|$
  & $|\Delta\lambda|$ & $|\Delta\NLL|$
  & $|\Delta\lambda|$ & $|\Delta\NLL|$ & $|\Delta\lambda|$ & $|\Delta\NLL|$
  & $|\Delta\lambda|$ & $|\Delta\NLL|$ \\
\midrule
SP + isotropic & $1.15_{\pm.24}$ & $21.8_{\pm5.2}$ & $2.45_{\pm.16}$ & $69.2_{\pm5.4}$ & --- & ---
  & $1.75_{\pm.26}$ & $35.0_{\pm5.6}$ & $3.10_{\pm.32}$ & $82.6_{\pm7.3}$ & --- & --- \\
SP + metric & $6.15_{\pm.24}$ & $133.5_{\pm2.1}$ & $7.45_{\pm.16}$ & $144.1_{\pm3.2}$ & --- & ---
  & $6.75_{\pm.26}$ & $105.2_{\pm4.1}$ & $8.10_{\pm.32}$ & $75.7_{\pm6.2}$ & --- & --- \\
$\mup$ + isotropic & $5.75_{\pm.35}$ & $58.3_{\pm5.0}$ & $5.60_{\pm.32}$ & $68.2_{\pm7.1}$ & --- & ---
  & $5.30_{\pm.26}$ & $34.2_{\pm3.3}$ & $4.70_{\pm.26}$ & $7.9_{\pm4.6}$ & --- & --- \\
\textbf{$\sigmat$ (ours)} & {\boldmath$0.75_{\pm.35}$} & {\boldmath$17.4_{\pm8.1}$} & {\boldmath$0.60_{\pm.32}$} & {\boldmath$17.8_{\pm9.3}$} & --- & ---
  & {\boldmath$0.30_{\pm.26}$} & {\boldmath$6.2_{\pm5.4}$} & {\boldmath$0.30_{\pm.26}$} & {\boldmath$5.8_{\pm5.0}$} & --- & --- \\
\bottomrule
\end{tabular}
\end{adjustbox}
\end{table*}

\section{Additional OOD-transfer results}
\label{app:ood-transfer}

\begin{figure}[t]
\centering
\includegraphics[width=\linewidth]{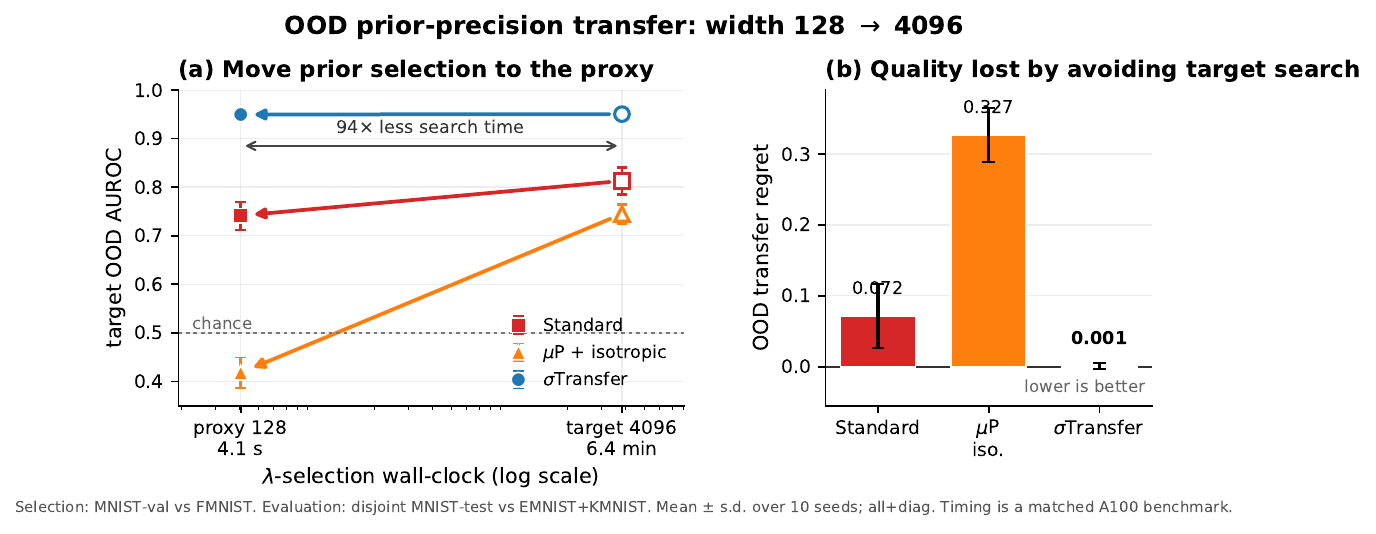}
\caption{\textbf{Time-to-quality tradeoff for OOD precision transfer}
($128\!\to\!4096$, all-parameter diagonal GGN). Filled markers use the
proxy-selected precision; open markers use target-side selection. Arrows point
toward the cheaper proxy search. The reported A100 times compare proxy-side
and target-side selection procedures ($4.1$ seconds versus $6.4$ minutes,
$94\times$). Zero-shot transfer removes the target sweep but retains one
target posterior at the transferred precision. The right panel reports paired
transfer regret over 10 seeds. Only $\sigmat$ removes the target search
without measurable OOD-quality loss.}
\label{fig:ood-time-quality}
\end{figure}

Figure~\ref{fig:ood-time-quality} shows the time-to-quality tradeoff,
and Table~\ref{tab:ood-transfer} gives the width-$4096$ OOD transfer
comparison summarized in the main text.
Table~\ref{tab:ood-transfer-widths} expands this
across all measured target widths. Selection and
evaluation use the same disjoint datasets.
The signed regret can be slightly negative because each target selects on
FMNIST but is evaluated on unseen EMNIST and on
KMNIST~\citep{clanuwat2018deep}; in those cases the
proxy-selected precision generalizes marginally better to the evaluation
distributions.

\begin{table}[ht]
\centering\small
\setlength{\tabcolsep}{8pt}
\begin{tabular}{lccc}
\toprule
\textbf{Method} & \textbf{AUROC using $\lambda^\star_{128}$} &
\textbf{AUROC using $\lambda^\star_{4096}$} &
\textbf{Transfer regret $\downarrow$} \\
\midrule
SP            & $0.741 \pm 0.029$ & $0.813 \pm 0.027$ & $0.072 \pm 0.046$ \\
$\mup$ + isotropic & $0.418 \pm 0.032$ & $0.745 \pm 0.020$ & $0.327 \pm 0.038$ \\
$\sigmat$ (ours) & $\mathbf{0.950 \pm 0.003}$ & $\mathbf{0.950 \pm 0.004}$ & $\mathbf{0.001 \pm 0.004}$ \\
\bottomrule
\end{tabular}
\caption{\textbf{OOD-aware precision transfer at width $4096$} (10 seeds;
mean $\pm$ sample s.d.). The precision is selected on
MNIST-validation versus FMNIST and evaluated on disjoint MNIST-test versus
EMNIST$+$KMNIST. Transfer regret is
$\mathrm{AUROC}(\lambda^\star_{4096})-\mathrm{AUROC}(\lambda^\star_{128})$;
lower is better.}
\label{tab:ood-transfer}
\end{table}

\begin{table}[ht]
\centering\small
\setlength{\tabcolsep}{4pt}
\begin{tabular}{llccc}
\toprule
\textbf{Method} & \textbf{Target width $n$} &
\textbf{AUROC using $\lambda^\star_{128}$} &
\textbf{AUROC using $\lambda^\star_n$} & \textbf{Regret $\downarrow$} \\
\midrule
\multirow{4}{*}{SP}
& 128  & $0.957\pm0.005$ & $0.957\pm0.005$ & $0.0000\pm0.0000$ \\
& 512  & $0.909\pm0.008$ & $0.909\pm0.008$ & $0.0000\pm0.0000$ \\
& 2048 & $0.807\pm0.012$ & $0.802\pm0.013$ & $-0.0044\pm0.0193$ \\
& 4096 & $0.741\pm0.029$ & $0.813\pm0.027$ & $0.0716\pm0.0455$ \\
\midrule
\multirow{4}{*}{$\mup$ + isotropic}
& 128  & $0.957\pm0.004$ & $0.957\pm0.004$ & $0.0000\pm0.0000$ \\
& 512  & $0.868\pm0.024$ & $0.868\pm0.024$ & $0.0000\pm0.0000$ \\
& 2048 & $0.445\pm0.025$ & $0.743\pm0.022$ & $0.2977\pm0.0280$ \\
& 4096 & $0.418\pm0.032$ & $0.745\pm0.020$ & $0.3269\pm0.0383$ \\
\midrule
\multirow{4}{*}{$\sigmat$}
& 128  & $0.962\pm0.003$ & $0.962\pm0.003$ & $0.0000\pm0.0000$ \\
& 512  & $0.958\pm0.004$ & $0.958\pm0.004$ & $0.0003\pm0.0005$ \\
& 2048 & $0.952\pm0.004$ & $0.951\pm0.004$ & $-0.0006\pm0.0012$ \\
& 4096 & $\mathbf{0.950\pm0.003}$ & $\mathbf{0.950\pm0.004}$ & $\mathbf{0.0008\pm0.0039}$ \\
\bottomrule
\end{tabular}
\caption{\textbf{OOD-aware precision transfer across width} (10 seeds;
mean $\pm$ sample s.d.). Both AUROC columns evaluate target width $n$;
$\lambda^\star_{128}$ is selected on the proxy and $\lambda^\star_n$ on that
target. Regret is
$\mathrm{AUROC}(\lambda^\star_n)-\mathrm{AUROC}(\lambda^\star_{128})$ on the
held-out evaluation distributions. The main text reports $n=4096$.}
\label{tab:ood-transfer-widths}
\end{table}

\section{Sequential active learning details}
\label{app:al-details}

For sequential acquisition on MNIST and Fashion-MNIST, we use
three-hidden-layer ReLU MLPs of widths $128$ and $4096$. Each run starts
with $20$ labelled examples and acquires $20$ per round until $1{,}000$
are labelled, giving $50$ evaluations including the initial fit. The
unlabelled pool initially contains $2{,}000$ examples. Each round scores
$500$ randomly sampled remaining candidates using EPIG, $200$ reference
inputs and $256$ posterior samples. We use $500$ validation and $1{,}000$
test examples. Initial labelled, pool and validation sets are disjoint,
class-balanced subsets of the official training split; reference and
test inputs are disjoint, class-balanced subsets of the official test
split. EPIG uses only the reference inputs.

Acquisition uses last-layer diagonal Laplace, including the head bias,
at fixed $\lambda=1$. We warm-start the weights each round and use
full-batch Adam (SP) or MuAdam ($\mup$), both at learning rate $10^{-3}$,
for at most $500$ updates. Validation NLL is checked every $10$ updates,
with patience $150$ updates, and the best checkpoint is restored. The
strategies share initial splits across $20$ seeds.

\begin{figure}[t]
\centering
\includegraphics[width=\linewidth]{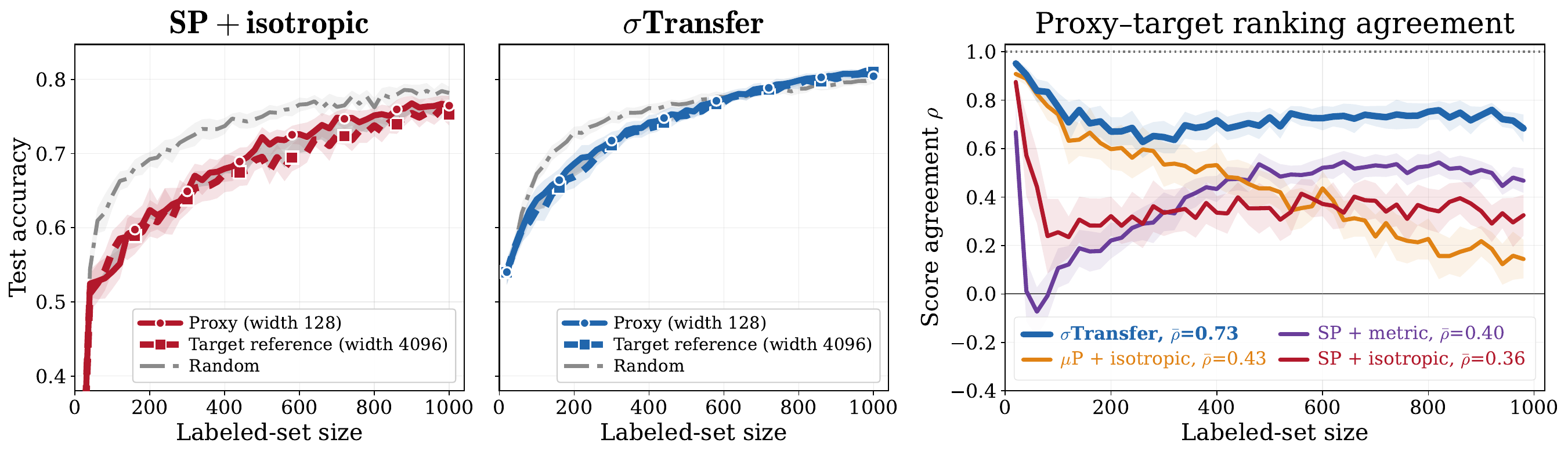}
\caption{\textbf{Sequential active learning on Fashion-MNIST}
(20 seeds; panels and encoding as in Figure~\ref{fig:al-pair}). The
agreement ordering replicates with a compressed baseline spread
($\bar\rho=0.73$ for $\sigmat$ against $0.36$--$0.43$). On this task the
acquisition curves do not separate from random, so we claim decision transfer
on both datasets and downstream utility on MNIST only.}
\label{fig:al-pair-fmnist}
\end{figure}

\begin{figure}[t]
\centering
\includegraphics[width=\linewidth]{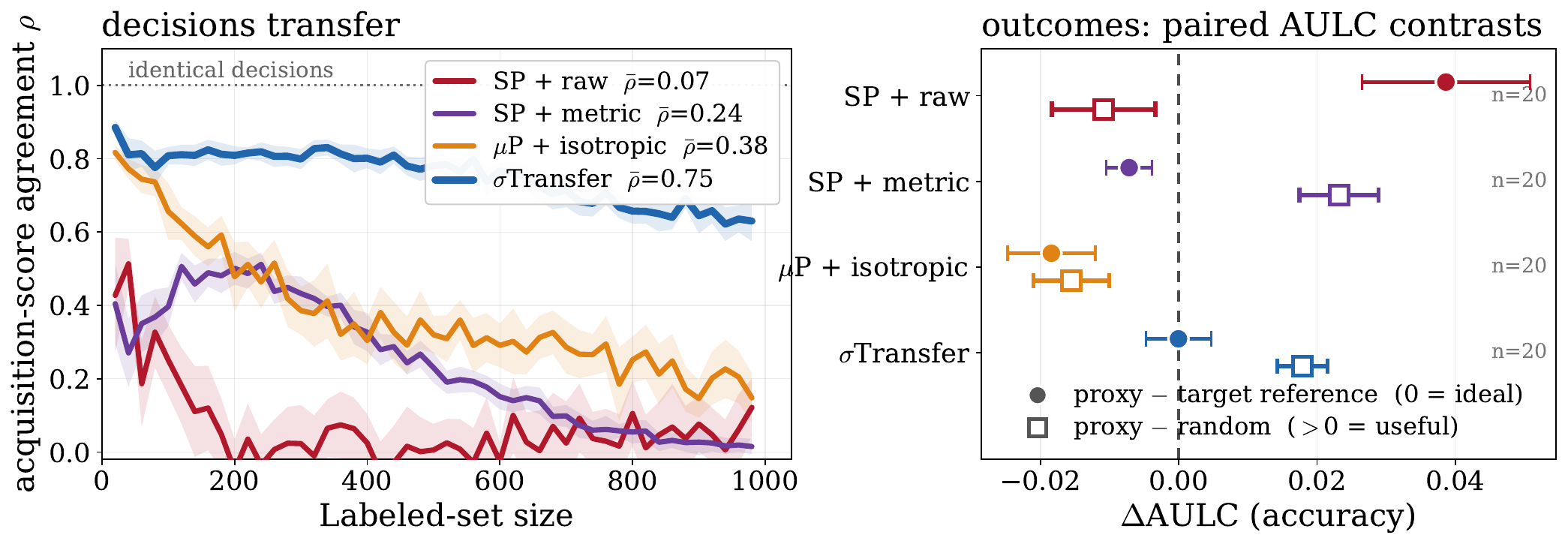}
\caption{\textbf{Sequential active learning: decisions and outcomes}
(10-class MNIST, last-layer diagonal posterior at fixed $\log_2\lambda=0$,
width $128\!\to\!4096$, 20 seeds). \emph{Left:} per-round Spearman
correlation between the width-$128$ and width-$4096$ acquisition scores
(band: 95\% CI of the mean). $\sigmat$ holds
$\rho\!\approx\!0.6$--$0.8$ for the whole trajectory while every baseline
decays toward zero, and the trajectory means are
$\bar\rho=0.75\pm0.01$ ($\sigmat$), $0.38\pm0.02$ ($\mup$ + isotropic),
$0.24\pm0.01$ (SP + metric) and $0.07\pm0.03$ (SP + isotropic). \emph{Right:}
paired area-under-the-learning-curve contrasts (per seed, then averaged;
bars: 95\% CI). Filled: proxy $-$ target reference, so $0$ means the
proxy acquires as well as the target's own posterior; open: proxy $-$
random. Only $\sigmat$ is simultaneously faithful (its fidelity interval
is the only one covering zero, $-0.0001$, $[-0.005,+0.005]$) and useful
($+0.018$, $[+0.014,+0.022]$); SP + isotropic's apparent $+0.039$ fidelity reflects
a degenerate target reference that itself falls $0.050$ below random. Full learning curves:
Figure~\ref{fig:mnist-al}; Fashion-MNIST: Figures~\ref{fig:al-summary-fmnist}
and~\ref{fig:fmnist-al}.}
\label{fig:al-summary}
\end{figure}

\begin{figure}[t]
\centering
\includegraphics[width=\linewidth]{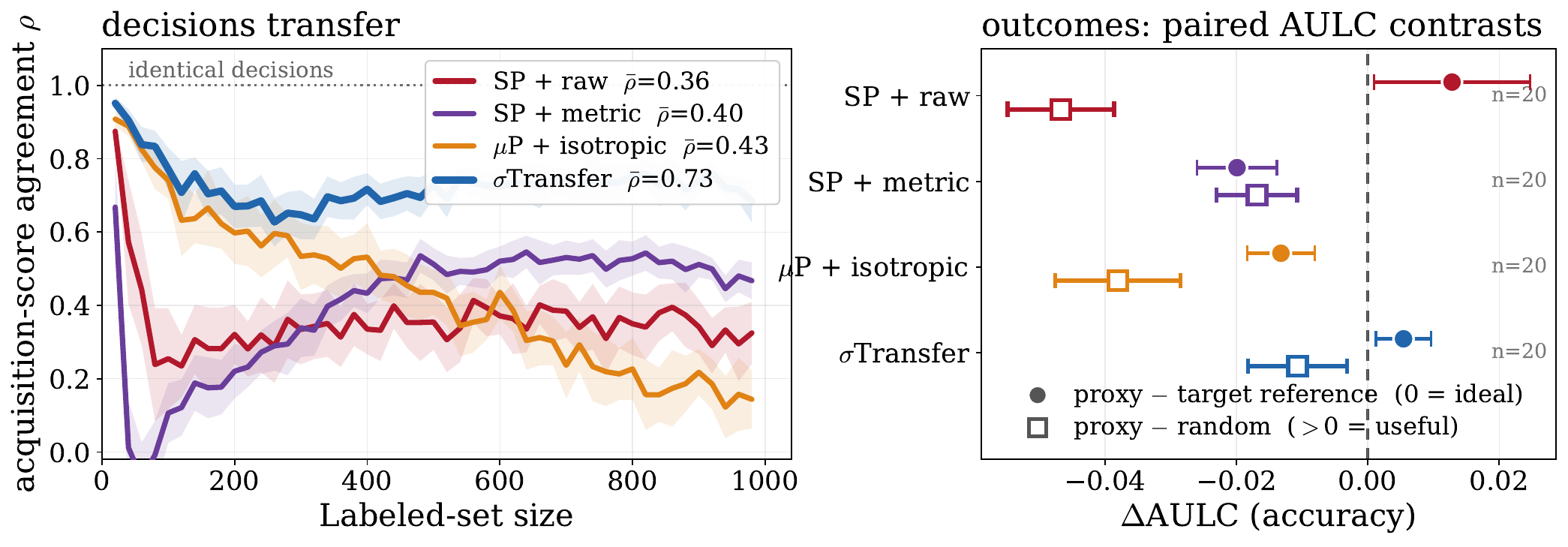}
\caption{\textbf{Sequential active learning on Fashion-MNIST: decisions and
outcomes} (20 seeds; same protocol and panels as
Figure~\ref{fig:al-summary}). The agreement ordering replicates
($\bar\rho=0.73\pm0.02$ for $\sigmat$ against $0.36$--$0.43$; see also
Figure~\ref{fig:al-pair-fmnist}), with a compressed baseline spread. The outcome panel is negative for every method:
on this task EPIG acquisition does not beat random over the trajectory
(a cold-start effect that endpoint accuracy alone conceals). At $20$ seeds
every method's utility interval excludes zero; $\sigmat$ is the least affected
($-0.011$, $[-0.018,-0.003]$) against $-0.047$ $[-0.055,-0.039]$ for
SP + isotropic. We
therefore claim decision transfer on both datasets and downstream utility on
MNIST only.}
\label{fig:al-summary-fmnist}
\end{figure}

\begin{figure}[t]
\centering
\includegraphics[width=\linewidth]{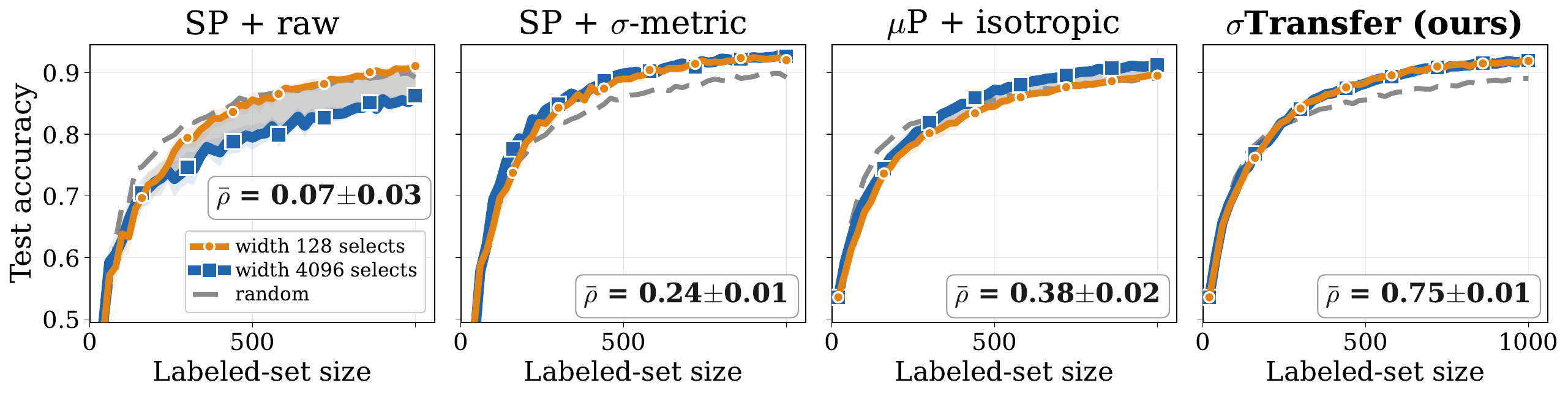}
\caption{\textbf{Sequential active learning on 10-class MNIST}
(last-layer diagonal posterior, 20 seeds; bands: 95\% CI of the mean). Every curve is
averaged over the same complete-trajectory seed set. One panel per method, titled
above each panel. Orange selects each acquisition batch with the width-$128$ posterior and
blue with the width-$4096$ posterior; both are then evaluated on the same
width-$4096$ model, and the grey band between them is the
acquisition-transfer gap. The thin orange line is drawn over the thicker blue
one, so a method in which the proxy reproduces the target's decisions shows
orange tracing blue with no blue visible beneath. Markers (circles for width
$128$, squares for width $4096$) are stamped every seventh acquisition round
and carry no information beyond series identity. $\bar\rho$ is the mean rank correlation between
the two posteriors' acquisition scores over the trajectory, the direct
measure of how closely the small model reproduces the large model's
decisions. Accuracy on this task is largely insensitive to acquisition
quality, but the decisions are not: $\bar\rho$ rises from $0.07$ (SP + isotropic)
through $0.24$ (SP + metric) and $0.38$ ($\mup$ + isotropic) to $0.75$ under
$\sigmat$, so neither the parametrization nor the prior geometry alone
suffices.}
\label{fig:mnist-al}
\end{figure}

\begin{figure}[t]
\centering
\includegraphics[width=\linewidth]{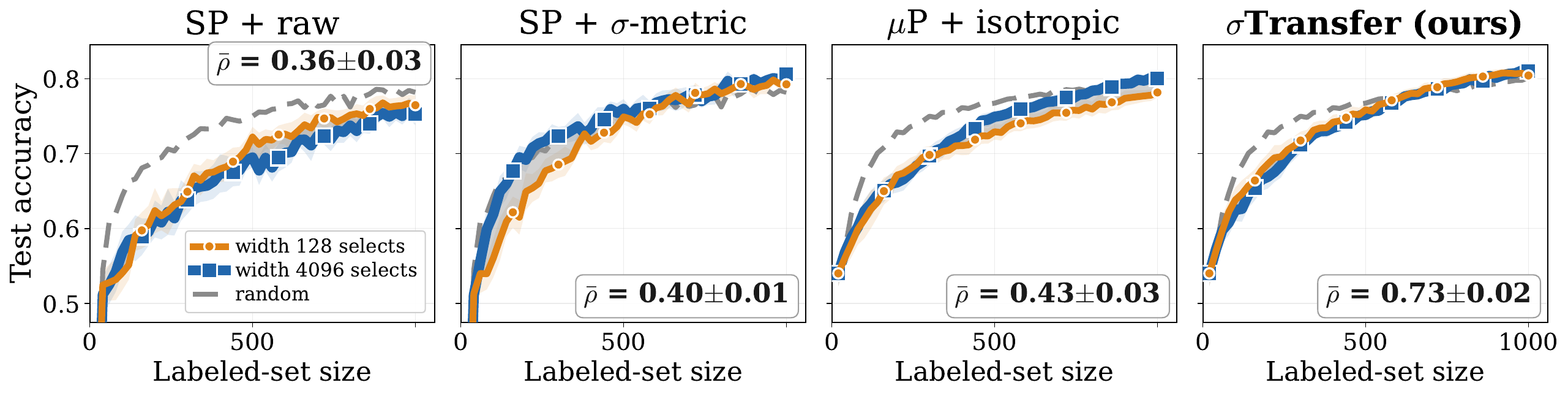}
\caption{\textbf{Sequential active learning on Fashion-MNIST}
(last-layer diagonal posterior, $20$ seeds; bands: 95\% CI of the mean; same protocol
and encoding as Figure~\ref{fig:mnist-al}). The decision-agreement ordering
matches MNIST, with $\bar\rho=0.73$ under $\sigmat$, compared with
$0.36$--$0.43$ for the three baselines. SP + isotropic has higher agreement
on Fashion-MNIST than on MNIST ($0.36$ versus $0.07$).}
\label{fig:fmnist-al}
\end{figure}

\begin{figure}[t]
\centering
\includegraphics[width=\linewidth]{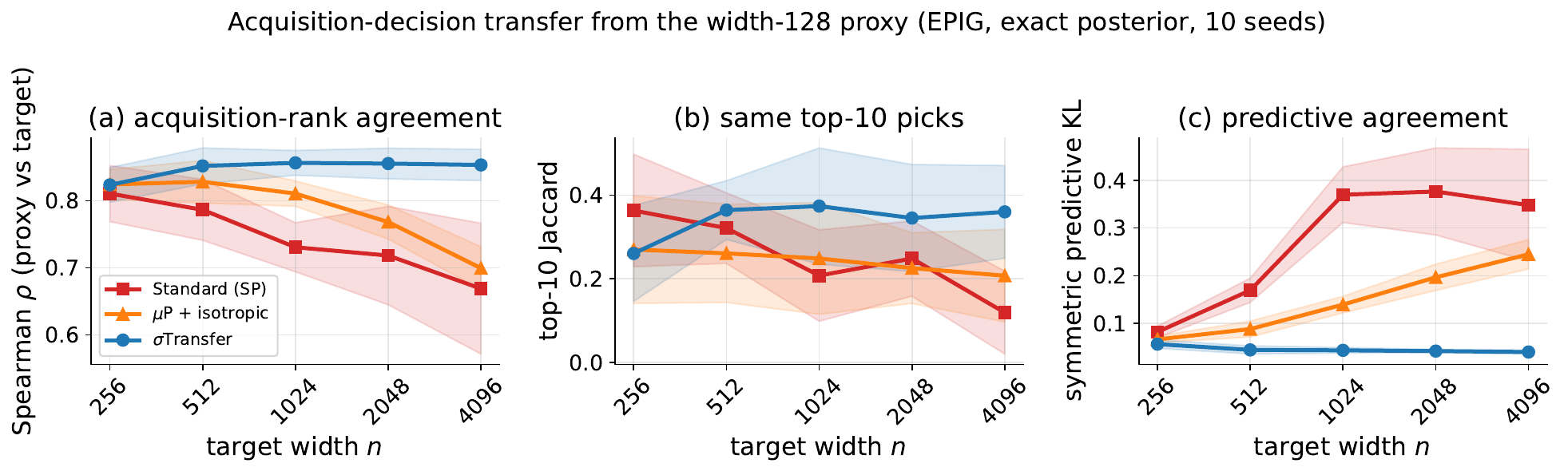}
\caption{\textbf{Direct acquisition transfer} (EPIG, exact last-layer
posterior, validation-NLL-selected transferred precision, 10 seeds; bands
95\% CI of the mean). \textbf{(a)} Proxy--target acquisition-rank agreement.
\textbf{(b)} Top-10 Jaccard overlap. \textbf{(c)} Symmetric predictive KL.
The target posterior is used only to measure agreement; the proxy choices do
not require it.}
\label{fig:m5a}
\end{figure}

Figure~\ref{fig:mnist-al} gives the full per-method learning curves for the
$50$-round sequential active-learning experiment summarised in
Section~\ref{sec:mnist-al}; Table~\ref{tab:mnist-al-utility} gives the
endpoint statistics at $n\!=\!1000$ labelled points, and
Figure~\ref{fig:m5a} and Table~\ref{tab:m5a-epig} the one-shot EPIG
acquisition-agreement metrics for both datasets.

\begin{table}[t]
\centering
\caption{\textbf{EPIG acquisition transfer on MNIST and FMNIST} (proxy 128
$\to$ target 4096, validation-NLL-selected $\lambda$, exact last-layer
posterior; MNIST 10 seeds, FMNIST 5 seeds). Agreement between the proxy and
target posteriors, with the target evaluated at the transferred
precision: Spearman rank correlation of EPIG scores, top-10 acquisition-set
Jaccard, and symmetric predictive KL (lower better). \textbf{Bold} = best
per column.}
\label{tab:m5a-epig}
\footnotesize \setlength{\tabcolsep}{2.5pt}
\begin{adjustbox}{width=\linewidth}
\begin{tabular}{l|ccc|ccc}
\toprule
\multirow{2}{*}{\textbf{Method}}
  & \multicolumn{3}{c|}{\textbf{MNIST}} & \multicolumn{3}{c}{\textbf{FMNIST}} \\
\cmidrule(lr){2-4} \cmidrule(lr){5-7}
  & Spearman $\rho$ & Jacc$_{10}$ & KL$_{\mathrm{sym}}$
  & Spearman $\rho$ & Jacc$_{10}$ & KL$_{\mathrm{sym}}$ \\
\midrule
SP + isotropic & $0.669_{\pm.103}$ & $0.119_{\pm.105}$ & $0.348_{\pm.124}$
  & $0.589_{\pm.137}$ & $0.118_{\pm.099}$ & $0.944_{\pm.193}$ \\
$\mup$ + isotropic & $0.700_{\pm.034}$ & $0.207_{\pm.117}$ & $0.245_{\pm.032}$
  & $0.629_{\pm.063}$ & $0.134_{\pm.124}$ & $0.418_{\pm.060}$ \\
\textbf{$\sigmat$ (ours)} & {\boldmath$0.854_{\pm.025}$} & {\boldmath$0.360_{\pm.117}$} & {\boldmath$0.040_{\pm.004}$}
  & {\boldmath$0.750_{\pm.066}$} & {\boldmath$0.198_{\pm.104}$} & {\boldmath$0.167_{\pm.021}$} \\
\bottomrule
\end{tabular}
\end{adjustbox}
\end{table}

\begin{table}[ht]
\centering\small
\setlength{\tabcolsep}{7pt}
\begin{tabular}{llccc}
\toprule
 & & \textbf{Test acc.} & \textbf{Test NLL} & \textbf{Regret} \\
\midrule
SP & random      & $0.901 \pm 0.016$ & $0.976 \pm 0.141$ & $+0.009$ \\
SP & proxy EPIG  & $0.920 \pm 0.012$ & $0.774 \pm 0.171$ & $-0.011$ \\
SP & oracle EPIG & $0.910 \pm 0.008$ & $0.931 \pm 0.111$ & --- \\
\addlinespace[2pt]
$\mup$ + isotropic & random      & $0.890 \pm 0.009$ & $0.920 \pm 0.120$ & $+0.028$ \\
$\mup$ + isotropic & proxy EPIG  & $0.889 \pm 0.021$ & $0.909 \pm 0.298$ & $+0.029$ \\
$\mup$ + isotropic & oracle EPIG & $0.918 \pm 0.004$ & $0.694 \pm 0.107$ & --- \\
\addlinespace[2pt]
$\sigmat$ & random      & $0.888 \pm 0.006$ & $0.953 \pm 0.152$ & $+0.037$ \\
$\sigmat$ & proxy EPIG  & $\mathbf{0.923 \pm 0.011}$ & $\mathbf{0.700 \pm 0.109}$ & $\mathbf{+0.003}$ \\
$\sigmat$ & oracle EPIG & $0.926 \pm 0.010$ & $0.661 \pm 0.153$ & --- \\
\bottomrule
\end{tabular}
\caption{\textbf{Sequential active-learning utility at $n=1000$}
(10-class MNIST, last-layer diagonal posterior at fixed $\log_2\lambda=0$,
width $128\!\to\!4096$; mean $\pm$ sample s.d.). Regret $=$ oracle-EPIG
accuracy $-$ method accuracy. Round wall-clock is dominated by retraining the
width-$4096$ model and is the same for every strategy, so this experiment is a
decision-transfer result, not a speed result.}
\label{tab:mnist-al-utility}
\end{table}

\section{Decision transfer for language models}
\label{app:ood-gate}

We use a width-$128$ model to decide whether a width-$2048$ model should
answer a prompt. The small model computes the uncertainty score, and the
large model supplies the prediction.

\paragraph{Models and pretraining.}
We train decoder-only Transformers at hidden widths
$128,256,512,1024,2048$. All models have $8$ layers, $8$ attention heads,
context length $512$, and the Pythia padded vocabulary of $50{,}304$
entries~\citep{biderman2023pythia}. They are trained on the same stream of
$10^9$ FineWeb-Edu tokens~\citep{penedo2024fineweb}. We use one pretraining
seed for each width and parametrization.

The $\mup$ models use \texttt{MuReadout}, \texttt{set\_base\_shapes},
and MuAdamW, with base width $128$ and base learning rate $0.01$.
SP uses \texttt{nn.Linear} and AdamW at learning rate $0.0003$. These
rates are fixed across widths. Both use $\beta=(0.9,0.95)$, weight decay
$0.1$, $2\%$ warmup, cosine decay to $10\%$ of the initial rate, and
gradient-norm clipping at $1$, without dropout. Each update accumulates
four batches of $32$ length-$512$ sequences.

\paragraph{Classification heads.}
We freeze the backbone and fit a linear classifier to the
final-layer-normalization output $h$ at the last real token. Its logits
are $W(h/c)+b$, where $c^2$ is the mean squared feature norm on the
training inputs. We minimize summed cross-entropy plus
$(\lVert W\rVert_F^2+\lVert b\rVert^2)/2$ using L-BFGS in float64.
For each backbone and task, prior comparisons share the fitted head,
which is held fixed throughout the precision sweep.

We apply last-layer Laplace with width-normalized or isotropic priors.
For raw features of dimension $d$, their weight covariance scales are
$d^{-1}I$ and $I$, respectively. In the fitted head's coordinates, these
become $S_W=(c^2/d)I$ and $S_W=c^2I$; both use $S_b=I$. These definitions
apply to both backbone parametrizations; $\mup$ with the width-normalized
prior gives $\sigmat$. We use full-GGN curvature in float64 and evaluate
prior precision on a $61$-point $\log_2\lambda$ grid.
For precision selection and OOD gating, scores use $200$ Monte Carlo
samples of the Gaussian logits, with common random numbers.

\paragraph{Active learning on AG News.}
Figure~\ref{fig:llm-decisions} starts with $50$, $500$ or $5{,}000$
labelled examples and acquires $250$ labels one at a time from a common
pool of $512$ examples, disjoint from the initial labelled sets. EPIG
uses $256$ validation reference inputs and $2{,}048$ posterior samples.
We refit the heads after each acquisition and use full-GGN Laplace at
fixed $\lambda=1$. After the final acquisition, we evaluate the target's
MAP predictions on $7{,}600$ test examples. The $\sigmat$ arm uses the
width-normalized prior defined above. The feature normalization $c$
is fixed from the full training feature cache throughout acquisition.

\paragraph{OOD data.}
For each classification task, we take OOD prompts from two other
datasets, denoted $A$ and $B$. We score them using the ID task's
classification head and count abstention as the correct action.
Sentiment-to-sentiment pairs are excluded. Table~\ref{tab:ood-pools}
gives the source assignments and test-set sizes.

\begin{table}[ht]
\centering\small
\caption{OOD source assignments and test-set sizes.}
\label{tab:ood-pools}
\begin{tabularx}{\linewidth}{@{}Xll@{}}
\toprule
ID tasks (test-set sizes) & Source $A$ (pool size) & Source $B$ (pool size) \\
\midrule
AG News (7600), DBpedia (8000) & SST-2 (872) & TREC (500) \\
IMDB (8000), Rotten Tomatoes (1066), Yelp Polarity (8000), SST-2 (872)
& AG News (7600) & TREC (500) \\
Emotion (2000), SNLI (8000), Yahoo (8000), TREC (500)
& AG News (7600) & SST-2 (872) \\
\bottomrule
\end{tabularx}
\end{table}

Dataset sources are \citet{zhang2015character} for AG News,
DBpedia, Yahoo Answers, and Yelp Polarity;
\citet{socher2013recursive} for SST-2;
\citet{li2002learning} for TREC;
\citet{maas2011learning} for IMDB;
\citet{pang2005seeing} for Rotten Tomatoes;
\citet{saravia2018carer} for Emotion; and
\citet{bowman2015large} for SNLI.

\paragraph{Decision rule and calibration.}
We choose the proxy's prior precision using ID validation data and OOD
source $A$. The selected $\lambda_{\mathrm{ood}}$ maximizes the AUROC of
BALD~\citep{houlsby2011bayesian} for distinguishing ID prompts from
source $A$. At this precision, we choose a threshold $\tau$ that
maximizes Youden's $J$, defined as ID retention plus OOD recall minus one.

The large model answers a prompt when the proxy's BALD score is at most
$\tau$; otherwise, the system abstains. Answered prompts use the large
model's MAP predictions. We evaluate on ID test data and OOD source $B$,
which was not used for calibration. We then exchange $A$ and $B$ and
average the two results.

\paragraph{Metrics and results.}
\emph{OOD recall} is the fraction of OOD prompts rejected, and
\emph{ID retention} is the fraction of ID prompts answered.
\emph{Answered accuracy} is the accuracy among answered prompts, counting
answered OOD prompts as errors. \emph{Abstention agreement} is the
fraction of the proxy's rejected prompts that are also rejected by the
separately calibrated target model.

Table~\ref{tfull} gives the results for each task. The final row reports
the mean and population standard deviation across the ten tasks, after averaging
the two source orderings within each task.

\begin{table}[ht]
\centering\small
\caption{Proxy-based abstention results for width $128\to2048$.
Each entry reports OOD recall / ID retention / answered accuracy /
abstention agreement, averaged over the two source orderings.}
\label{tfull}
\setlength{\tabcolsep}{4pt}
\begin{adjustbox}{max width=\linewidth}
\begin{tabular}{l cc}
\toprule
Dataset & SP $+$ isotropic & $\sigmat$ \\
\midrule
AG News         & $0.846$ / $0.969$ / $0.872$ / $0.784$ & $0.895$ / $0.963$ / $0.866$ / $0.702$ \\
DBpedia         & $0.981$ / $0.999$ / $0.899$ / $0.985$ & $0.987$ / $0.998$ / $0.869$ / $0.966$ \\
Emotion         & $1.000$ / $0.999$ / $0.485$ / $0.998$ & $0.998$ / $0.999$ / $0.479$ / $0.998$ \\
IMDB            & $0.605$ / $0.826$ / $0.603$ / $0.846$ & $0.725$ / $0.837$ / $0.643$ / $0.873$ \\
Rotten Tomatoes & $0.875$ / $0.977$ / $0.477$ / $0.955$ & $0.937$ / $0.984$ / $0.547$ / $0.969$ \\
SNLI            & $0.958$ / $0.992$ / $0.496$ / $0.953$ & $0.974$ / $0.986$ / $0.512$ / $0.921$ \\
SST-2           & $0.836$ / $0.968$ / $0.456$ / $0.950$ & $0.732$ / $0.939$ / $0.431$ / $0.912$ \\
TREC            & $0.998$ / $1.000$ / $0.696$ / $1.000$ & $0.999$ / $0.996$ / $0.655$ / $0.998$ \\
Yahoo           & $0.676$ / $0.976$ / $0.468$ / $0.774$ & $0.662$ / $0.968$ / $0.454$ / $0.816$ \\
Yelp Polarity   & $0.775$ / $0.927$ / $0.713$ / $0.729$ & $0.854$ / $0.926$ / $0.760$ / $0.734$ \\
\midrule
Mean $\pm$ SD  & $0.855${\scriptsize$\pm.13$} / $0.963${\scriptsize$\pm.05$} / $0.616${\scriptsize$\pm.16$} / $0.898${\scriptsize$\pm.10$}
                & $0.876${\scriptsize$\pm.12$} / $0.959${\scriptsize$\pm.05$} / $0.621${\scriptsize$\pm.16$} / $0.889${\scriptsize$\pm.10$} \\
\bottomrule
\end{tabular}
\end{adjustbox}
\end{table}

\paragraph{Comparison with target decisions.}
For evaluation, we also apply the same calibration procedure to the
target's own posterior. This gives mean OOD recall of $0.972/0.980$,
ID retention of $0.975/0.978$, and answered accuracy of $0.693/0.694$ for
SP/$\sigmat$. Using the proxy instead reduces OOD recall by approximately
$10$--$12$ percentage points, ID retention by $1$--$2$ points, and
answered accuracy by $7$--$8$ points. Deployment requires only the proxy
posterior and target MAP predictions.

\paragraph{Additional comparisons.}
The target's mean NLL on answered ID prompts differs from ungated
prediction by $-0.004/-0.005$ nats for SP/$\sigmat$. Replacing BALD with
one minus the maximum probability of the posterior predictive
($1-\mathrm{MSP}$; \citealp{hendrycks2016baseline}) gives proxy OOD recall of
$0.41$--$0.42$ while retaining approximately half the ID prompts.
Applying this alternative to the target gives OOD recall of
$0.84$--$0.85$ and answered accuracy of $0.64$--$0.66$.

Selecting $\lambda$ using ID evidence alone gives OOD recall of
$0.85$--$0.86$; selecting $\tau$ still uses source $A$. BALD magnitudes
vary across widths. In the experiment reported here, the threshold is
always applied to the proxy's scores.

\section{Prior-precision transfer in language models}
\label{app:llm-lambda}

\begin{figure}[ht]
\centering
\includegraphics[width=\linewidth]{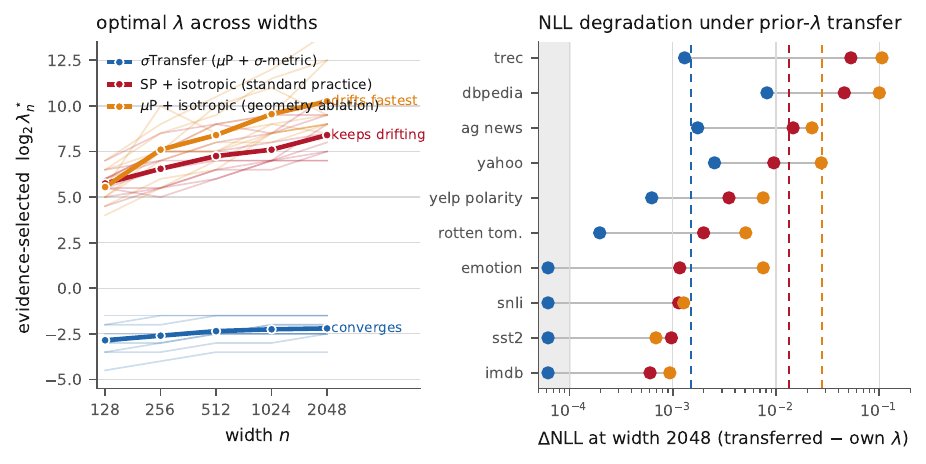}
\caption{Figure~\ref{fig:llm-lambda} with the geometry ablation added:
the same $\mup$ checkpoints with an isotropic prior drift faster
than SP ($5.5\to10.2$ bits mean optimum, $4.7\pm1.3$ bits at
$128\to2048$) and produce the largest mean target-NLL degradation ($0.028$,
reaching $0.1$ on TREC).}
\label{fig:llm-lambda-3arm}
\end{figure}

Both tables use the fixed-checkpoint plug-in evidence-score criterion on ID training data at
the proxy, the shared half-bit grid, and the target's own evidence
search only as an evaluation reference. In zero-shot deployment, the
target constructs one posterior at the transferred proxy precision.
$\Delta$NLL is transferred-minus-own test NLL at
the target; sweep arrays use $200$ MC samples.

\begin{table}[ht]
\centering\small
\caption{Our language models, width $128\to2048$
(Section~\ref{sec:llm-lambda}). $\lambda$ values are $\log_2$.}
\label{tab:llm-lambda-internal}
\vspace{3pt}
\setlength{\tabcolsep}{5pt}
\begin{tabular}{l ccc ccc}
\toprule
& \multicolumn{3}{c}{$\sigmat$} & \multicolumn{3}{c}{SP $+$ isotropic} \\
\cmidrule(lr){2-4}\cmidrule(lr){5-7}
Dataset & $\lambda^\star_{128}$ & $\lambda^\star_{2048}$ & $\Delta$NLL
        & $\lambda^\star_{128}$ & $\lambda^\star_{2048}$ & $\Delta$NLL \\
\midrule
AG News         & $-3.5$ & $-2.5$ & $+0.0017$ & $+4.5$ & $+7.0$  & $+0.0146$ \\
DBpedia         & $-4.5$ & $-3.5$ & $+0.0082$ & $+5.0$ & $+7.5$  & $+0.0458$ \\
Emotion         & $-3.0$ & $-2.5$ & $-0.0001$ & $+6.0$ & $+8.5$  & $+0.0012$ \\
IMDB            & $-2.0$ & $-1.5$ & $-0.0000$ & $+6.5$ & $+9.5$  & $+0.0006$ \\
Rotten Tomatoes & $-2.0$ & $-1.5$ & $+0.0002$ & $+7.0$ & $+10.5$ & $+0.0020$ \\
SNLI            & $-1.5$ & $-1.5$ & $+0.0000$ & $+5.5$ & $+7.5$  & $+0.0011$ \\
SST-2           & $-3.0$ & $-2.5$ & $+0.0000$ & $+6.0$ & $+9.0$  & $+0.0010$ \\
TREC            & $-2.5$ & $-2.0$ & $+0.0013$ & $+5.5$ & $+8.0$  & $+0.0530$ \\
Yahoo           & $-3.5$ & $-2.5$ & $+0.0025$ & $+5.5$ & $+7.0$  & $+0.0095$ \\
Yelp Polarity   & $-3.0$ & $-2.0$ & $+0.0006$ & $+6.0$ & $+9.5$  & $+0.0035$ \\
\midrule
Mean gap (bits) / $\Delta$NLL & \multicolumn{3}{c}{$0.65\pm0.34$ / $+0.0014$}
                              & \multicolumn{3}{c}{$2.65\pm0.63$ / $+0.0132$} \\
\bottomrule
\end{tabular}
\end{table}

\begin{table}[ht]
\centering\small
\caption{Public u-$\mup$/SP pairs, 1B$\to$7B, all ten tasks (primary
cells; extends Table~\ref{tab:umup-lambda}). Each model's search time is its
curvature eigendecomposition plus grid sweep (u-$\mup$ family; SP within
$10\%$); Yahoo and DBpedia run on CPU because their $41$k/$57$k-dimensional
eigenproblems exceed the GPU solver. The reported ratios compare proxy- and
target-side search procedures, not end-to-end Regime-A deployment; the 7B
model still constructs one posterior at the transferred precision. Timing
results are available for nine tasks; SST-2 is excluded from the timing
summary. Speedups are approximate ratios of the displayed, rounded times.}
\label{tab:llm-lambda-public}
\vspace{3pt}
\setlength{\tabcolsep}{4.5pt}
\begin{tabular}{l cc cc ccc}
\toprule
& \multicolumn{2}{c}{u-$\mup$ $+$ $\sigma$-metric}
& \multicolumn{2}{c}{SP $+$ isotropic} & & & \\
\cmidrule(lr){2-3}\cmidrule(lr){4-5}
Dataset & gap & $\Delta$NLL & gap & $\Delta$NLL
        & search 1B & search 7B & speedup \\
\midrule
AG News         & $0.0$ & $+0.0000$ & $0.5$ & $+0.0030$ & $4$\,s  & $9$\,s    & ${\sim}2.3\times$ \\
DBpedia         & $0.0$ & $+0.0000$ & $1.0$ & $+0.0247$ & $40$\,s & $3.7$\,h  & $330\times$ \\
Emotion         & $0.5$ & $+0.0002$ & $1.0$ & $+0.0030$ & $5$\,s  & $24$\,s   & $5\times$ \\
IMDB            & $0.0$ & $+0.0000$ & $0.5$ & $-0.0016$ & $2$\,s  & $4$\,s    & $2\times$ \\
Rotten Tomatoes & $0.5$ & $+0.0002$ & $0.5$ & $+0.0011$ & $2$\,s  & $3$\,s    & $2\times$ \\
SNLI            & $0.5$ & $+0.0000$ & $1.5$ & $+0.0011$ & $3$\,s  & $6$\,s    & $2\times$ \\
SST-2           & $0.0$ & $+0.0000$ & $1.0$ & $+0.0017$ & ---     & ---       & --- \\
TREC            & $0.0$ & $+0.0000$ & $0.0$ & $+0.0000$ & $5$\,s  & $23$\,s   & $5\times$ \\
Yahoo           & $0.0$ & $+0.0000$ & $0.5$ & $+0.0048$ & $19$\,s & $81$\,min & $260\times$ \\
Yelp Polarity   & $0.0$ & $+0.0000$ & $0.5$ & $-0.0015$ & $3$\,s  & $4$\,s    & $1\times$ \\
\midrule
Mean & $0.15\pm0.24$ & $+0.0000$ & $0.70\pm0.42$ & $+0.0036$ & & & \\
\bottomrule
\end{tabular}
\end{table}

\subsection{Comparison with the diagonal neural g-prior}
\label{app:gprior}

We test whether the diagonal neural g-prior of
\citet{antoran2022sampling} supports prior-precision transfer in our
fixed-checkpoint Laplace setting. Its precision is $\alpha D$, where
$D=\operatorname{diag}(H)$ and $H$ is the summed last-layer GGN. We
compute $D$ from training data at each width and transfer $\alpha$ from
width $128$ to $2048$. The fitted classification heads remain fixed
throughout the precision search.

\begin{table}[ht]
\centering\small
\caption{Diagonal neural g-prior on our language models
($128\to2048$, ten tasks, one checkpoint seed). $|\Delta\log_2\alpha|$ is the
mean absolute proxy--target selection gap and $\Delta\NLL$ is mean
transferred-minus-target-local test NLL, both under the fixed-checkpoint plug-in evidence score.
Counts give tasks with both proxy and target selections strictly inside the
search grid. Validation-NLL gaps are omitted because only five of twenty
pairs satisfy this condition.}
\label{tab:gprior-internal}
\vspace{3pt}
\setlength{\tabcolsep}{5pt}
\begin{tabular}{l ccc c}
\toprule
Parametrization & $|\Delta\log_2\alpha|$ & $\Delta\NLL$
& Evidence: interior & Val. NLL: interior \\
\midrule
SP     & $3.95$ & $+0.0367$ & $10/10$ & $4/10$ \\
$\mup$ & $3.50$ & $+0.0229$ & $10/10$ & $1/10$ \\
\bottomrule
\end{tabular}
\end{table}

The mean absolute precision gap is $3.95$ bits for SP and $3.50$ bits
for $\mup$, with mean target test-NLL increases of $0.0367$ and $0.0229$,
respectively, when transferring the proxy precision
(Table~\ref{tab:gprior-internal}). Thus this normalization leaves
substantial precision drift in our fixed-checkpoint experiment.

\section{External validation on public u-\texorpdfstring{$\mup$}{muP}/SP pairs: details}
\label{app:umup-details}

\paragraph{Checkpoints.}
We use public unit-scaled $\mup$ (u-$\mup$) and SP checkpoints from
\citet{blake2025u} at approximately 1B ($d{=}2048$, 16 layers) and 7B
($d{=}4096$, 32 layers), pretrained on 300B SlimPajama tokens. Width and
depth increase together. We evaluate one checkpoint per size and
parametrization; SP training hyperparameters were not retuned for each size.
Checkpoint revisions and inference code are pinned. The extracted
final-normalization representations reproduce the released logits for both
1B models. Last-layer inference follows Appendix~\ref{app:ood-gate}.

\begin{table}[ht]
\centering\small
\caption{Public 1B$\to$7B prior-precision transfer on eight tasks excluding
Yahoo and DBpedia. Precision gaps are in bits (mean $\pm$ s.d.\ across tasks);
the $61$-point grid has $0.5$-bit spacing. $\Delta$NLL is mean target test NLL
at the proxy's precision minus that at the target's own precision, both
selected by fixed-checkpoint plug-in evidence. The two primary comparisons
cover all ten tasks in Table~\ref{tab:llm-lambda-public}. Sweep arrays use
$200$ MC samples in fp16.}
\label{tab:umup-lambda}
\vspace{3pt}
\begin{tabular}{l ccc}
\toprule
Cell & gap, plug-in evid.\ (bits) & gap, corrected evid.\ & $\Delta$NLL \\
\midrule
u-$\mup$ $+$ $\sigma$-metric & $0.19\pm0.26$ & $0.56\pm0.32$ & $+4\times10^{-5}$ \\
SP $+$ $\sigma$-metric       & $0.06\pm0.18$ & $0.38\pm0.23$ & $-1\times10^{-4}$ \\
u-$\mup$ $+$ isotropic             & $0.88\pm0.23$ & $1.19\pm0.65$ & $+1.1\times10^{-3}$ \\
SP $+$ isotropic (standard)        & $0.69\pm0.46$ & $0.56\pm0.42$ & $+8.6\times10^{-4}$ \\
\bottomrule
\end{tabular}
\end{table}

\paragraph{Precision transfer.}
Table~\ref{tab:umup-lambda} compares precision transfer across
the four combinations of parametrization and prior.
The corrected evidence score adds a quadratic estimate of the gain in log
posterior from locally adjusting the trained weights at each prior precision.
All selected precisions are interior to the grid. The correction changes the
selection gaps and method ordering: $\sigmat$ and SP with an isotropic prior
both have a mean gap of $0.56$ bits under the corrected criterion.

Using the gate protocol in Appendix~\ref{app:ood-gate}, both public-model
proxy gates retain approximately $99.8\%$ of in-distribution inputs and reject
$96\%$ of out-of-distribution inputs on average, with similar performance on
a common target.

\section{Technical formulation of the \texorpdfstring{$\sigmat$}{sigmaTransfer} prior}
\label{app:sn}

We derive the prior scales used in Section~\ref{sec:keyidea-revised}
for the MLP setting of Theorem~\ref{thm:prior-kernel-stability} and give
the implementation details.

\paragraph{Prior coordinates.}
For the zero-mean prior $\theta_n\sim\mathcal N(0,\lambda^{-1}S_n)$
with $\lambda>0$, write $S_n=\blockdiag(s_{1,n}I,\ldots,s_{G,n}I)$,
with one block per parameter tensor and $s_{g,n}>0$.
Taking $T_n=\blockdiag(\sqrt{s_{1,n}}I,\ldots,\sqrt{s_{G,n}}I)$ gives
$S_n=T_nT_n^\top$. The prior $\phi_n\sim\mathcal N(0,\lambda^{-1}I)$
then induces the stated prior on $\theta_n=T_n\phi_n$. At the trained checkpoint, the Jacobians
satisfy $J_{\phi,n}=J_nT_n$, so
\[
K_n(x,x')=J_n(x)S_nJ_n(x')^\top
        =J_{\phi,n}(x)J_{\phi,n}(x')^\top.
\]
A common width-independent factor in $S_n$ can be absorbed into
$\lambda$; relative block constants specify the prior.

\paragraph{MLP construction.}
Consider a scalar-output MLP with fixed depth $L$, fixed input dimension
$n_0=d_{\rm in}$, and hidden widths $n_1=\cdots=n_L=n$. Set $h^0(x)=x$ and
\[
\begin{aligned}
z_i^\ell(x)&=\frac{1}{\sqrt{n_{\ell-1}}}
  \sum_jW_{ij}^\ell h_j^{\ell-1}(x)+b_i^\ell,
\qquad h_i^\ell(x)=\operatorname{ReLU}(z_i^\ell(x)),\\
f_n(x)&=\frac1n a^\top h^L(x)+c.
\end{aligned}
\]
All quantities below are evaluated at $\hat\theta_n=\theta_{n,T}$,
with $T$ fixed independently of width and training satisfying
Assumptions~\ref{ass:tp-mup} and~\ref{ass:relu-kink}.
In these coordinates, choose
\[
\begin{aligned}
S_{W^1,n}&=n\,d_{\rm in}I,\qquad
S_{W^\ell,n}=nI\quad(2\leq\ell\leq L),\\
S_{b^\ell,n}&=nI,\qquad S_{a,n}=nI,\qquad S_{c,n}=1.
\end{aligned}
\]
Thus $T_n$ has factor $\sqrt{n d_{\rm in}}$ on the input-weight block,
$\sqrt n$ on the remaining weight and hidden-bias blocks, and $1$ on
the output bias.
For an internal matrix, $\widetilde W^\ell=W^\ell/\sqrt n$ is the
canonical $\mup$ matrix parameter; parameter and optimizer scalings
transform together under \citet[Proposition~2.2.3 and Definition~2.9.12]{yang2023tensor}.
The prior map $T_n$ is defined by the choice of $S_n$ above.

\paragraph{Kernel derivation.}
\label{app:sn-mlp-derivation}
For $\ell=1,\ldots,L$, define the sensitivity
$r_i^\ell(x)=n\,\partial f_n(x)/\partial z_i^\ell(x)$ and the Gram entries
\[
\Sigma_{\ell,n}(x,x')=\frac{h^\ell(x)^\top h^\ell(x')}{n},\qquad
\Pi_{\ell,n}(x,x')=\frac{r^\ell(x)^\top r^\ell(x')}{n}.
\]
For the input layer, set
$\Sigma_{0,n}(x,x')=\Sigma_0(x,x')=d_{\rm in}^{-1}x^\top x'$.
The parameter derivatives are
\[
\begin{aligned}
\frac{\partial f_n}{\partial a_i}&=\frac{h_i^L}{n},&
\frac{\partial f_n}{\partial b_i^\ell}&=\frac{r_i^\ell}{n},\\
\frac{\partial f_n}{\partial W_{ij}^\ell}
 &=\frac{r_i^\ell h_j^{\ell-1}}{n\sqrt{n_{\ell-1}}},&
\frac{\partial f_n}{\partial c}&=1.
\end{aligned}
\]
Multiplying by the prior blocks and summing gives the exact
finite-width identity
\begin{equation}
\label{eq:mlp-sigma-kernel-decomposition}
\begin{aligned}
K_n(x,x')={}&1+\Sigma_{L,n}(x,x')
+\Pi_{1,n}(x,x')\bigl(1+d_{\rm in}\Sigma_0(x,x')\bigr)\\
&+\sum_{\ell=2}^L\Pi_{\ell,n}(x,x')
\bigl(1+\Sigma_{\ell-1,n}(x,x')\bigr).
\end{aligned}
\end{equation}
Under the stated assumptions, Lemma~\ref{lem:tp-gram} establishes joint
almost-sure convergence of these Gram entries on every fixed finite
input set. The kernel therefore converges by continuity of finite sums
and products. Corollary~\ref{cor:fixed-output-kernel} gives the extension
to any fixed output dimension.

\paragraph{Stored-coordinate scales.}
\label{app:stored-scales}
Under an invertible linear change $\theta'_n=R_n\theta_n$, the prior and
Jacobian transform as $S'_n=R_nS_nR_n^\top$ and $J'_n=J_nR_n^{-1}$,
preserving $J_nS_nJ_n^\top$. Absorbing the forward factors into
$W_{\rm stored}^\ell=W^\ell/\sqrt{n_{\ell-1}}$ and
$a_{\rm stored}=a/n$ gives input-weight scale $n$, internal-weight
scales $n/n_{\ell-1}$ for $\ell\geq2$, and readout scale $1/n$.
For a readout $f=c_{\rm out}w^\top h+c$, the weight scale is
$1/(\mathrm{fan\_in}\,c_{\rm out}^2)$.
Table~\ref{tab:stored-scales} lists the implementation scales.

\begin{table}[ht]
\centering\small
\begin{tabular}{ll}
\toprule
Parameter block & Stored-coordinate $s_{g,n}$ \\
\midrule
Input weight & $\mathrm{fan\_out}$ \\
Hidden weight & $\mathrm{fan\_out}/\mathrm{fan\_in}$ \\
Hidden bias & $\mathrm{fan\_out}$ \\
\texttt{MuReadout} weight & $1/(\mathrm{fan\_in}\,c_{\rm out}^2)$ \\
Output bias & $1$ \\
\bottomrule
\end{tabular}
\caption{Prior-covariance scales for Linear/MLP parameters, with
$c_{\rm out}=\texttt{output\_mult}/\texttt{width\_mult}$.}
\label{tab:stored-scales}
\end{table}

\paragraph{Implementation and checks.}
\label{app:implementation}
The adapter supports \texttt{nn.Linear} and \texttt{MuReadout}
parameters. It supplies \texttt{laplace-torch} with precision
$\lambda/s_{g,n}$ for every parameter in block $g$: a per-parameter
vector for full or diagonal curvature, and one scalar per parameter
tensor for KFAC. The structured-curvature results use undamped KFAC
(\texttt{damping=False}), whose blocks are $Q_g+(\lambda/s_{g,n})I$.
For \texttt{MuReadout}, \texttt{width\_mult} is frozen before functional
Jacobian evaluation, and extracted last-layer features include
$c_{\rm out}$ to match the forward pass.

The prior-kernel coordinate identity above is checked numerically to
machine precision. Full-curvature posterior-precision checks against
$H_n+\lambda S_n^{-1}$ have relative Frobenius error below
$2\times10^{-8}$ for all-parameter regression, all-parameter
classification, and last-layer classification. The diagonal and
undamped-KFAC function-covariance identities in
Proposition~\ref{prop:structured-coordinate-covariance}, under blockwise
rescaling, are verified to relative errors below $2\times10^{-15}$.

\section{Proofs of the theoretical statements}
\label{app:proofs}

This appendix states the formal versions of
Theorems~\ref{thm:prior-kernel-stability}--\ref{thm:decision-transfer}
(summarized informally in Section~\ref{sec:proof-goals}) and gives
their proofs.
For the first result, finite-set forward and backward Gram convergence
is derived from TP-admissible finite-step $\mup$ training rather than
assumed. The subsequent posterior and transfer results additionally
assume stable data-space curvature, compact prior-precision search,
convergence of the selected Laplace parameter subset's prior-weighted
squared norm for fixed-checkpoint plug-in evidence-score transfer, and
nonzero decision margins wherever exact argmax agreement is claimed.

Throughout, $D$ denotes a finite training set, $V$ a finite validation
set, $\mathcal P$ a finite candidate pool, and $\U$ a fixed finite input
collection containing all points needed by a decision score. For finite input sets $A,B$, write
\[
    K_n(A,B)
    =
    J_n(A)S_nJ_n(B)^\top
\]
for the width-normalized prior kernel used by $\sigmat$, induced by the zero-mean weight prior
$\theta_n\sim\mathcal N(0,\sigma_p^2S_n)$.
The Laplace posterior mean and covariance at prior precision $\lambda=\sigma_p^{-2}$ are
\[
    m_n^\lambda(X)=f_n(X;\hat\theta_n),
    \qquad
    C_n^\lambda(X,X')
    =
    J_n(X)
    \left(
    H_n+\lambda S_n^{-1}
    \right)^{-1}
    J_n(X')^\top.
\]
The results after Theorem~\ref{thm:prior-kernel-stability} are
deterministic continuity statements applied pathwise. When their
premises are supplied by that theorem, convergence and eventual
agreement are understood almost surely. If analogous premises hold only
in probability, the numerical conclusions hold in probability and the
probability of the stated decision agreement tends to one.

\subsection{Formal statement and proof of Theorem~\ref{thm:prior-kernel-stability}: prior-kernel stability}
\label{app:proof-kernel}

\begin{assumption}[TP-admissible finite-step $\mup$ training]
\label{ass:tp-mup}
The depth, input and output dimensions, training and evaluation sets, and
number of optimizer steps $T<\infty$ are fixed independently of width.
The width sequence uses the Gaussian setup and general $\mup$ scaling of
\citet[Setup~2.6.3 and Definition~2.9.12]{yang2023tensor}, up to the
equivalent finite-width reparametrization of their Proposition~2.2.3.
Wide parameter groups are independently initialized, zero-initialized
biases and deterministic scalar parameters are handled as in their
Remarks~2.9.10 and~2.9.27, and all widths are realized on a common
probability space. The minibatch path and other external training choices
are shared across widths. Apart from the ReLU derivative treated below,
the loss-error and optimizer-history maps satisfy the regularity
conditions of \citet[Assumption~2.9.17]{yang2023tensor}; in particular,
the optimizer is
entrywise and globally Lipschitz in each fixed finite gradient history.
This includes SGD and positive-stabilizer Adam/AdamW with the prescribed
$\mup$ gradient, or equivalently stabilizer, scaling
\citep[Remark~2.2.6]{yang2023tensor}. When used, decoupled weight decay has
width-independent multiplicative shrinkage as in their Section~2.10.1.
\end{assumption}

\begin{assumption}[ReLU smoothing condition]
\label{ass:relu-kink}
Use the convention $\operatorname{ReLU}'(0)=0$.
At every $t\leq T$, layer $\ell\leq L$, and input $x$ in the fixed
training/evaluation set, the empirical preactivation law has no
asymptotic atom at the ReLU kink:
\[
    \lim_{\delta\downarrow0}\limsup_{n\to\infty}
    \frac1{n_\ell}\sum_{i=1}^{n_\ell}
    \mathbf 1\!\left\{
        \abs{z_{i,n,t}^{\ell}(x)}\leq\delta
    \right\}
    =0
    \qquad\text{almost surely}.
\]
For every fixed $p<\infty$, the empirical $p$th moments of all forward,
backward, and optimizer-history coordinates through time $T$ are almost
surely bounded uniformly in $n$, using $n^{-1}\sum_i$ for vector histories
and $n^{-2}\sum_{i,j}$ for matrix-indexed histories.
\end{assumption}

\begin{lemma}[Finite-horizon ReLU stability]
\label{lem:relu-stability}
Under Assumptions~\ref{ass:tp-mup} and~\ref{ass:relu-kink}, couple the
exact-ReLU network with a network trained from the same initialization
and along the same external training path, but with ReLU replaced by a
smooth function $\varphi_\delta$ satisfying
\[
    \sup_u\abs{\varphi_\delta(u)-\operatorname{ReLU}(u)}
    \leq C\delta,
    \qquad
    0\leq\varphi_\delta'\leq1,
    \qquad
    \varphi_\delta'(u)=\mathbf 1\{u>0\}
    \quad\text{when }\abs{u}>\delta .
\]
Along a fixed countable sequence $\delta\downarrow0$, for every fixed
$t\leq T$, layer $\ell$, and input $x$,
\[
    \lim_{\delta\downarrow0}\limsup_{n\to\infty}
    \left(
        \norm{h_{n,t}^{\ell,\delta}(x)-h_{n,t}^{\ell}(x)}_{n,2}
        +
        \norm{r_{n,t}^{\ell,\delta}(x)-r_{n,t}^{\ell}(x)}_{n,2}
    \right)
    =0
    \qquad\text{almost surely},
\]
where $\norm{v}_{n,2}^2=n^{-1}\sum_i\abs{v_i}^2$. The analogous
discrepancies vanish in operator norm for each effective hidden matrix,
in $\norm{\cdot}_{n,2}$ for each wide vector parameter, and in absolute
value for scalar parameters. Consequently, the trained outputs and the
forward and backward Gram entries are asymptotically unchanged by
smoothing.
\end{lemma}

\begin{proof}
The forward nonlinearity obeys
\[
    \norm{\varphi_\delta(u^\delta)-\operatorname{ReLU}(u)}_{n,2}
    \leq
    \norm{u^\delta-u}_{n,2}+C\delta .
\]
For the derivative gate, whenever $\epsilon>\delta$,
\[
    \abs{\mathbf 1\{u>0\}-\varphi_\delta'(v)}
    \leq
    \mathbf 1\{\abs{u}\leq\epsilon\}
    +
    \mathbf 1\{\abs{u-v}\geq\epsilon-\delta\}.
\]
After averaging, Assumption~\ref{ass:relu-kink} controls the first term,
whereas Markov's inequality bounds the second by
$\norm{u-v}_{n,2}^2/(\epsilon-\delta)^2$. Sending first $n\to\infty$,
then $\delta\downarrow0$, and finally $\epsilon\downarrow0$ gives
normalized $L^2$ convergence of the gates. The moment condition and
H\"older's inequality give the same conclusion after multiplication by
the corresponding reverse signals.

Updates preserve the coupling. For an effective hidden matrix,
$\Delta\widetilde W_{ij}=\eta_\vartheta n^{-1}
Q_{\vartheta,t}(G_{ij})$, where $G_{ij}$ is its finite gradient history.
Global Lipschitz continuity of $Q_{\vartheta,t}$ gives
\[
    \norm{\Delta\widetilde W^\delta-\Delta\widetilde W}_{\op}
    \leq
    \eta_\vartheta L_{\vartheta,t}
    \left[
        \frac1{n^2}\sum_{i,j}
        \norm{G_{ij}^\delta-G_{ij}}_2^2
    \right]^{1/2}.
\]
Each history component is a finite sum of forward--reverse outer
products, so the right-hand side vanishes by Cauchy--Schwarz and the
preceding estimates. The same calculation, using
$\abs{Q(g)}\leq\abs{Q(0)}+L\norm{g}_2$ and the moment bound, shows that
every increment has $O(1)$ operator norm. Gaussian initialization
supplies $O(1)$ effective-matrix norms
\citep[Proposition~2.6.8]{yang2023tensor}; hence the trained effective
matrices remain $O(1)$ through the fixed number of steps. Moreover,
\[
    \norm{W^\delta v^\delta-Wv}_{n,2}
    \leq
    \norm{W^\delta}_{\op}\norm{v^\delta-v}_{n,2}
    +
    \norm{W^\delta-W}_{\op}\norm{v}_{n,2},
\]
and likewise for transpose multiplication. Vector and scalar updates
are analogous, and fixed multiplicative weight decay is linear.

These estimates close a finite induction through every forward,
reverse, and update instruction. Finally, the normalized readout
$f=n^{-1}a^\top h^L+c$ is stable by Cauchy--Schwarz, as are the empirical
forward and backward inner products. This proves the lemma.
\end{proof}

\begin{lemma}[Tensor-program Gram convergence]
\label{lem:tp-gram}
Let $h_{n,t}^\ell(x)$ be the layer-$\ell$ activations and define the
rescaled backward sensitivities
\[
    r_{n,t}^\ell(x)
    =
    n_\ell
    \frac{\partial f_n(x;\theta_{n,t})}
         {\partial z_n^\ell(x;\theta_{n,t})}.
\]
Under Assumptions~\ref{ass:tp-mup} and~\ref{ass:relu-kink}, for every
$t\leq T$, $\ell\leq L$, and pair $x,x'$ in a fixed finite input set,
\[
    \Sigma_{\ell,n}^{t}(x,x')
    =
    \frac1{n_\ell}\sum_i
    h_{i,n,t}^\ell(x)h_{i,n,t}^\ell(x'),
    \qquad
    \Pi_{\ell,n}^{t}(x,x')
    =
    \frac1{n_\ell}\sum_i
    r_{i,n,t}^\ell(x)r_{i,n,t}^\ell(x')
\]
converge jointly almost surely to finite limits. The corresponding
outputs converge jointly as well.
\end{lemma}

\begin{proof}
Fix $\delta$. The smoothed MLP is
$\mathrm{NE}\otimes\mathrm{OR}^{\top}$-representable
\citep[Section~2.6 and Example~2.9.6]{yang2023tensor}; its vector and
scalar biases are admissible parameter types
\citep[Definitions~2.9.1 and~2.9.7]{yang2023tensor}. The forward and
scaled reverse computations form one total program
\citep[Definitions~2.9.14 and~2.9.16]{yang2023tensor}, and the displayed
$1/n$ readout makes the reverse variables precisely the sensitivities
$r_{n,t}^{\ell}$. The smooth derivative and the loss and optimizer maps
satisfy the required regularity conditions. Thus fixed-$T$ training,
including memoryful Adam/AdamW and decoupled weight decay, is covered by
\citet[Theorem~2.9.30 and Section~2.10.1]{yang2023tensor}.

Append the coordinatewise products
\[
    h_{i,n,t}^{\ell}(x)h_{i,n,t}^{\ell}(x'),
    \qquad
    r_{i,n,t}^{\ell}(x)r_{i,n,t}^{\ell}(x'),
\]
and their averaging instructions to this program. The product map is
pseudo-Lipschitz, so the
$\mathrm{NE}\otimes\mathrm{OR}^{\top}$ Master Theorem
\citep[Theorem~2.6.10]{yang2023tensor} gives joint almost-sure
convergence of the two empirical averages and of the output for this
fixed $\delta$. Intersect the probability-one events over the chosen
countable sequence $\delta\downarrow0$.

Lemma~\ref{lem:relu-stability} and Cauchy--Schwarz make the exact-ReLU
Gram entries and outputs uniformly close, in the stated iterated limit,
to their smoothed counterparts. The fixed-$\delta$ limits are therefore
Cauchy as $\delta\downarrow0$, and a three-$\epsilon$ argument gives the
claimed limits for exact ReLU. There are only finitely many inputs,
layers, and training steps, so the convergence is joint.
\end{proof}

\begin{formaltheorem}[Formal statement of Theorem~\ref{thm:prior-kernel-stability}]
Consider the scalar-output, fixed-depth MLP and prior covariance $S_n$
defined in Appendix~\ref{app:sn}, with fixed input dimension and hidden
widths $n_1=\cdots=n_L=n$. Let
$\hat\theta_n=\theta_{n,T}$ be obtained under
Assumptions~\ref{ass:tp-mup} and~\ref{ass:relu-kink}, and let $J_n$ be
the Jacobian at this checkpoint. On a fixed finite input set $X$, write
$\Sigma_\ell(x,x')$ and $\Pi_\ell(x,x')$ for the almost-sure
$t=T$ limits in Lemma~\ref{lem:tp-gram}, and set
$\Sigma_0(x,x')=n_0^{-1}x^\top x'$. Then, jointly over $x,x'\in X$,
\[
    K_n(x,x')
    =
    J_n(x)S_nJ_n(x')^\top
    \xrightarrow{\mathrm{a.s.}}
    K_\infty(x,x'),
\]
where
\[
\begin{aligned}
K_\infty(x,x')={}&1+\Sigma_L(x,x')
+\Pi_1(x,x')\bigl(1+d_{\rm in}\Sigma_0(x,x')\bigr)\\
&+\sum_{\ell=2}^{L}\Pi_\ell(x,x')
\bigl(1+\Sigma_{\ell-1}(x,x')\bigr).
\end{aligned}
\]
The trained outputs $f_n(X;\theta_{n,T})$ converge jointly almost surely
as well.
\end{formaltheorem}

\begin{proof}
Equation~\eqref{eq:mlp-sigma-kernel-decomposition} is an exact
finite-width identity. By Lemma~\ref{lem:tp-gram}, all of its Gram
factors converge jointly almost surely on the fixed finite set $X$.
Since $L$ and $d_{\rm in}$ are fixed, continuity of finite products and sums gives the
displayed $K_\infty$ and almost-sure entrywise convergence of the finite
kernel matrix. Output convergence is the final assertion of
Lemma~\ref{lem:tp-gram}.
\end{proof}

\begin{corollary}[Fixed-dimensional outputs]
\label{cor:fixed-output-kernel}
Let the preceding network have a fixed output dimension $q$, a shared
hidden trunk, and separate readout vectors and biases for its outputs.
For output index $a$, define
\[
    r_i^{\ell,a}(x)
    =
    n_\ell\frac{\partial f_{n,a}(x)}{\partial z_i^\ell(x)},
\]
and, for output indices $a,b$, define
\[
    \Pi_{\ell,n}^{ab}(x,x')
    =
    \frac1{n_\ell}\sum_i
    r_i^{\ell,a}(x)r_i^{\ell,b}(x').
\]
Under Assumptions~\ref{ass:tp-mup} and~\ref{ass:relu-kink}, all these
entries converge jointly almost surely. The prior-kernel blocks
\[
    K_n^{ab}(x,x')
    :=
    J_{n,a}(x)S_nJ_{n,b}(x')^\top
    \xrightarrow{\mathrm{a.s.}}
    K_\infty^{ab}(x,x')
\]
converge jointly, where
\[
\begin{aligned}
K_\infty^{ab}(x,x')={}&\delta_{ab}\bigl(1+\Sigma_L(x,x')\bigr)
+\Pi_1^{ab}(x,x')\bigl(1+d_{\rm in}\Sigma_0(x,x')\bigr)\\
&+\sum_{\ell=2}^{L}\Pi_\ell^{ab}(x,x')
\bigl(1+\Sigma_{\ell-1}(x,x')\bigr).
\end{aligned}
\]
Here $\delta_{ab}$ is the Kronecker delta.
The $q$-dimensional trained outputs converge jointly almost surely as
well.
\end{corollary}

\begin{proof}
Definition~2.9.16 of \citet{yang2023tensor} places the forward and
backpropagation programs for all $q$ output coordinates in one total
program. Since $q$ is fixed, the proof of Lemma~\ref{lem:tp-gram} applies
jointly to every output and output pair, giving both output convergence
and the limits $\Pi_\ell^{ab}$. Separate readout and output-bias blocks
contribute the factor $\delta_{ab}$; the shared hidden blocks give the
displayed sum.
\end{proof}

\subsection{Coordinate-invariant interpretation of Theorem~\ref{thm:prior-kernel-stability}}

Let $\phi_n$ denote the width-normalized local perturbation coordinates
used by $\sigmat$ and suppose the stored implementation parameters satisfy
\[
    \theta_n=T_n\phi_n,
\]
where $T_n$ is block diagonal. Then
\[
    J_{\phi,n}(X)
    =
    J_{\theta,n}(X)T_n.
\]
An isotropic zero-mean prior in these width-normalized coordinates,
\[
    \phi_n\sim\mathcal N(0,\sigma_p^2I),
\]
induces in stored coordinates
\[
    \theta_n
    =
    T_n\phi_n
    \sim
    \mathcal N(0,\sigma_p^2T_nT_n^\top).
\]
Therefore
\[
    S_n=T_nT_n^\top.
\]
The induced prior kernel is
\[
    J_{\theta,n}(X)S_nJ_{\theta,n}(X')^\top
    =
    J_{\theta,n}(X)T_nT_n^\top J_{\theta,n}(X')^\top
    =
    J_{\phi,n}(X)J_{\phi,n}(X')^\top.
\]
Thus the $\sigmat$ prior is equivalently an isotropic local prior in
width-normalized perturbation coordinates. Prior-kernel stability is the
statement that this normalized-coordinate Jacobian kernel has a stable
large-width limit.

\subsection{Formal statement and proof of Theorem~\ref{thm:posterior-covariance-stability}: posterior-covariance stability}
\label{app:proof-posterior}

\begin{formaltheorem}[Formal statement of Theorem~\ref{thm:posterior-covariance-stability}]
Fix finite sets $D,X,X'$, and assume $S_n\succ0$ for every $n$. Suppose
that for every $A,B\subseteq D\cup X\cup X'$,
\[
    K_n(A,B)
    =
    J_n(A)S_nJ_n(B)^\top
    \longrightarrow
    K_\infty(A,B).
\]
Assume the Laplace curvature has generalized Gauss--Newton form
$H_n = J_n(D)^\top\Omega_nJ_n(D)$,
where $\Omega_n\succeq0$ and $\Omega_n\to\Omega_\infty$ in operator norm.
Then, for every compact interval
$\Lambda=[\lambda_-,\lambda_+]\subset(0,\infty)$,
\[
    \sup_{\lambda\in\Lambda}
    \norm{
    C_n^\lambda(X,X')-C_\infty^\lambda(X,X')
    }_{\op}
    \to0,
\]
where
\begingroup\small
\[
\begin{aligned}
C_\infty^\lambda(X,X')
&=
\lambda^{-1}K_\infty(X,X')
-
\lambda^{-2}
K_\infty(X,D)\Omega_\infty^{1/2}
\bigl(
I+\lambda^{-1}
\Omega_\infty^{1/2}
K_\infty(D,D)
\Omega_\infty^{1/2}
\bigr)^{-1}
\\
&\qquad\times
\Omega_\infty^{1/2}K_\infty(D,X').
\end{aligned}
\]
\endgroup
\end{formaltheorem}

\begin{proof}
For any finite set $A$, define the whitened Jacobian
\[
    G_{A,n}
    =
    J_n(A)S_n^{1/2}.
\]
Then
\[
    K_n(A,B)
    =
    G_{A,n}G_{B,n}^\top.
\]
The GGN posterior precision can be written as
\[
\begin{aligned}
    J_n(D)^\top\Omega_nJ_n(D)+\lambda S_n^{-1}
    &=
    S_n^{-1/2}
    \left(
    G_{D,n}^\top\Omega_nG_{D,n}
    +
    \lambda I
    \right)
    S_n^{-1/2}.
\end{aligned}
\]
Therefore
\[
\begin{aligned}
    C_n^\lambda(X,X')
    &=
    J_n(X)
    \left(
    J_n(D)^\top\Omega_nJ_n(D)+\lambda S_n^{-1}
    \right)^{-1}
    J_n(X')^\top
    \\
    &=
    G_{X,n}
    \left(
    G_{D,n}^\top\Omega_nG_{D,n}
    +
    \lambda I
    \right)^{-1}
    G_{X',n}^\top.
\end{aligned}
\]

Apply the Woodbury identity to
\[
    G_{D,n}^\top\Omega_nG_{D,n}+\lambda I.
\]
With
\[
    M_n(\lambda)
    =
    I+\lambda^{-1}
    \Omega_n^{1/2}
    G_{D,n}G_{D,n}^\top
    \Omega_n^{1/2},
\]
we obtain
\[
\begin{aligned}
    C_n^\lambda(X,X')
    &=
    \lambda^{-1}
    G_{X,n}G_{X',n}^\top
    \\
    &\quad
    -
    \lambda^{-2}
    G_{X,n}G_{D,n}^\top
    \Omega_n^{1/2}
    M_n(\lambda)^{-1}
    \Omega_n^{1/2}
    G_{D,n}G_{X',n}^\top.
\end{aligned}
\]
Since
\[
    G_{A,n}G_{B,n}^\top
    =
    K_n(A,B),
\]
this becomes the finite-dimensional data-space expression
\[
\begin{aligned}
    C_n^\lambda(X,X')
    &=
    \lambda^{-1}K_n(X,X')
    \\
    &\quad
    -
    \lambda^{-2}
    K_n(X,D)\Omega_n^{1/2}
    \left(
    I+\lambda^{-1}
    \Omega_n^{1/2}
    K_n(D,D)
    \Omega_n^{1/2}
    \right)^{-1}
    \Omega_n^{1/2}K_n(D,X').
\end{aligned}
\]

By assumption,
\[
    K_n(A,B)\to K_\infty(A,B)
\]
for all finite $A,B\subseteq D\cup X\cup X'$, and
\[
    \Omega_n\to\Omega_\infty.
\]
The square-root map is continuous on finite-dimensional positive semidefinite matrices, so
\[
    \Omega_n^{1/2}\to\Omega_\infty^{1/2}.
\]
Moreover,
\[
    M_n(\lambda)
    =
    I+\lambda^{-1}
    \Omega_n^{1/2}
    K_n(D,D)
    \Omega_n^{1/2}
    \succeq I.
\]
Hence $M_n(\lambda)$ is invertible and
\[
    \norm{M_n(\lambda)^{-1}}_{\op}\le 1.
\]
The convergence of $M_n(\lambda)$ to
\[
    M_\infty(\lambda)
    =
    I+\lambda^{-1}
    \Omega_\infty^{1/2}
    K_\infty(D,D)
    \Omega_\infty^{1/2}
\]
is uniform for $\lambda\in[\lambda_-,\lambda_+]$, because $\lambda^{-1}$ is bounded and continuous on this compact interval. Since both $M_n(\lambda)$ and $M_\infty(\lambda)$ are bounded below by $I$,
\[
    \norm{
    M_n(\lambda)^{-1}-M_\infty(\lambda)^{-1}
    }_{\op}
    \le
    \norm{
    M_n(\lambda)-M_\infty(\lambda)
    }_{\op},
\]
so the inverses also converge uniformly. Substituting each uniformly convergent finite-dimensional factor into the Woodbury expression yields
\[
    \sup_{\lambda\in\Lambda}
    \norm{
    C_n^\lambda(X,X')-C_\infty^\lambda(X,X')
    }_{\op}
    \to0.
\]
\end{proof}

\begin{corollary}[Last-layer Laplace stability]
\label{cor:last-layer-stability}
For the MLP of Appendix~\ref{app:sn}, consider a last-layer Laplace approximation over only the readout weights $a$ and output bias $c$. If the final-layer feature Gram matrix satisfies
\[
    \frac{1}{n_L}
    \sum_{i=1}^{n_L}
    h_i^L(x)h_i^L(x')
    \to
    \Sigma_L(x,x'),
\]
then the last-layer prior kernel
\[
    K_{LL,n}(x,x')
    =
    1+
    \frac{1}{n_L}
    \sum_{i=1}^{n_L}
    h_i^L(x)h_i^L(x')
\]
converges to
\[
    K_{LL,\infty}(x,x')
    =
    1+\Sigma_L(x,x').
\]
Consequently, the last-layer Laplace posterior covariance is stable under the assumptions of Theorem~\ref{thm:posterior-covariance-stability}.
\end{corollary}

\begin{proof}
For the readout weights,
\[
    \frac{\partial f_n(x)}{\partial a_i}
    =
    \frac{h_i^L(x)}{n_L},
    \qquad
    S_{a,n}=n_LI_a.
\]
Therefore
\[
    n_L
    \sum_{i=1}^{n_L}
    \frac{\partial f_n(x)}{\partial a_i}
    \frac{\partial f_n(x')}{\partial a_i}
    =
    \frac{1}{n_L}
    \sum_{i=1}^{n_L}
    h_i^L(x)h_i^L(x').
\]
The output-bias contribution is $1$. Hence
\[
    K_{LL,n}(x,x')
    =
    1+
    \frac{1}{n_L}
    \sum_{i=1}^{n_L}
    h_i^L(x)h_i^L(x')
    \to
    1+\Sigma_L(x,x').
\]
The posterior-covariance statement then follows by applying Theorem~\ref{thm:posterior-covariance-stability} with $K_n=K_{LL,n}$.
\end{proof}

\begin{corollary}[Exact-Hessian residual condition]
\label{cor:exact-hessian-residual}
Under the kernel and curvature assumptions of
Theorem~\ref{thm:posterior-covariance-stability}, suppose the exact
Hessian decomposes as
\[
    H_n
    =
    J_n(D)^\top\Omega_nJ_n(D)+R_n.
\]
If
\[
    E_n
    =
    S_n^{1/2}R_nS_n^{1/2}
    \to0
\]
in operator norm, then the exact-Hessian Laplace posterior covariance has the same large-width limit as the GGN posterior covariance.
The convergence is uniform for $\lambda$ in every compact interval
$\Lambda\subset(0,\infty)$.
\end{corollary}

\begin{proof}
In whitened coordinates, the exact-Hessian covariance is
\[
    \widetilde C_n^\lambda(X,X')
    =
    G_{X,n}
    \left(
    G_{D,n}^\top\Omega_nG_{D,n}
    +
    E_n
    +
    \lambda I
    \right)^{-1}
    G_{X',n}^\top.
\]
Let
\[
    A_n(\lambda)
    =
    G_{D,n}^\top\Omega_nG_{D,n}
    +
    \lambda I.
\]
Then
\[
    A_n(\lambda)\succeq\lambda_- I,
    \qquad
    \norm{A_n(\lambda)^{-1}}_{\op}\le\lambda_-^{-1}
\]
uniformly for $\lambda\in\Lambda$. For large $n$,
$\norm{E_n}_{\op}<\lambda_-/2$, so
\[
    A_n(\lambda)+E_n\succeq\frac{\lambda_-}{2}I
\]
and
\[
    \norm{
    \left(A_n(\lambda)+E_n\right)^{-1}
    }_{\op}
    \le
    \frac{2}{\lambda_-}.
\]
The resolvent identity gives
\[
\begin{aligned}
    \widetilde C_n^\lambda(X,X')-C_n^\lambda(X,X')
    &=
    -
    G_{X,n}
    \left(A_n(\lambda)+E_n\right)^{-1}
    E_n
    A_n(\lambda)^{-1}
    G_{X',n}^\top.
\end{aligned}
\]
Since
\[
    \norm{G_{X,n}}_{\op}^2
    =
    \norm{K_n(X,X)}_{\op}
    =
    O(1)
\]
and similarly for $X'$, we obtain
\[
    \sup_{\lambda\in\Lambda}
    \norm{
    \widetilde C_n^\lambda(X,X')-C_n^\lambda(X,X')
    }_{\op}
    \le
    O(1)\frac{2}{\lambda_-^2}\norm{E_n}_{\op}
    \to0.
\]
Thus the exact-Hessian residual does not change the function-space covariance limit.
\end{proof}

\subsection{Structured-curvature extensions: diagonal and KFAC}
\label{app:structured-approximations}

Theorem~\ref{thm:posterior-covariance-stability} uses the full GGN because
its data-space form is determined entirely by $K_n$ and $\Omega_n$.
Diagonal and Kronecker-factored approximations retain the width-normalized coordinate
interpretation, but their width limits depend on more than $K_n$. We separate
these two facts below.

Let $\theta_n=T_n\phi_n$ map width-normalized coordinates $\phi_n$ to stored
coordinates $\theta_n$, let $S_n=T_nT_n^\top$, and write
$J_{\phi,n}=J_{\theta,n}T_n$. For an approximate positive-semidefinite
curvature $\widetilde H_{\theta,n}$ in stored coordinates, define
$\widetilde H_{\phi,n}=T_n^\top\widetilde H_{\theta,n}T_n$ and
\[
    \widetilde C_n^\lambda(A,B)
    =
    J_n(A)
    \left(\widetilde H_{\theta,n}+\lambda S_n^{-1}\right)^{-1}
    J_n(B)^\top.
\]

\begin{proposition}[Finite-width coordinate covariance]
\label{prop:structured-coordinate-covariance}
At every finite width, the following stored-coordinate computations equal
the corresponding computations with an isotropic prior in the
width-normalized coordinates.
\begin{enumerate}[leftmargin=1.5em,itemsep=0.2em]
\item For diagonal Laplace, the claim holds whenever $T_n$ is diagonal.
\item For block-Kronecker Laplace, it holds whenever
$T_n=\blockdiag_g(t_{n,g}I_g)$ with $t_{n,g}>0$, the curvature blocks and prior groups are
aligned, and the block approximation $\mathcal Q_g$ is positively
homogeneous:
\[
    \mathcal Q_g(t^2H)=t^2\mathcal Q_g(H).
\]
The Kronecker statement concerns additive prior damping
$\mathcal Q_g(H)+\lambda I_g$, not multiplicative factor damping.
\end{enumerate}
The verified Linear/MLP width-normalized prior satisfies these structural
conditions because its scale is constant within every weight or bias tensor.
\end{proposition}

\begin{proof}
The exact curvature transforms as
$H_{\phi,n}=T_n^\top H_{\theta,n}T_n$. If $T_n$ is diagonal, then
\[
    \operatorname{diag}(H_{\phi,n})
    =
    T_n^\top\operatorname{diag}(H_{\theta,n})T_n.
\]
Consequently,
\[
\begin{aligned}
    &J_{\phi,n}(A)
    \left(\operatorname{diag}(H_{\phi,n})+\lambda I\right)^{-1}
    J_{\phi,n}(B)^\top
    \\
    &\quad=
    J_{\theta,n}(A)
    \left(
    \operatorname{diag}(H_{\theta,n})+\lambda S_n^{-1}
    \right)^{-1}
    J_{\theta,n}(B)^\top.
\end{aligned}
\]
For a Kronecker block, $T_{n,g}=t_{n,g}I_g$ gives
$H_{\phi,n,g}=t_{n,g}^2H_{\theta,n,g}$. Homogeneity therefore gives
$\mathcal Q_g(H_{\phi,n,g})=t_{n,g}^2\mathcal Q_g(H_{\theta,n,g})$.
In particular,
\[
\begin{aligned}
    &J_{\phi,n,g}(A)
    \left(\mathcal Q_g(H_{\phi,n,g})+\lambda I_g\right)^{-1}
    J_{\phi,n,g}(B)^\top
    \\
    &\quad=
    J_{\theta,n,g}(A)
    \left(
    \mathcal Q_g(H_{\theta,n,g})
    +\lambda t_{n,g}^{-2}I_g
    \right)^{-1}
    J_{\theta,n,g}(B)^\top.
\end{aligned}
\]
Summing the block contributions proves the second statement.
\end{proof}

Finite-width covariance does not by itself imply convergence across width.
The following proposition states one sufficient condition that covers both
diagonal and Kronecker curvature.

\begin{proposition}[Conditional width stability for structured curvature]
\label{prop:structured-width-stability}
Work in the width-normalized coordinates and set
$G_n(A)=J_{\theta,n}(A)T_n$. Let the structured curvature have spectral
decomposition
\[
    \widetilde H_{\phi,n}
    =
    \sum_{r=1}^{p_n}q_{n,r}u_{n,r}u_{n,r}^\top,
    \qquad q_{n,r}\ge0.
\]
For fixed finite sets $A,B$, define the matrix-valued spectral measure
\[
    \mathcal M_n^{A,B}
    =
    \sum_{r=1}^{p_n}
    \bigl(G_n(A)u_{n,r}\bigr)
    \bigl(G_n(B)u_{n,r}\bigr)^\top
    \delta_{q_{n,r}}.
\]
For finite signed measures, write
\[
    d_{\mathrm{BL}}(\mu,\nu)
    =
    \sup_{\norm{h}_\infty+\operatorname{Lip}(h)\le1}
    \left|\int h\,d(\mu-\nu)\right|.
\]
Suppose that every scalar entry of $\mathcal M_n^{A,B}$ converges in
$d_{\mathrm{BL}}$ to a finite signed measure $\mathcal M_\infty^{A,B}$
for all required $A,B$. Then, for every compact
$\Lambda=[\lambda_-,\lambda_+]\subset(0,\infty)$,
\[
    \sup_{\lambda\in\Lambda}
    \left\|
    \widetilde C_n^\lambda(A,B)
    -
    \int_{[0,\infty)}\frac{1}{t+\lambda}
    \,d\mathcal M_\infty^{A,B}(t)
    \right\|
    \longrightarrow0.
\]
For diagonal Laplace, $q_{n,r}$ are the per-coordinate diagonal curvatures.
For undamped KFAC, they are the eigenvalues of its block-Kronecker curvature.
\end{proposition}

\begin{proof}
The spectral theorem gives the exact representation
\[
\begin{aligned}
    \widetilde C_n^\lambda(A,B)
    &=
    G_n(A)(\widetilde H_{\phi,n}+\lambda I)^{-1}G_n(B)^\top
    \\
    &=
    \int_{[0,\infty)}\frac{1}{t+\lambda}
    \,d\mathcal M_n^{A,B}(t).
\end{aligned}
\]
For $h_\lambda(t)=(t+\lambda)^{-1}$ and
$\lambda\in[\lambda_-,\lambda_+]$,
\[
    \norm{h_\lambda}_\infty\le\lambda_-^{-1},
    \qquad
    \operatorname{Lip}(h_\lambda)\le\lambda_-^{-2}.
\]
The family $\{h_\lambda:\lambda\in\Lambda\}$ therefore has uniformly
bounded norm in the bounded-Lipschitz class. Applying the assumed convergence entrywise
and using equivalence of norms in the fixed-dimensional output space yields
the stated uniform convergence.
\end{proof}

\begin{corollary}[Structured-approximation validation-NLL and decision transfer]
\label{cor:structured-transfer}
Under Proposition~\ref{prop:structured-width-stability}, if the corresponding
predictive means converge uniformly and the assumptions of
Theorems~\ref{thm:lambda-transfer} and~\ref{thm:decision-transfer} hold, then
their validation-NLL precision-transfer and decision-transfer conclusions hold with
$C_n^\lambda$ replaced by $\widetilde C_n^\lambda$.
\end{corollary}

\begin{proof}
Proposition~\ref{prop:structured-width-stability} supplies the uniform
covariance convergence used in the proofs of
Theorems~\ref{thm:lambda-transfer} and~\ref{thm:decision-transfer}; the
remaining assumptions are unchanged.
\end{proof}

\paragraph{Status for the structured-curvature experiments.}
The finite-width coordinate-covariance identity is verified numerically,
but the projected spectral-measure convergence condition of
Proposition~\ref{prop:structured-width-stability} is not directly
tested. The diagonal/KFAC asymptotic extension is therefore conditional;
the reported finite-width transfer is empirical evidence for its
conclusion.

\paragraph{Why prior-kernel convergence alone is insufficient.}
Consider one training input $d$, one test input $x$, scalar output, and unit
likelihood curvature. Two normalized Jacobian pairs are
\[
\begin{array}{c|cc}
 & G(d) & G(x) \\
\hline
\text{I} & (1,0) & (0,1) \\
\text{II} & 2^{-1/2}(1,1) & 2^{-1/2}(-1,1).
\end{array}
\]
Both have exactly the same prior kernel:
$K(d,d)=K(x,x)=1$ and $K(d,x)=0$. Yet diagonal curvature gives
\[
    \widetilde C_{\mathrm{I}}^\lambda(x,x)=\lambda^{-1},
    \qquad
    \widetilde C_{\mathrm{II}}^\lambda(x,x)
    =(\lambda+\tfrac12)^{-1}.
\]
An alternating width sequence therefore has a constant $K_n$ but a
nonconvergent diagonal posterior covariance. The example does not challenge
Proposition~\ref{prop:structured-coordinate-covariance}: it uses a rotation,
whereas the width-normalized prior is diagonal and block-scalar. It shows why the
extra spectral condition in Proposition~\ref{prop:structured-width-stability}
is necessary. KFAC likewise depends on its retained block-factor spectrum,
not only on the sum $G_nG_n^\top$.

\paragraph{Other approximations.}
Proposition~\ref{prop:structured-width-stability} applies more generally to
any positive-semidefinite curvature approximation whose projected spectral
measures satisfy its assumption. The implementation checks cover full,
diagonal, and KFAC curvature for Linear/MLP models.

\subsection{Formal statement and proof of Theorem~\ref{thm:lambda-transfer}: inference-hyperparameter transfer}
\label{app:proof-precision}

\begin{formaltheorem}[Formal statement of Theorem~\ref{thm:lambda-transfer}]
Let $\Lambda=[\lambda_-,\lambda_+]\subset(0,\infty)$ be the prior-precision search interval.
Suppose that on a finite validation set $V$, the predictive mean and
covariance converge uniformly in $\lambda$:
\[
    \sup_{\lambda\in\Lambda}
    \norm{m_n^\lambda(V)-m_\infty^\lambda(V)} \to 0,
    \qquad
    \sup_{\lambda\in\Lambda}
    \norm{C_n^\lambda(V,V)-C_\infty^\lambda(V,V)} \to 0.
\]
For scalar Gaussian regression with observation noise $\tau^2>0$, or for softmax classification with Gaussian logit posterior, the validation NLL curves converge uniformly:
$\sup_{\lambda\in\Lambda} |\NLL_n(\lambda)-\NLL_\infty(\lambda)| \to 0$.
If $\NLL_\infty$ has a unique minimizer $\lambda^\star$, then every sequence
of minimizers satisfies $\lambda_n^\star \to \lambda^\star$.
Moreover, for a proxy width $n$ and target width $N$,
\[
    \NLL_N(\lambda_n^\star)
    -
    \min_{\lambda\in\Lambda}\NLL_N(\lambda)
    \le
    2\sup_{\lambda\in\Lambda}|\NLL_n(\lambda)-\NLL_N(\lambda)|.
\]
\end{formaltheorem}

\begin{proof}
Let
\[
    \mathcal L_n(\lambda)=\NLL_n(\lambda).
\]
Because $\lambda\ge\lambda_->0$, the finite-width Laplace mean and
covariance are continuous in $\lambda$; their uniform limits are therefore
continuous as well.
We first show uniform convergence of the validation NLL curves from uniform convergence of the predictive posterior mean and covariance.

\paragraph{Gaussian regression.}
For the usual sum of marginal predictive NLLs, put
\[
    s_{n,\lambda,i}
    =
    C_{n,\lambda}(x_i,x_i)+\tau^2,
    \qquad \tau^2>0.
\]
Then
\[
    \mathcal L_n(\lambda)
    =
    \frac12\sum_{i=1}^{|V|}
    \left[
    \frac{(y_i-m_{n,\lambda}(x_i))^2}{s_{n,\lambda,i}}
    +
    \log s_{n,\lambda,i}
    +
    \log(2\pi)
    \right].
\]
Positive semidefiniteness gives
$s_{n,\lambda,i}\geq\tau^2$. The assumptions imply uniform convergence
of every $m_{n,\lambda}(x_i)$ and $s_{n,\lambda,i}$. Because the
validation set and output dimension are fixed, the arguments remain in
a compact subset of
$\mathbb R\times[\tau^2,\infty)$, on which each displayed summand is
uniformly continuous. Therefore
\[
    \sup_{\lambda\in\Lambda}
    \abs{
    \mathcal L_n(\lambda)-\mathcal L_\infty(\lambda)
    }
    \to0.
\]

\paragraph{Softmax classification.}
For classification, let the Laplace posterior over logits at validation point $x_i$ be
\[
    f_i\sim
    \mathcal N
    \left(
    m_{n,\lambda}(x_i),
    C_{n,\lambda}(x_i,x_i)
    \right).
\]
The predictive probability of the observed class $y_i$ is
\[
    p_{n,\lambda}(y_i\mid x_i)
    =
    \E_{f_i}
    \left[
    \operatorname{softmax}_{y_i}(f_i)
    \right].
\]
The validation NLL is
\[
    \mathcal L_n(\lambda)
    =
    -
    \sum_{i=1}^{|V|}
    \log p_{n,\lambda}(y_i\mid x_i).
\]
The map
\[
    f\mapsto\operatorname{softmax}_{y_i}(f)
\]
is bounded and continuous. In finite dimension, expectations of bounded continuous functions under a Gaussian distribution depend continuously on the Gaussian mean and covariance. Because the posterior means and covariances converge uniformly over compact $\Lambda$, we get
\[
    \sup_{\lambda\in\Lambda}
    \abs{
    p_{n,\lambda}(y_i\mid x_i)
    -
    p_{\infty,\lambda}(y_i\mid x_i)
    }
    \to0
\]
for each $i$.
More explicitly, the limiting mean--covariance pairs form a compact
image of $\Lambda$, and the finite-width pairs eventually remain in a
compact neighborhood of that image; the Gaussian-expectation map is
therefore uniformly continuous on the relevant set.

For every finite mean and covariance,
\[
    p_{\infty,\lambda}(y_i\mid x_i)>0.
\]
The map
\[
    \lambda\mapsto p_{\infty,\lambda}(y_i\mid x_i)
\]
is continuous on compact $\Lambda$, so
\[
    \rho_i
    =
    \inf_{\lambda\in\Lambda}
    p_{\infty,\lambda}(y_i\mid x_i)
    >
    0.
\]
Uniform convergence then implies that for sufficiently large $n$,
\[
    p_{n,\lambda}(y_i\mid x_i)\ge\frac{\rho_i}{2}
\]
for all $\lambda\in\Lambda$. Since $-\log t$ is uniformly continuous on $[\rho_i/2,1]$,
\[
    \sup_{\lambda\in\Lambda}
    \abs{
    -\log p_{n,\lambda}(y_i\mid x_i)
    +
    \log p_{\infty,\lambda}(y_i\mid x_i)
    }
    \to0.
\]
Summing over the finite validation set gives
\[
    \sup_{\lambda\in\Lambda}
    \abs{
    \mathcal L_n(\lambda)-\mathcal L_\infty(\lambda)
    }
    \to0.
\]
The same argument covers a deterministic continuous probit
approximation and a fixed common-random-number Monte Carlo rule, using
the symmetric covariance square root for reparameterization.
Independent fixed-sample Monte Carlo draws at each width retain
simulation noise and are not covered without a growing sample count or
an additional probabilistic error bound.

It remains to prove convergence of minimizers. Let
\[
    \lambda^\star
    =
    \argmin_{\lambda\in\Lambda}\mathcal L_\infty(\lambda)
\]
be the unique limiting minimizer, and let
\[
    \lambda_n^\star
    \in
    \argmin_{\lambda\in\Lambda}\mathcal L_n(\lambda).
\]
Fix $\varepsilon>0$ and define
\[
    B_\varepsilon
    =
    \{
    \lambda\in\Lambda:
    \abs{\lambda-\lambda^\star}\ge\varepsilon
    \}.
\]
Since $\Lambda$ is compact, $B_\varepsilon$ is compact. Since $\mathcal L_\infty$ is continuous and has a unique minimizer at $\lambda^\star$,
\[
    \gamma_\varepsilon
    =
    \inf_{\lambda\in B_\varepsilon}
    \left[
    \mathcal L_\infty(\lambda)-\mathcal L_\infty(\lambda^\star)
    \right]
    >
    0.
\]
Uniform convergence gives, for sufficiently large $n$,
\[
    \sup_{\lambda\in\Lambda}
    \abs{
    \mathcal L_n(\lambda)-\mathcal L_\infty(\lambda)
    }
    <
    \frac{\gamma_\varepsilon}{3}.
\]
If $\lambda_n^\star\in B_\varepsilon$, then
\[
    \mathcal L_\infty(\lambda_n^\star)
    \ge
    \mathcal L_\infty(\lambda^\star)+\gamma_\varepsilon.
\]
Therefore
\[
    \mathcal L_n(\lambda_n^\star)
    \ge
    \mathcal L_\infty(\lambda_n^\star)
    -
    \frac{\gamma_\varepsilon}{3}
    \ge
    \mathcal L_\infty(\lambda^\star)
    +
    \frac{2\gamma_\varepsilon}{3}.
\]
But also
\[
    \mathcal L_n(\lambda^\star)
    \le
    \mathcal L_\infty(\lambda^\star)
    +
    \frac{\gamma_\varepsilon}{3}.
\]
Thus
\[
    \mathcal L_n(\lambda_n^\star)
    >
    \mathcal L_n(\lambda^\star),
\]
contradicting the fact that $\lambda_n^\star$ minimizes $\mathcal L_n$. Hence, for sufficiently large $n$,
\[
    \abs{\lambda_n^\star-\lambda^\star}<\varepsilon.
\]
Since $\varepsilon>0$ was arbitrary,
\[
    \lambda_n^\star\to\lambda^\star.
\]

Finally, for a proxy width $n$ and target width $N$, define
\[
    \Delta_{n,N}
    =
    \sup_{\lambda\in\Lambda}
    \abs{
    \mathcal L_n(\lambda)-\mathcal L_N(\lambda)
    }.
\]
Let
\[
    \lambda_N^\star
    \in
    \argmin_{\lambda\in\Lambda}
    \mathcal L_N(\lambda).
\]
Because $\lambda_n^\star$ minimizes $\mathcal L_n$,
\[
    \mathcal L_n(\lambda_n^\star)
    \le
    \mathcal L_n(\lambda_N^\star).
\]
Therefore
\[
\begin{aligned}
    \mathcal L_N(\lambda_n^\star)
    -
    \mathcal L_N(\lambda_N^\star)
    &=
    \left[
    \mathcal L_N(\lambda_n^\star)
    -
    \mathcal L_n(\lambda_n^\star)
    \right]
    \\
    &\quad
    +
    \left[
    \mathcal L_n(\lambda_n^\star)
    -
    \mathcal L_n(\lambda_N^\star)
    \right]
    \\
    &\quad
    +
    \left[
    \mathcal L_n(\lambda_N^\star)
    -
    \mathcal L_N(\lambda_N^\star)
    \right]
    \\
    &\le
    \Delta_{n,N}+0+\Delta_{n,N}
    =
    2\Delta_{n,N}.
\end{aligned}
\]
This proves the transferred excess-NLL bound.
\end{proof}

\begin{corollary}[Evidence-score precision transfer]
\label{cor:evidence-transfer}
Retain the fixed-training-set and full-GGN assumptions of
Theorem~\ref{thm:posterior-covariance-stability}. Use the zero-mean
$\sigmat$ prior $\theta_n\sim\mathcal N(0,\lambda^{-1}S_n)$; for each $n$, hold all likelihood
hyperparameters fixed and do not refit $\hat\theta_n$ as $\lambda$ varies.
Fix a compact
$\Lambda=[\lambda_-,\lambda_+]\subset(0,\infty)$. Let
\[
q_n=\hat\theta_n^\top S_n^{-1}\hat\theta_n.
\]
Assume $q_n\to q_\infty<\infty$, and define
\[
B_n=\Omega_n^{1/2}K_n(D,D)\Omega_n^{1/2},
\qquad
B_\infty=\Omega_\infty^{1/2}K_\infty(D,D)\Omega_\infty^{1/2}.
\]
Define the centered fixed-checkpoint plug-in Laplace evidence score used in our
experiments by
\[
\mathcal E_n(\lambda)
=
-\frac{\lambda}{2}q_n
+\frac12\log\det(\lambda S_n^{-1})
-\frac12\log\det(H_n+\lambda S_n^{-1}),
\qquad \lambda\in\Lambda.
\]
Terms independent of $\lambda$, including the likelihood at
$\hat\theta_n$, are omitted because they do not affect the selected
precision within this fixed-checkpoint score. Let
\[
\mathcal E_\infty(\lambda)
=
-\frac{\lambda}{2}q_\infty
-\frac12\log\det\!\left(I+\lambda^{-1}B_\infty\right).
\]
Then
\[
\sup_{\lambda\in\Lambda}
\abs{\mathcal E_n(\lambda)-\mathcal E_\infty(\lambda)}
\to0.
\]
If $\mathcal E_\infty$ has a unique maximizer
$\lambda_{\mathrm{ev}}^\star$, then every sequence
$\lambda_{n,\mathrm{ev}}^\star\in\argmax_{\lambda\in\Lambda}
\mathcal E_n(\lambda)$ satisfies
$\lambda_{n,\mathrm{ev}}^\star\to\lambda_{\mathrm{ev}}^\star$.
Moreover, for proxy width $n$ and target width $N$,
\[
\max_{\lambda\in\Lambda}\mathcal E_N(\lambda)
-\mathcal E_N(\lambda_{n,\mathrm{ev}}^\star)
\le
2\sup_{\lambda\in\Lambda}
\abs{\mathcal E_n(\lambda)-\mathcal E_N(\lambda)}.
\]
For a fixed finite grid $G\subset\Lambda$, if
$\lambda_{\mathrm{ev},G}^\star=
\argmax_{\lambda\in G}\mathcal E_\infty(\lambda)$ is unique, then
\[
\argmax_{\lambda\in G}\mathcal E_n(\lambda)
=\lambda_{\mathrm{ev},G}^\star
\]
for all sufficiently large $n$.
\end{corollary}

\begin{proof}
By definition, $\mathcal E_n$ is the $\lambda$-dependent part of the
fixed-checkpoint plug-in score. Let
$\widetilde J_n=J_n(D)S_n^{1/2}$.
Factoring $S_n^{-1/2}$ from the two parameter-space determinants and
then applying Sylvester's determinant identity gives
\[
\begin{aligned}
&\log\det(\lambda S_n^{-1})
-\log\det(H_n+\lambda S_n^{-1})
\\
&\quad=
-\log\det\!\left(I+\lambda^{-1}\widetilde J_n^\top\Omega_n\widetilde J_n\right)
=
-\log\det\!\left(I+\lambda^{-1}B_n\right).
\end{aligned}
\]
Hence
\[
\mathcal E_n(\lambda)
=
-\frac{\lambda}{2}q_n
-\frac12\log\det\!\left(I+\lambda^{-1}B_n\right).
\]
Each $B_n$ is positive semidefinite, so the matrix inside the
log-determinant is positive definite for every $\lambda>0$, including
when $\Omega_n$ is singular.
Because the stacked training output has fixed finite dimension,
convergence of $K_n(D,D)$ and $\Omega_n$, together with continuity of
the positive-semidefinite square-root map, implies
$B_n\to B_\infty$ in operator norm. If $r$ is the dimension of
$B_n$, Weyl's inequality and the Lipschitz bound for
$t\mapsto\log(1+t/\lambda)$ give
\[
\sup_{\lambda\in\Lambda}
\abs{\mathcal E_n(\lambda)-\mathcal E_\infty(\lambda)}
\le
\frac{\lambda_+}{2}\abs{q_n-q_\infty}
+\frac{r}{2\lambda_-}\norm{B_n-B_\infty}_{\op}
\to0.
\]
Applying the minimizer-consistency argument of
Theorem~\ref{thm:lambda-transfer} to $-\mathcal E_n$ proves convergence
of the maximizers. Applying the same three-term comparison used for
the excess-NLL bound proves the stated finite-width evidence-regret
bound.
The finite-grid conclusion follows from the same uniform convergence:
once the uniform error is smaller than half the positive gap between
the best and second-best limiting grid points, no other grid point can
overtake the maximizer.
\end{proof}

\paragraph{Scope of the evidence result.}
The convergence of $q_n$ is an additional assumption, not a consequence
of prior-kernel or posterior convergence: likelihood-null parameter
directions can change $q_n$ and the plug-in evidence-score optimum without changing
the predictive posterior. This condition can hold for normalized
last-layer readouts, but generally fails
for a zero-mean all-parameter prior containing $n\times n$ hidden
matrices: in the explicit-coordinate convention, one such block with
$O(1)$ entries contributes
$n^{-1}\norm{\hat W_n}_{F}^{2}=\Theta(n)$ to $q_n$. The corollary covers only the full-GGN
version of the fixed-checkpoint plug-in criterion on the chosen Laplace parameter
subset; $\hat\theta_n$ and $S_n$ are restricted to that subset. The
experiments do not by themselves establish asymptotic convergence of
$q_n$. This plug-in score is not the marginal likelihood obtained by
refitting the posterior mode at every $\lambda$. Such refitted evidence,
mode-relaxation corrections, and diagonal or
KFAC curvature require further mode, gradient, or log-determinant
stability assumptions and are not covered here.
The approximate-optimization and finite-grid statements below also
apply to this evidence criterion after replacing $\mathcal L_n$ by
$-\mathcal E_n$.

\begin{corollary}[Approximate hyperparameter optimization]
\label{cor:approx-hp}
If $\tilde\lambda_n$ is an approximate minimizer satisfying
\[
    \mathcal L_n(\tilde\lambda_n)
    \le
    \inf_{\lambda\in\Lambda}\mathcal L_n(\lambda)+\eta_n,
    \qquad
    \eta_n\to0,
\]
then $\tilde\lambda_n\to\lambda^\star$ under the assumptions of Theorem~\ref{thm:lambda-transfer}.
\end{corollary}

\begin{proof}
The proof is the same as the argmin-consistency part of Theorem~\ref{thm:lambda-transfer}. For fixed $\varepsilon>0$, use the same limiting margin $\gamma_\varepsilon>0$. For sufficiently large $n$,
\[
    \sup_{\lambda\in\Lambda}
    \abs{\mathcal L_n(\lambda)-\mathcal L_\infty(\lambda)}
    <
    \frac{\gamma_\varepsilon}{3},
    \qquad
    \eta_n<\frac{\gamma_\varepsilon}{3}.
\]
Any approximate minimizer outside the $\varepsilon$-ball around $\lambda^\star$ would then have objective value more than $\eta_n$ above the best attainable value, contradicting the definition of $\tilde\lambda_n$.
\end{proof}

\begin{corollary}[Finite-grid hyperparameter search]
\label{cor:grid-hp}
Let $G=\{\lambda_1,\ldots,\lambda_R\}\subset\Lambda$ be a finite grid. If
\[
    \max_{\lambda\in G}
    \abs{
    \mathcal L_n(\lambda)-\mathcal L_\infty(\lambda)
    }
    \to0
\]
and the limiting grid optimum
\[
    \lambda_G^\star
    =
    \argmin_{\lambda\in G}
    \mathcal L_\infty(\lambda)
\]
is unique, then
\[
    \argmin_{\lambda\in G}
    \mathcal L_n(\lambda)
    =
    \lambda_G^\star
\]
for all sufficiently large $n$.
\end{corollary}

\begin{proof}
Define the grid margin
\[
    \gamma_G
    =
    \min_{\lambda\in G,\lambda\ne\lambda_G^\star}
    \left[
    \mathcal L_\infty(\lambda)-\mathcal L_\infty(\lambda_G^\star)
    \right]
    >
    0.
\]
Once
\[
    \max_{\lambda\in G}
    \abs{
    \mathcal L_n(\lambda)-\mathcal L_\infty(\lambda)
    }
    <
    \frac{\gamma_G}{2},
\]
no grid point can overtake $\lambda_G^\star$. Hence the grid minimizer agrees exactly for all sufficiently large $n$.
\end{proof}

\subsection{Formal statement and proof of Theorem~\ref{thm:decision-transfer}: decision and active-learning transfer}
\label{app:proof-decision}

\begin{formaltheorem}[Formal statement of Theorem~\ref{thm:decision-transfer}]
Let $\mathcal P=\{u_1,\ldots,u_M\}$ be a finite candidate pool, let
$\U\supseteq\mathcal P$ be a fixed finite input collection containing
all candidate and auxiliary target/evaluation points used by the scores,
and let $\mathcal A$ be a finite action set, such as
$\mathcal A=\mathcal P$ for single-point acquisition or
$\mathcal A=\{S\subset\mathcal P: |S|=K\}$ for batch acquisition. Fix $\lambda>0$. Suppose
$m_n^\lambda(\U)\to m_\infty^\lambda(\U)$ and $C_n^\lambda(\U,\U)\to C_\infty^\lambda(\U,\U)$.
Let
$A_n^\lambda(a) = F_a(m_n^\lambda(\U), C_n^\lambda(\U,\U),\lambda)$,
where each $F_a$ is continuous at
$(m_\infty^\lambda(\U),C_\infty^\lambda(\U,\U),\lambda)$. Then
$\max_{a\in\mathcal A} |A_n^\lambda(a)-A_\infty^\lambda(a)| \to 0$.
If $a_\infty^\star = \argmax_{a\in\mathcal A} A_\infty^\lambda(a)$
is unique with margin
$\gamma = A_\infty^\lambda(a_\infty^\star) - \max_{a\ne a_\infty^\star} A_\infty^\lambda(a) > 0$,
then for all sufficiently large $n$,
$\argmax_{a\in\mathcal A} A_n^\lambda(a) = a_\infty^\star$.
\end{formaltheorem}

\begin{proof}
For readability, write $A_{n,\lambda}=A_n^\lambda$ and use the same
subscript convention for $m$ and $C$ throughout this proof.
For each action $a\in\mathcal A$,
\[
    A_{n,\lambda}(a)
    =
    F_a
    \left(
    m_{n,\lambda}(\U),
    C_{n,\lambda}(\U,\U),
    \lambda
    \right).
\]
By assumption,
\[
    m_{n,\lambda}(\U)\to m_{\infty,\lambda}(\U),
    \qquad
    C_{n,\lambda}(\U,\U)\to C_{\infty,\lambda}(\U,\U).
\]
Since $F_a$ is continuous at the limiting posterior,
\[
    A_{n,\lambda}(a)\to A_{\infty,\lambda}(a)
\]
for each fixed $a\in\mathcal A$. Because $\mathcal A$ is finite, pointwise convergence implies uniform convergence:
\[
    \delta_n
    :=
    \max_{a\in\mathcal A}
    \abs{
    A_{n,\lambda}(a)-A_{\infty,\lambda}(a)
    }
    \to0.
\]

Now let
\[
    a_n\in
    \argmax_{a\in\mathcal A}
    A_{n,\lambda}(a),
\]
and let
\[
    a_\infty^\star\in
    \argmax_{a\in\mathcal A}
    A_{\infty,\lambda}(a).
\]
Because $a_n$ maximizes $A_{n,\lambda}$,
\[
    A_{n,\lambda}(a_n)
    \ge
    A_{n,\lambda}(a_\infty^\star).
\]
Therefore
\[
\begin{aligned}
    A_{\infty,\lambda}(a_\infty^\star)
    -
    A_{\infty,\lambda}(a_n)
    &=
    \left[
    A_{\infty,\lambda}(a_\infty^\star)
    -
    A_{n,\lambda}(a_\infty^\star)
    \right]
    \\
    &\quad
    +
    \left[
    A_{n,\lambda}(a_\infty^\star)
    -
    A_{n,\lambda}(a_n)
    \right]
    \\
    &\quad
    +
    \left[
    A_{n,\lambda}(a_n)
    -
    A_{\infty,\lambda}(a_n)
    \right]
    \\
    &\le
    \delta_n+0+\delta_n
    =
    2\delta_n.
\end{aligned}
\]
Thus the proxy-selected action is asymptotically optimal for the limiting acquisition objective.

Finally assume the limiting maximizer is unique with margin
\[
    \gamma
    =
    A_{\infty,\lambda}(a_\infty^\star)
    -
    \max_{a\ne a_\infty^\star}
    A_{\infty,\lambda}(a)
    >
    0.
\]
For large enough $n$,
\[
    \delta_n<\frac{\gamma}{2}.
\]
For the optimal action,
\[
    A_{n,\lambda}(a_\infty^\star)
    \ge
    A_{\infty,\lambda}(a_\infty^\star)-\delta_n.
\]
For any $a\ne a_\infty^\star$,
\[
    A_{n,\lambda}(a)
    \le
    A_{\infty,\lambda}(a)+\delta_n
    \le
    A_{\infty,\lambda}(a_\infty^\star)-\gamma+\delta_n.
\]
Therefore
\[
    A_{n,\lambda}(a_\infty^\star)
    -
    A_{n,\lambda}(a)
    \ge
    \gamma-2\delta_n
    >
    0.
\]
So no suboptimal action can overtake $a_\infty^\star$, and
\[
    \argmax_{a\in\mathcal A}
    A_{n,\lambda}(a)
    =
    a_\infty^\star
\]
for all sufficiently large $n$.
\end{proof}

\begin{corollary}[Decision transfer at proxy-selected precision]
\label{cor:composed-transfer}
For each width $j$, define
$A_j^\lambda(a)=F_a(m_j^\lambda(\U),C_j^\lambda(\U,\U),\lambda)$.
Let $n_k\to\infty$ be proxy widths and let $N_k\to\infty$ be target
widths. Suppose $\lambda_{n_k}^\star\to\lambda^\star$ is selected on
each proxy either by validation NLL under
Theorem~\ref{thm:lambda-transfer} or by the fixed-checkpoint plug-in evidence
score under Corollary~\ref{cor:evidence-transfer}. Suppose that on the
finite input collection $\U$,
\[
\sup_{\lambda\in\Lambda}
\norm{m_n^\lambda(\U)-m_\infty^\lambda(\U)}\to0,
\qquad
\sup_{\lambda\in\Lambda}
\norm{C_n^\lambda(\U,\U)-C_\infty^\lambda(\U,\U)}\to0,
\]
that the limiting posterior is continuous in $\lambda$, and that each
$F_a$ is jointly continuous in $(m,C,\lambda)$. Then, uniformly over the
finite action set,
\[
A_{n_k}^{\lambda_{n_k}^\star}(a)\to
A_\infty^{\lambda^\star}(a),
\qquad
A_{N_k}^{\lambda_{n_k}^\star}(a)\to
A_\infty^{\lambda^\star}(a).
\]
If the limiting best action is unique with the margin of
Theorem~\ref{thm:decision-transfer}, the proxy decision and the target
decision evaluated at the transferred proxy precision agree for all
sufficiently large $k$.
\end{corollary}

\begin{proof}
Apply uniform posterior convergence along the proxy sequence and the
target sequence. Together with
$\lambda_{n_k}^\star\to\lambda^\star$ and joint continuity, this gives
both displayed limits for each action. Finiteness of the action set makes
both convergences uniform. The margin argument in the proof of
Theorem~\ref{thm:decision-transfer} makes both maximizers equal to the
unique limiting maximizer for all sufficiently large $k$.
\end{proof}

\begin{corollary}[Proxy-to-target decision transfer]
\label{cor:proxy-target-decision}
Let $n$ be a proxy width and $N$ a target width. Define
\[
    \epsilon_{n,N}
    =
    \max_{a\in\mathcal A}
    \abs{
    A_{n,\lambda}(a)-A_{N,\lambda}(a)
    }.
\]
If
\[
    a_n\in\argmax_{a\in\mathcal A}A_{n,\lambda}(a)
\]
and
\[
    a_N^\star\in\argmax_{a\in\mathcal A}A_{N,\lambda}(a),
\]
then
\[
    A_{N,\lambda}(a_N^\star)
    -
    A_{N,\lambda}(a_n)
    \le
    2\epsilon_{n,N}.
\]
If the target model has margin
\[
    \gamma_N
    =
    A_{N,\lambda}(a_N^\star)
    -
    \max_{a\ne a_N^\star}
    A_{N,\lambda}(a)
    >
    0
\]
and
\[
    2\epsilon_{n,N}<\gamma_N,
\]
then
\[
    a_n=a_N^\star.
\]
\end{corollary}

\begin{proof}
The proof is identical to the near-optimality and margin parts of Theorem~\ref{thm:decision-transfer}, replacing the infinite-width score $A_{\infty,\lambda}$ by the target-width score $A_{N,\lambda}$.
\end{proof}

\subsection{Examples of acquisition functions covered}

\paragraph{Posterior-variance acquisition.}
For scalar regression and $x\in\mathcal P$,
\[
    A_{n,\lambda}(x)
    =
    C_{n,\lambda}(x,x).
\]
Since $C_{n,\lambda}(\U,\U)\to C_{\infty,\lambda}(\U,\U)$, the marginal variance converges uniformly over $x\in\mathcal P$. For vector-valued outputs, the same argument covers
\[
    A_{n,\lambda}(x)
    =
    \tr C_{n,\lambda}(x,x)
\]
and
\[
    A_{n,\lambda}(x)
    =
    \log\det
    \left(
    C_{n,\lambda}(x,x)+\tau^2I
    \right),
    \qquad
    \tau^2>0,
\]
because trace and log-determinant on uniformly positive-definite matrices are continuous.

\paragraph{Gaussian information gain.}
For a candidate batch $S\subseteq\mathcal P$ and observation noise $\tau^2>0$, define
\[
    A_{n,\lambda}(S)
    =
    \frac12
    \log\det
    \left(
    I+\tau^{-2}C_{n,\lambda}(S,S)
    \right).
\]
The submatrix $C_{n,\lambda}(S,S)$ converges to $C_{\infty,\lambda}(S,S)$, and
\[
    I+\tau^{-2}C_{n,\lambda}(S,S)\succeq I.
\]
Hence the log-determinant is continuous and
\[
    A_{n,\lambda}(S)\to A_{\infty,\lambda}(S).
\]
Because there are finitely many $K$-batches in a finite pool, convergence is uniform over all batches.

\paragraph{Integrated variance reduction.}
For scalar regression, a candidate batch $S\subseteq\mathcal P$, and a
fixed target/evaluation set $T\subseteq\U$, define
\[
    A_{n,\lambda}(S)
    =
    \sum_{x\in T}
    C_{n,\lambda}(x,S)
    \left(
    C_{n,\lambda}(S,S)+\tau^2I
    \right)^{-1}
    C_{n,\lambda}(S,x),
    \qquad
    \tau^2>0.
\]
Since
\[
    C_{n,\lambda}(S,S)+\tau^2I\succeq\tau^2I,
\]
matrix inversion is stable. Every factor in every summand converges, and $T$ is finite. Therefore
\[
    A_{n,\lambda}(S)\to A_{\infty,\lambda}(S).
\]

\paragraph{Predictive information gain between candidate and target sets.}
Let $S\subseteq\mathcal P$ be a candidate batch and let
$T\subseteq\U$ be a fixed target set. For Gaussian regression with observation noise $\tau^2>0$, define
\[
    Y_S=f_S+\varepsilon_S,
    \qquad
    Y_T=f_T+\varepsilon_T,
\]
with independent
\[
    \varepsilon_S,\varepsilon_T\sim\mathcal N(0,\tau^2I).
\]
The covariance of $(Y_S,Y_T)$ is
\[
    \Sigma_{n,\lambda}(S,T)
    =
    \begin{pmatrix}
    C_{n,\lambda}(S,S)+\tau^2I & C_{n,\lambda}(S,T)\\
    C_{n,\lambda}(T,S) & C_{n,\lambda}(T,T)+\tau^2I
    \end{pmatrix}.
\]
All blocks converge, and the noise terms make the covariance uniformly positive definite. Therefore any mutual-information expression written as a finite sum and difference of log-determinants of these covariance matrices converges.

\paragraph{BALD-style classification acquisition.}
For classification and $x\in\mathcal P$, suppose the Laplace posterior over logits at $x$ is
\[
    f_n(x)\sim
    \mathcal N
    \left(
    m_{n,\lambda}(x),
    C_{n,\lambda}(x,x)
    \right).
\]
The BALD score is
\[
    A_{n,\lambda}(x)
    =
    H
    \left(
    \E_{f_n(x)}
    [
    p(y\mid f_n(x))
    ]
    \right)
    -
    \E_{f_n(x)}
    \left[
    H(p(y\mid f_n(x)))
    \right].
\]
The softmax map is continuous, and categorical entropy is continuous and bounded on the probability simplex. Since Gaussian means and covariances converge in finite dimension, the corresponding Gaussian laws converge weakly. Expectations of bounded continuous functions therefore converge, so
\[
    A_{n,\lambda}(x)\to A_{\infty,\lambda}(x).
\]
On a finite pool, this convergence is uniform over $x\in\mathcal P$.
This statement concerns the exact Gaussian expectations. It also holds
for a fixed common-random-number Monte Carlo rule by continuity of its
finite reparameterized sum. Independent fixed-sample draws across widths
require a separate Monte Carlo error bound.

\paragraph{Top-$K$ ranking agreement.}
Suppose individual scores $a_{n,\lambda}(x)$ satisfy
\[
    \max_{x\in\mathcal P}
    \abs{
    a_{n,\lambda}(x)-a_{\infty,\lambda}(x)
    }
    \to0.
\]
Let the limiting scores be ordered as
\[
    a_{\infty,\lambda}(x_{(1)})
    \ge
    a_{\infty,\lambda}(x_{(2)})
    \ge
    \cdots
    \ge
    a_{\infty,\lambda}(x_{(M)}).
\]
If the top-$K$ margin
\[
    \Delta_K
    =
    a_{\infty,\lambda}(x_{(K)})
    -
    a_{\infty,\lambda}(x_{(K+1)})
    >
    0,
\]
then for sufficiently large $n$ the width-$n$ top-$K$ set equals the limiting top-$K$ set. Indeed, once the uniform score error is below $\Delta_K/2$, every limiting top-$K$ candidate still has larger width-$n$ score than every non-top-$K$ candidate.

\paragraph{Selective prediction and detection.}
Suppose a detection or rejection score is
\[
    s_{n,\lambda}(x)
    =
    g
    \left(
    m_{n,\lambda}(x),
    C_{n,\lambda}(x,x)
    \right),
\]
where $g$ is continuous. Then
\[
    s_{n,\lambda}(x)\to s_{\infty,\lambda}(x)
\]
uniformly on finite $\U$. For a threshold rule
\[
    \text{flag }x
    \quad\Longleftrightarrow\quad
    s_{n,\lambda}(x)>t,
\]
exact flag agreement requires the threshold margin
\[
    \eta
    =
    \min_{x\in\U}
    \abs{
    s_{\infty,\lambda}(x)-t
    }
    >
    0.
\]

Once the uniform score error is below $\eta$, the width-$n$ and limiting rules make identical flag/no-flag decisions on all of $\U$.

\paragraph{Finite-round sequential active learning.}
The one-step result extends to finitely many acquisition rounds by
induction. Suppose the following convergence and margin conditions hold
conditionally along the common limiting acquisition history: at each
round $r=1,\ldots,R$,
\[
    m_{n,\lambda}^{(r)}(\U)\to m_{\infty,\lambda}^{(r)}(\U),
    \qquad
    C_{n,\lambda}^{(r)}(\U,\U)\to C_{\infty,\lambda}^{(r)}(\U,\U),
\]
and the limiting acquisition rule at that round has a nonzero margin.
Then the round-$1$ selected action agrees for all sufficiently large
$n$. Conditional on that agreement, the updated training set is the
same for the width-$n$ and limiting processes, so the round-$2$
hypotheses apply on their common history. Repeating for finitely many
rounds gives agreement of the entire finite acquisition trajectory.
\end{document}